\documentclass[shortAfour]{sagej}
\pdfoutput=1
\usepackage{graphicx}
\usepackage{amssymb}
\usepackage{bm}
\usepackage[justification=centering]{caption}
\usepackage[justification=centering]{subcaption}
\usepackage[round]{natbib}
\usepackage{tikz}
\usetikzlibrary{positioning}
\usepackage[hidelinks, bookmarks=true]{hyperref}
\usepackage{cleveref}
\usepackage[titletoc,title]{appendix}

\usepackage{amsthm}
\usepackage{thmtools, thm-restate}

\DeclareMathOperator*{\argmin}{arg\,min}
\DeclareMathOperator*{\argmax}{arg\,max}
\DeclareMathOperator*{\softmin}{\text{soft}\!\min}

\makeatletter
\newcommand{\manuallabel}[2]{\def\textbf{\@currentlabel{#2}}\label{#1} }
\makeatother

\newcounter{hypothsetcounter}
\newcounter{hypothcounter}[hypothsetcounter]
\renewcommand\thehypothsetcounter{\arabic{hypothsetcounter}} 
\renewcommand\thehypothcounter{\alph{hypothcounter}} 

\makeatletter
\renewcommand\p@hypothcounter{\thehypothsetcounter}
\makeatother

\newcommand{\newhypothset}{
    \refstepcounter{hypothsetcounter}%
}

\newcommand{\hypothcounter}[2]{
    \refstepcounter{hypothcounter}\label{#2}%
    \noindent \textbf{H\thehypothsetcounter \thehypothcounter} \indent \textit{#1}%
}

\crefformat{hypothcounter}{\textbf{H#2#1#3}}

\newcommand{\expctover}[2]{\mathbb{E}_{#1}\!\left[#2\right]}
\newcommand{\expctarg}[1]{\mathbb{E}\!\left[#1\right]}
\newcommand{\given}{\mid}

\newcommand{\usertext}{\text{u}}
\newcommand{\robottext}{\text{r}}
\newcommand{\policy}{\pi}
\newcommand{\policyuser}{\policy^\usertext}
\newcommand{\policyusergoal}{\policy^\usertext_\goal}
\newcommand{\policyrobot}{\policy^\robottext}

\newcommand{\goal}{g}
\newcommand{\Goal}{G}

\newcommand{\actionuser}{u}
\newcommand{\Actionuser}{U}
\newcommand{\actionrobot}{a}
\newcommand{\Actionrobot}{A}

\newcommand{\stateenv}{x}
\newcommand{\Stateenv}{X}
\newcommand{\staterobot}{x}

\newcommand{\stateenvgoal}{s}
\newcommand{\Stateenvgoal}{S}
\newcommand{\staterobgoal}{s}

\newcommand{\state}{\staterobgoal}

\newcommand{\belief}{b}

\newcommand{\transition}{T}
\newcommand{\transitionbelief}{\tau}

\newcommand{\costuser}{C^\usertext}
\newcommand{\costrobot}{C^\robottext}
\newcommand{\costusergoal}{C^\usertext_\goal}
\newcommand{\costrobotgoal}{C^\robottext_\goal}

\newcommand{\transitionuser}{\transition^{\usertext}}

\newcommand{\transitionallargs}{\transition(\stateenv' \given \stateenv, \actionuser, \actionrobot)}

\newcommand{\stateactionsarg}[1]{\state_{#1}, \actionuser_{#1}, \actionrobot_{#1}}
\newcommand{\stateactions}{\stateactionsarg{ }}

\newcommand{\stateenvactionsarg}[1]{\stateenv_{#1}, \actionuser_{#1}, \actionrobot_{#1}}
\newcommand{\stateenvactions}{\stateenvactionsarg{ }}

\newcommand{\beliefactionsarg}[1]{\belief_{#1}, \actionuser_{#1}, \actionrobot_{#1}}
\newcommand{\beliefactions}{\beliefactionsarg{ }}

\newcommand{\qrobot}{Q}
\newcommand{\vrobot}{V}
\newcommand{\sumtime}{\sum_{t}}
\newcommand{\vopt}{\vrobot^{*}}
\newcommand{\qopt}{\qrobot^{*}}
\newcommand{\vhs}{\vrobot^{\text{HS}}}
\newcommand{\qhs}{\qrobot^{\text{HS}}}

\newcommand{\qmdp}{Q}
\newcommand{\vmdp}{V}
\newcommand{\qsoft}{Q^{\approx}}
\newcommand{\vsoft}{V^{\approx}}
\newcommand{\qgoal}{\qmdp_{\goal}}
\newcommand{\vgoal}{\vmdp_{\goal}}

\newcommand{\vuser}{\vmdp^{\usertext}}
\newcommand{\qgoaluser}{\qmdp_{\goal}^{\usertext}}
\newcommand{\vgoaluser}{\vmdp_{\goal}^{\usertext}}

\newcommand{\vgoalrobot}{\vmdp_{\goal}^{\robottext}}
\newcommand{\qgoalsoft}{\qsoft_\goal}
\newcommand{\vgoalsoft}{\vsoft_\goal}
\newcommand{\qgoalt}[1]{\qmdp_{\goal}^{#1}}
\newcommand{\vgoalt}[1]{\vmdp_{\goal}^{#1}}

\newcommand{\qgoalsoftt}[1]{\qsoft_{ \goal,{#1}}}
\newcommand{\vgoalsoftt}[1]{\vsoft_{ \goal,{#1}}}
\newcommand{\costgoal}{C_\goal}

\newcommand{\costgoaluser}{\costgoal^\usertext}
\newcommand{\costgoalrobot}{\costgoal^\robottext}

\newcommand{\target}{\kappa}
\newcommand{\Target}{K}
\newcommand{\costtarg}{C_{\target}}
\newcommand{\costtargprime}{C_{\target'}}
\newcommand{\costtargstar}{C_{\target^*}}
\newcommand{\costtargrobot}{C_{\target}^{\robottext}}

\newcommand{\costtarguser}{C_{\target}^{\usertext}}

\newcommand{\qtarg}{Q_{\target}}
\newcommand{\vtarg}{V_{\target}}

\newcommand{\vtargt}[1]{V_{\target}^{#1}}

\newcommand{\qtargstar}{Q_{\target^*}}
\newcommand{\vtargstart}[1]{V_{\target^*}^{#1}}
\newcommand{\qtargsoft}{Q^{\approx}_{\target}}
\newcommand{\vtargsoft}{V^{\approx}_{\target}}
\newcommand{\qtargsoftt}[1]{\qsoft_{ \target,{#1}}}
\newcommand{\vtargsoftt}[1]{\vsoft_{ \target,{#1}}}

\newcommand{\traj}{\xi}

\newcommand{\trajat}[1]{\traj_{#1}^{t \rightarrow T}}
\newcommand{\trajatp}[1]{\traj_{#1}^{t+1 \rightarrow T}}
\newcommand{\trajtot}{\traj^{0 \rightarrow t}}

\newcommand{\goalrestrictionset}{\mathcal{R}}
\newcommand{\usergoal}{\goal^\usertext}
\newcommand{\userGoal}{\Goal^\usertext}
\newcommand{\robotgoal}{\goal^\robottext}
\newcommand{\robotGoal}{\Goal^\robottext}

\newcommand{\arbitration}{\alpha}

\begin{document}

\runninghead{autonomy et al.}

\title{Shared Autonomy via Hindsight Optimization for Teleoperation and Teaming}

\author{Shervin Javdani\affilnum{1}, Henny Admoni\affilnum{1}, Stefania Pellegrinelli\affilnum{2}, Siddhartha S. Srinivasa\affilnum{1}, and J. Andrew Bagnell\affilnum{1}}

\affiliation{
\affilnum{1}The Robotics Institute, Carnegie Mellon University
\affilnum{2}ITIA-CNR, Institute of Industrial Technologies and Automation, National Research Council of Italy\\
}

\corrauth{
Shervin Javdani
Carnegie Mellon University
Robotics Institute
5000 Forbes Ave
Pittsburgh, PA 15213
}

\email{sjavdani@cmu.edu}

\begin{abstract}
In shared autonomy, a user and autonomous system work together to achieve shared goals. To collaborate effectively, the autonomous system must know the user's goal. As such, most prior works follow a predict-then-act model, first predicting the user's goal with high confidence, then assisting given that goal. Unfortunately, confidently predicting the user's goal may not be possible until they have nearly achieved it, causing predict-then-act methods to provide little assistance. However, the system can often provide useful assistance even when confidence for any single goal is low (e.g. move towards multiple goals). In this work, we formalize this insight by modelling shared autonomy as a Partially Observable Markov Decision Process (POMDP), providing assistance that minimizes the expected cost-to-go with an unknown goal. As solving this POMDP optimally is intractable, we use hindsight optimization to approximate. We apply our framework to both shared-control teleoperation and human-robot teaming. Compared to predict-then-act methods, our method achieves goals faster, requires less user input, decreases user idling time, and results in fewer user-robot collisions.


\end{abstract}


\maketitle

\section{Introduction}
\label{sec:intro}

\graphicspath{{./}{./images/}{./images_rss_2015/}}

Human-robot collaboration studies interactions between humans and robots sharing a workspace. One instance of collaboration arises in \emph{shared autonomy}, where both the user and robotic system act simultaneously to achieve shared goals. For example, in \emph{shared control teleoperation}~\citep{goertz_1963, rosenberg_1993, aigner_1997, debus_2000, dragan_2013_assistive}, both the user and system control a single entity, the robot, in order to achieve the user's goal. In \emph{human-robot teaming}, the user and system act independently to achieve a set of related goals~\citep{hoffman_2007, arai_2010, dragan_2013_legible, koppula_2013, mainprice_2013, gombolay_2014, nikolaidis_2017_mutual}. 

While each instance of shared autonomy has many unique requirements, they share a key common challenge - for the autonomous system to be an effective collaborator, it needs to know the user's goal. For example, feeding with shared control teleoperation, an important task for assistive robotics~\citep{chung_2013}, requires knowing what the user wants to eat (\cref{subfig:eating_ambiguous_goal}). Wrapping gifts with a human-robot team requires knowing which gift the user will wrap to avoid getting in their way and hogging shared resources (\cref{subfig:teaming_ambiguous_goal}).

\begin{figure}[t]
\centering
\begin{subfigure}{0.24\textwidth}
  \includegraphics[width=1.0\textwidth, trim=0 0 0 0, clip=true]{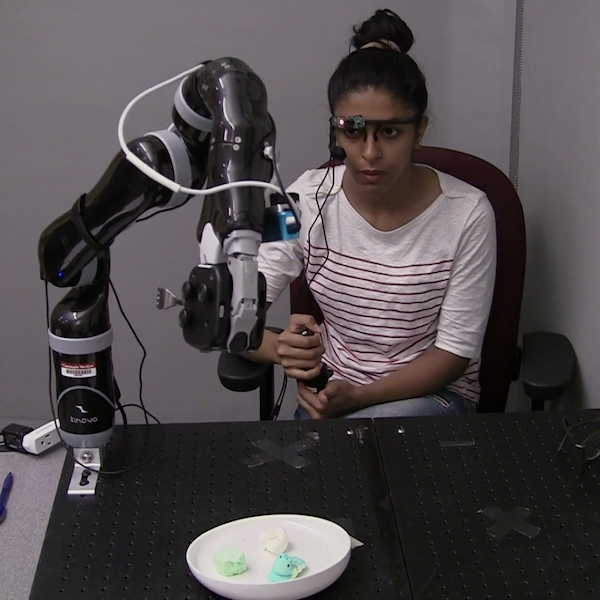}
  \caption{Shared Control Teleop}
  \label{subfig:eating_ambiguous_goal}
\end{subfigure}
\hfill
\begin{subfigure}{0.24\textwidth}
  \includegraphics[width=1.0\textwidth, trim=0 0 0 0, clip=true]{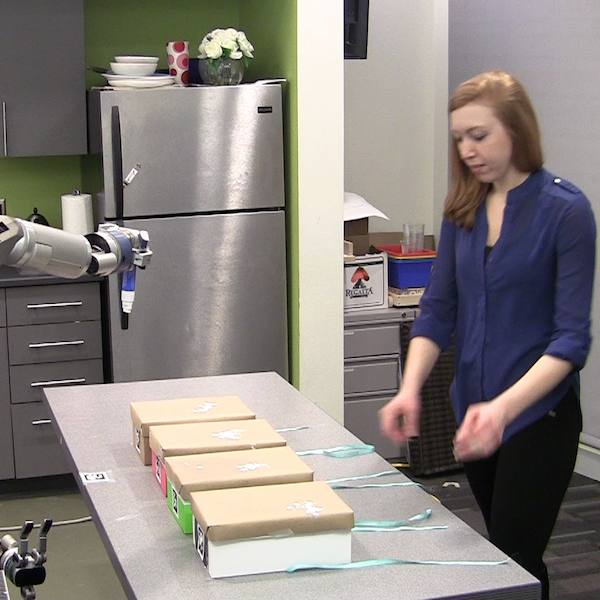}
  \caption{Human-Robot Teaming}
  \label{subfig:teaming_ambiguous_goal}
\end{subfigure}
\caption{We can provide useful assistance even when we do not know the user's goal. (\subref{subfig:eating_ambiguous_goal}) Our feeding experiment, where the user wants to eat one of the bites of food on the plate. With an unknown goal, our method autonomously orients the fork and moves towards all bites. In contrast, predict-then-act methods only helped position the fork at the end of execution. Users commented that the initial assistance orienting the fork and getting close to all bites was the most helpful, as this was the most time consuming portion of the task. (\subref{subfig:teaming_ambiguous_goal}) Our teaming experiment, where the user wraps a box, and the robot must stamp a different box. Here, the user's motion so far suggests their goal is likely either the green or white box. Though we cannot confidently predict their single goal, our method starts making progress for the other boxes.}
 \label{fig:ambiguous_goals}
\end{figure}

In general, the system does not know the user's goal apriori. We could alleviate this issue by requiring users to explicitly specify their goals (e.g. through voice commands). However, there are often a continuum of goals to choose from (e.g. location to place an object, size to cut a bite of food), making it impossible for users to precisely specify their goals. Furthermore, prior works suggest requiring explicit communication leads to ineffective collaboration~\citep{vanhooydonck_2003, goodrich_2003, green_2007}. Instead, implicit information should be used to make collaboration seamless. In shared autonomy, this suggests utilizing sensing of the environment and user actions to infer the user's goal. This idea has been successfully applied for shared control teleoperation~\citep{li_2003,yu_2005,kragic_2005,kofman_2005,aarno_2008,carlson_2012, dragan_2013_assistive, hauser_2013, muelling_2015} and human-robot teaming~\citep{hoffman_2007, nguyen_2011, macindoe_2012, mainprice_2013, koppula_2013, lasota_2015}.



As providing effective assistance requires knowing the user's goal, most shared autonomy methods do not assist when the goal is unknown. These works split shared autonomy into two parts: 1) predict the user's goal with high probability, and 2) assist for that single goal, potentially using prediction confidence to regulate assistance. We refer to this approach as \emph{predict-then-act}. While this has been effective in simple scenarios with few goals~\citep{yu_2005, kofman_2005, carlson_2012, dragan_2013_assistive, koppula_2013, muelling_2015}, it is often impossible to predict the user's goal until the end of execution (e.g. cluttered scenes), causing these methods to provide little assistance. Addressing this lack of assistance is of great practical importance - with uncertainty over only goals in our feeding experiment, a predict-then-act method provided assistance for only $31\%$ of the time on average, taking $29.4$ seconds on average before the confidence threshold was initially reached.

In this work, we present a general framework for goal-directed shared autonomy that does not rely on predicting a single user goal (\cref{fig:robot_human_model}). We assume the user's goal is fixed (e.g. they want a particular bite of food), and the autonomous system should adapt to the user goal\footnote{While we assume the goal is fixed, we do not assume how the user will achieve that goal (e.g. grasp location) is fixed.}. Our key insight is that there are useful assistance actions for \emph{distributions over goals}, even when confidence for a particular goal is low (e.g. move towards multiple goals) (\cref{fig:ambiguous_goals}). We formalize this notion by modelling our problem as a Partially Observable Markov Decision Process (POMDP)~\citep{kaelbling_1998_pomdp}, treating the user's goal as hidden state. When the system is uncertain of the user goal, our framework naturally optimizes for an assistance action that is helpful for many goals. When the system confidently predicts a single user goal, our framework focuses assistance given that goal (\cref{fig:teledata}). 

As our state and action spaces are both continuous, solving for the optimal action in our POMDP is intractable. Instead, we approximate using QMDP~\citep{littman_1995}, also referred to as hindsight optimization~\citep{chong_2000,yoon_2008}. This approximation has many properties suitable for shared autonomy: it is computationally efficient, works well when information is gathered easily~\citep{koval_2014}, and will not oppose the user to gather information. The result is a system that minimizes the expected cost-to-go to assist for any distribution over goals. 

We apply our framework in user study evaluations for both shared control teleoperation and human-robot teaming. For shared control teleoperation, users performed two tasks: a simpler object grasping task (\cref{sec:experiment_rss_2015}), and a more difficult feeding task (\cref{sec:experiment_hri_2016}). In both cases, we find that our POMDP based method enabled users to achieve goals faster and with less joystick input than a state-of-the-art predict-then-act method~\citep{dragan_2013_assistive}. Subjective user preference differed by task, with no statistical difference for the simpler object grasping task, and users preferring our POMDP method for the more difficult feeding task. 

For human-robot teaming (\cref{sec:experiment_iros_2016}), the user and robot performed a collaborative gift-wrapping task, where both agents had to manipulate the same set of objects while avoiding collisions. We found that users spent less time idling and less time in collision while collaborating with a robot using our method. However, results for total task completion time are mixed, as predict-then-act methods are able to take advantage of more optimized motion planners, enabling faster execution once the user goal is confidently predicted.


\begin{figure}[t]
\centering
\includegraphics[width=0.49\textwidth, trim=34 156 45 64, clip=true]{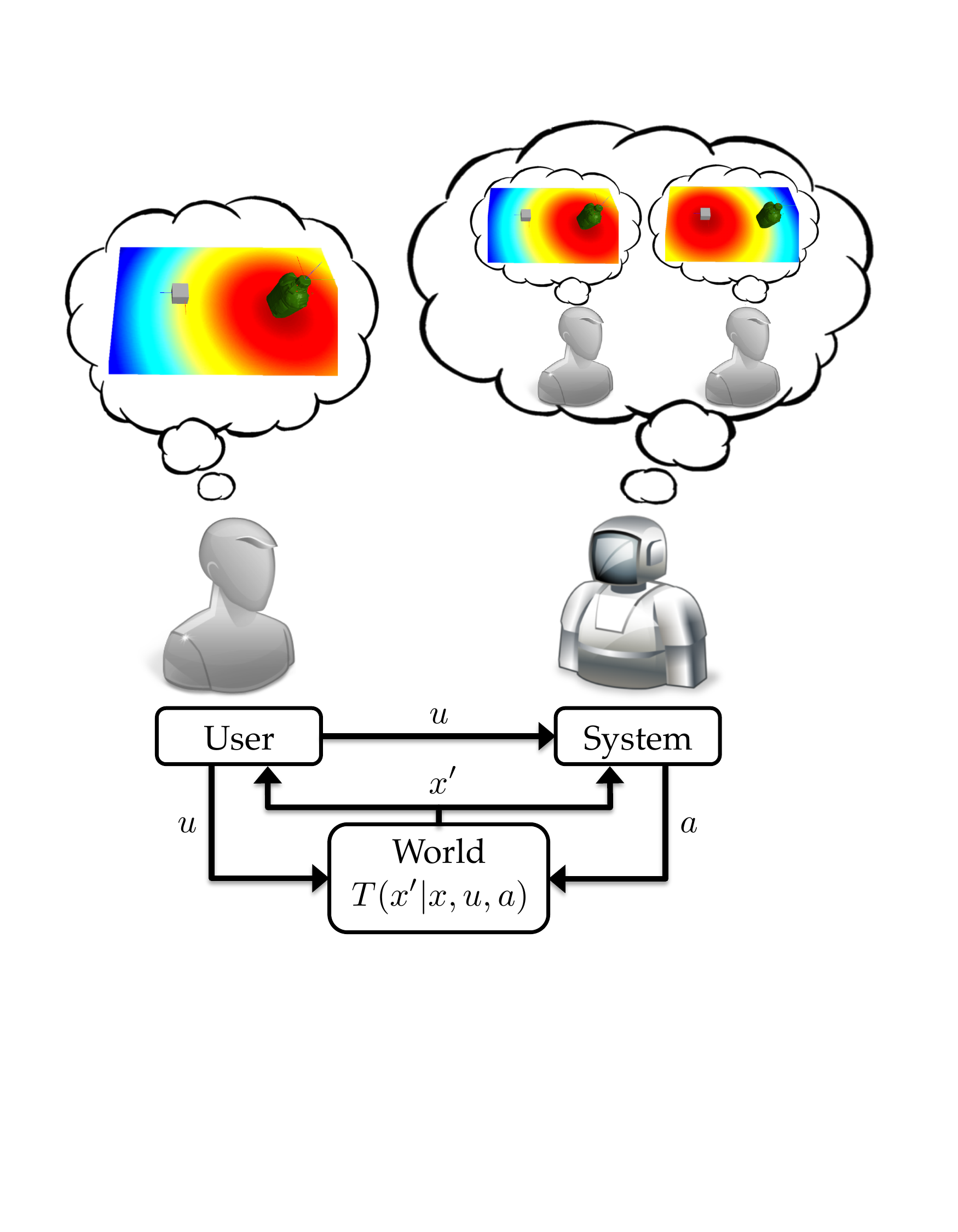}
\caption{ Our shared autonomy framework. We assume the user executes a policy for their single goal, depicted as a heatmap plotting the value function at each position. Our shared autonomy system models all possible user goals and their corresponding policies. From user actions $\actionuser$, a distribution over goals is inferred. Using this distribution and the value functions for each goal, the system selects an action $\actionrobot$. The world transitions from $\stateenv$ to $\stateenv'$. The user and shared autonomy system both observe this state, and repeat action selection.}
 \label{fig:robot_human_model}
\end{figure} 

\begin{figure}[t]
\centering
 \begin{subfigure}{0.240\textwidth}
   \centering 
   \hspace*{1mm} \includegraphics[width=0.92\textwidth, trim=400 50 400 450, clip=true]{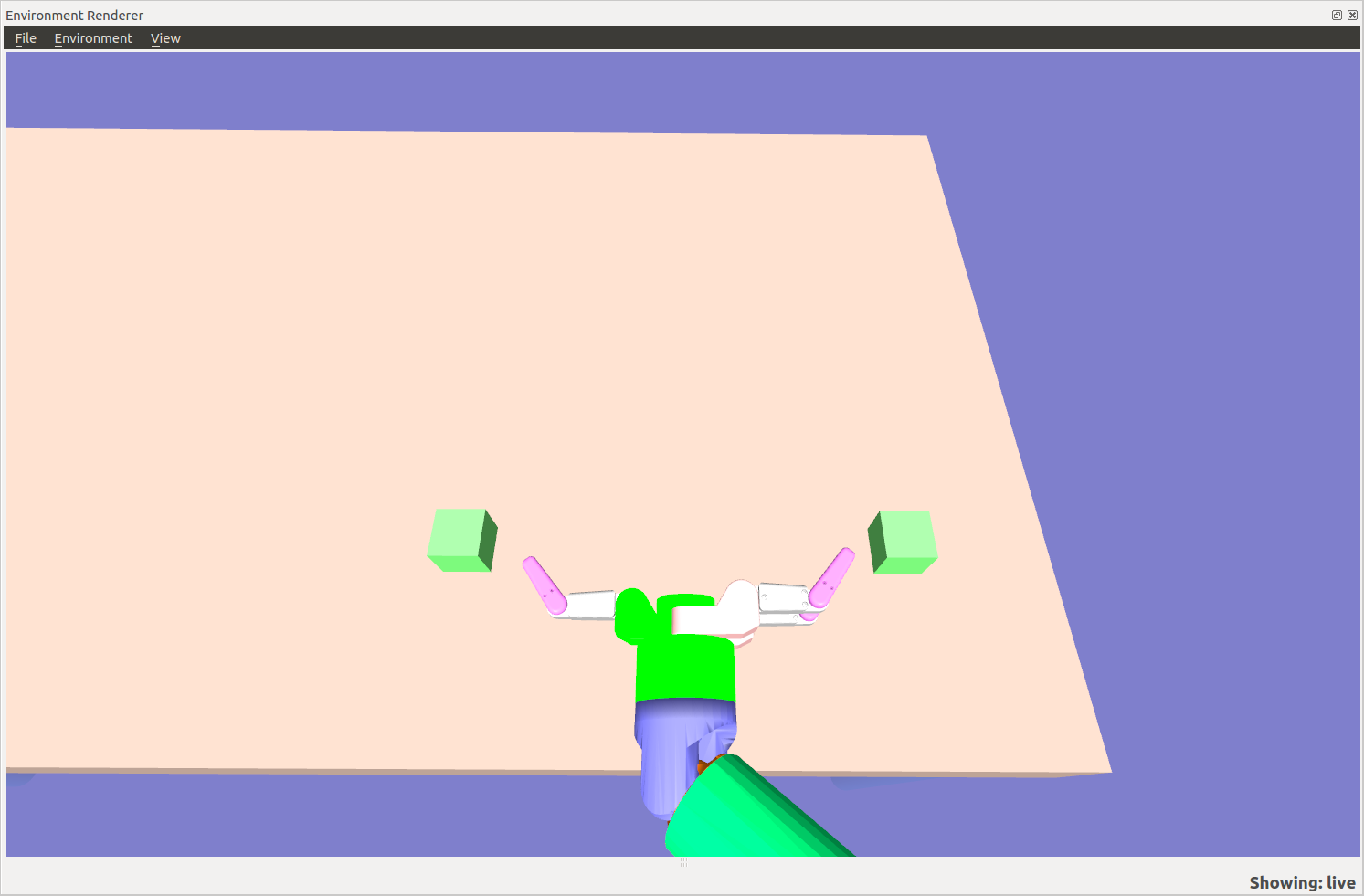}
   \hspace*{-2mm} \includegraphics[trim=1 3 0 -3, clip=true]{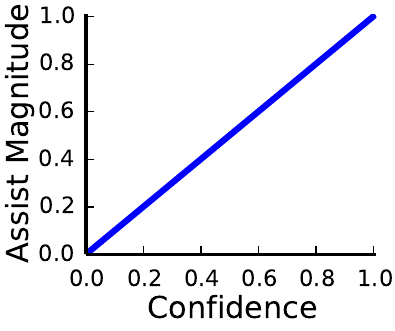} \hfill
   \caption{}
 \label{fig:teledata_cen}
 \end{subfigure}
 \hfill
 \begin{subfigure}{0.240\textwidth}
   \centering
   \hspace*{1mm} \includegraphics[width=0.92\textwidth, trim=400 50 400 450, clip=true]{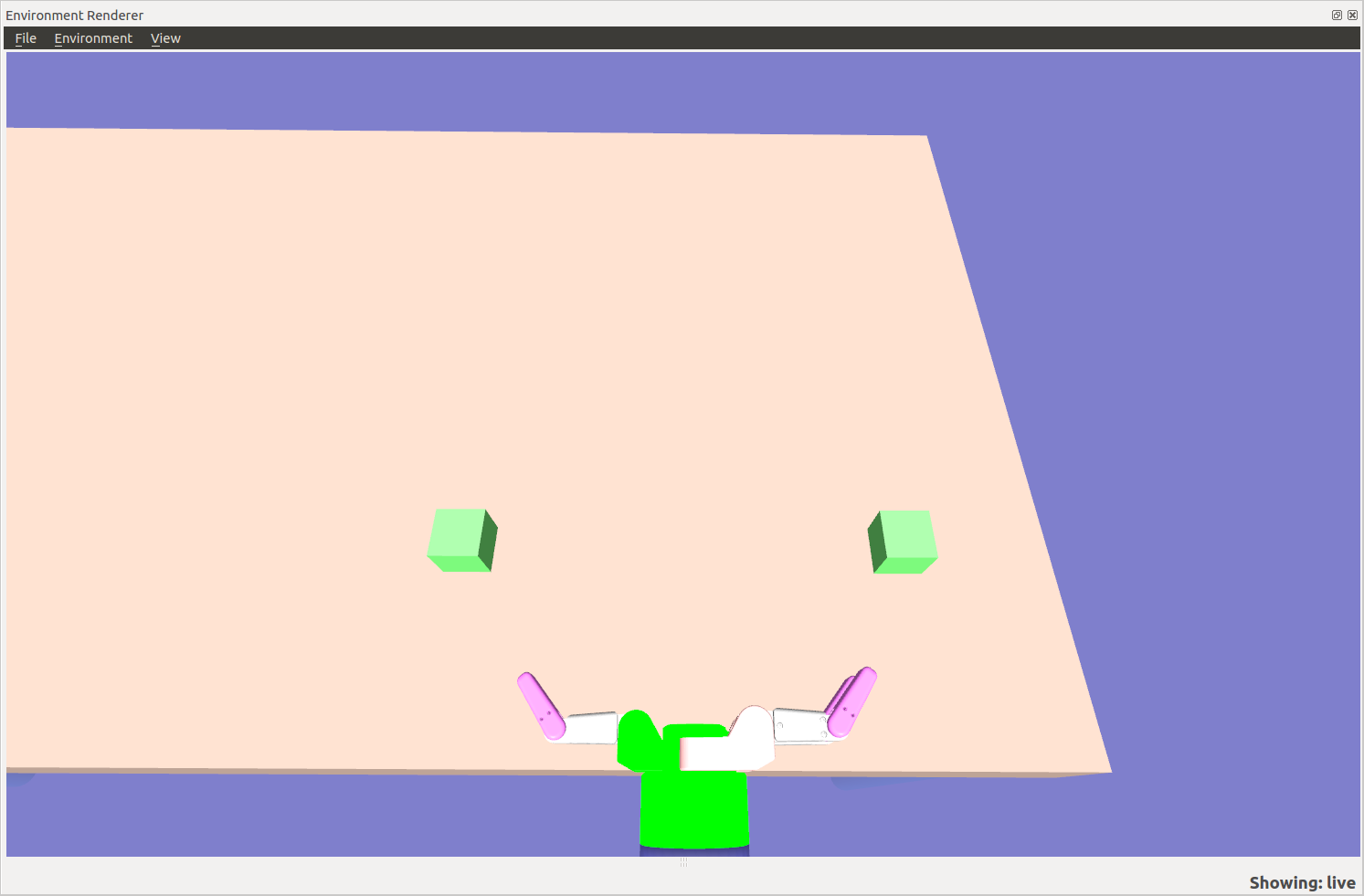}
   \hspace*{-1mm}\includegraphics[trim=1 3 0 -3, clip=true]{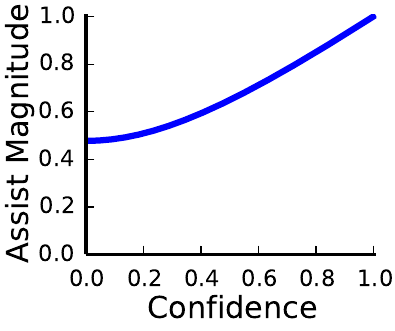} \hfill
  \caption{}
 \label{fig:teledata_back}
 \end{subfigure}
 \caption{Arbitration as a function of confidence with two goals. Confidence $=\max_g p(g) - \min_g p(g)$, which ranges from $0$ (equal probability) to $1$ (all probability on one goal). (\subref{fig:teledata_cen}) The hand is directly between the two goals, where no action assists for both goals. As confidence for one goal increases, assistance increases. (\subref{fig:teledata_back}) From here, going forward assists for both goals, enabling the assistance policy to make progress even with $0$ confidence.}
\label{fig:teledata}
\end{figure}

\section{Related Works}
\label{sec:related}

\subsection{Shared Control Teleoperation}
\label{sec:related_shared_teleop}

\graphicspath{{./}{./images/}{./images_hri_2016/}}

Shared control teleoperation has been used to assist disabled users using robotic arms~\citep{kim_2006_bci, kim_2012, mcmullen_2014, katyal_2014, schroeer_2015, muelling_2015} or wheelchairs~\citep{argall_2016,carlson_2012}, operate robots remotely~\citep{shen_2004, you_2011, leeper_2012}, decrease operator fatigue in surgical settings~\citep{park_2001, marayong_2003, kragic_2005, aarno_2005_virtualfixtures, li_2007}, and many other applications. As such, there are a great many methods catering to the specific needs of each domain.

One common paradigm launches a fully autonomous takeover when some trigger is activated, such as a user command~\citep{shen_2004, bien_2004, simpson_2005, kim_2012}, or when a goal predictor exceeds some confidence threshold~\citep{fagg_2004, kofman_2005, mcmullen_2014, katyal_2014}. Others have utilized similar triggers to initiate a subtask in a sequence~\citep{schroeer_2015, jain_argall_2015}. While these systems are effective at accomplishing the task, studies have shown that users often prefer having more control~\citep{kim_2012}.

Another line of work utilizes high level user commands, and relies on autonomy to generate robot motions. Systems have been developed to enable users to specify an end-effector path in 2D, which the robot follows with full configuration space plans~\citep{you_2011, hauser_2013}. Point-and-click interfaces have been used for object grasping with varying levels of autonomy~\citep{leeper_2012}. Eye gaze has been utilized to select a target object and grasp position~\citep{bien_2004}.

Another paradigm augments user inputs minimally to maintain some desired property, e.g. collision avoidance, without necessarily knowing exactly what goal the user wants to achieve. Sensing and complaint controllers have been used increase safety during teleoperation~\citep{kim_2006_bci, vogel_2014}. \emph{Potential field} methods have been employed to push users away from obstacles~\citep{crandall_2002} and towards goals~\citep{aigner_1997}. For assistive robotics using modal control, where users control subsets of the degrees-of-freedom of the robot in discrete modes (\cref{fig:control_modes}), \citet{herlant_2016} demonstrate a method for automatic time-optimal mode switching.

\begin{figure}
  \begin{subfigure}[b]{.32\linewidth}
    \includegraphics[width=\linewidth]{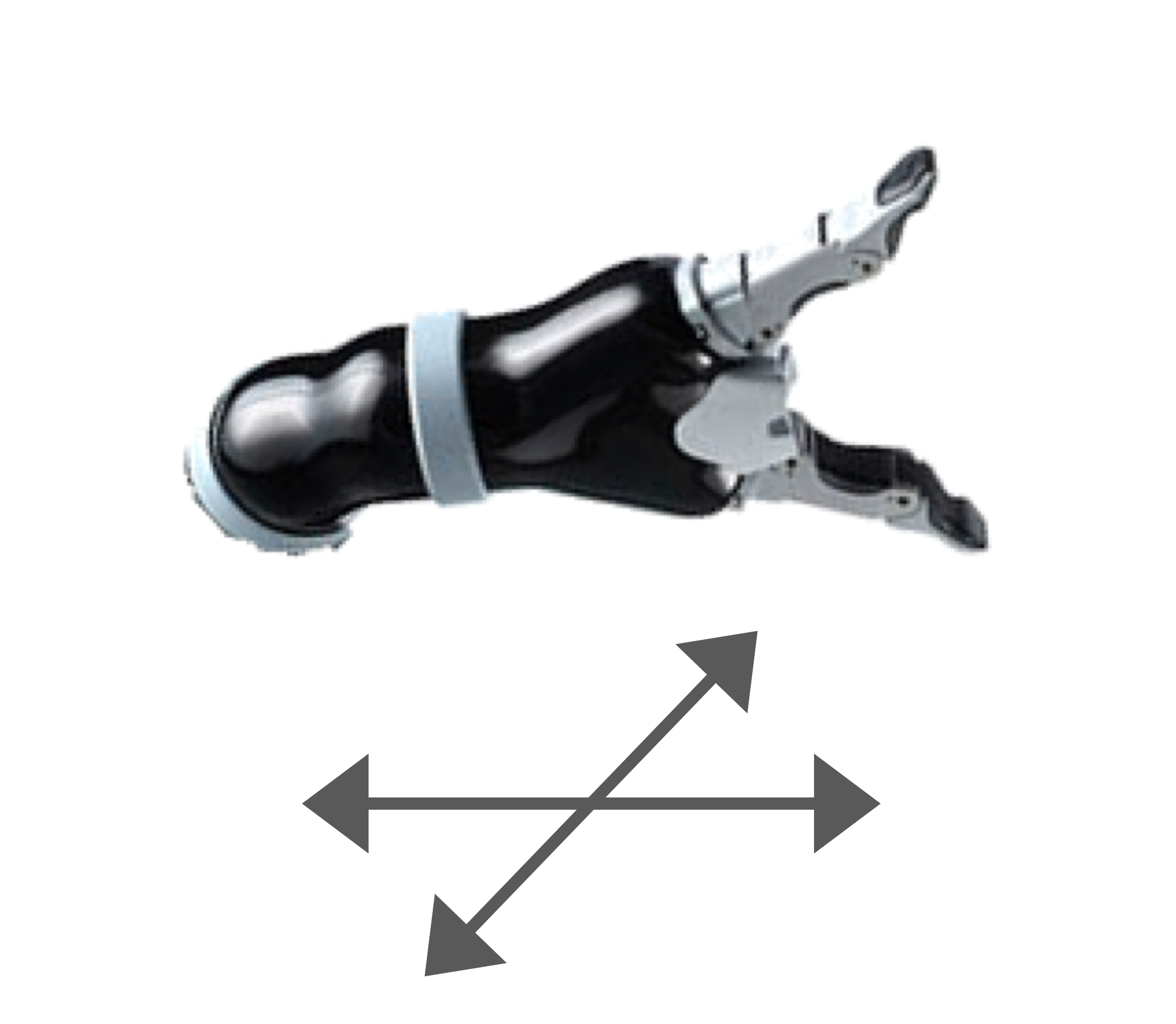}
    \caption{Mode 1}
  \end{subfigure}
  \begin{subfigure}[b]{.32\linewidth}
    \includegraphics[width=\linewidth, trim=0 4 0 0, clip=true]{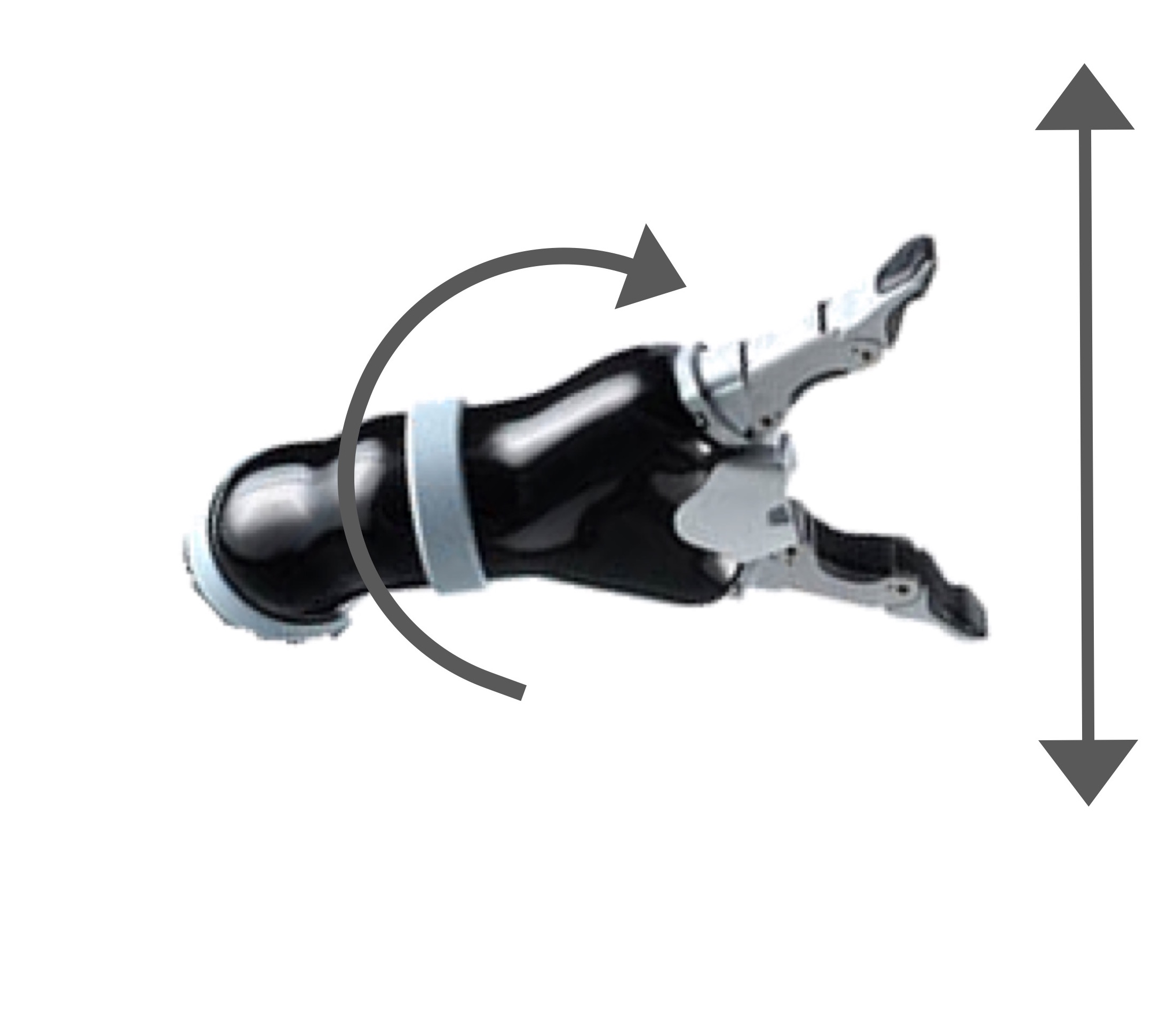}
    \caption{Mode 2}
  \end{subfigure}
  \begin{subfigure}[b]{.32\linewidth}
    \includegraphics[width=\linewidth, trim=0 6 0 4, clip=true]{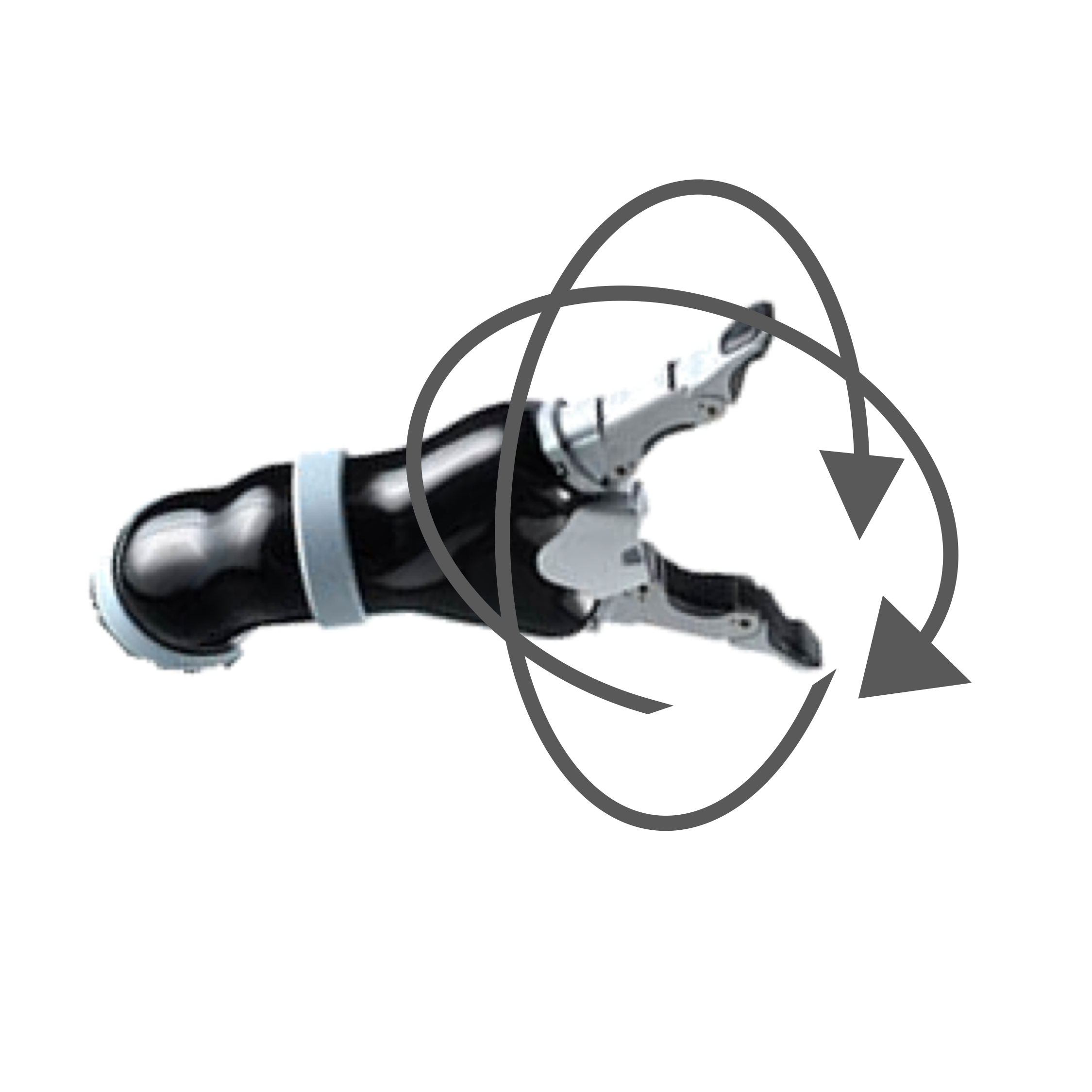}
    \caption{Mode 3}
  \end{subfigure}
  \caption{Modal control used in our feeding experiment on the Kinova MICO, with three control modes and a 2 degree-of-freedom input device. Fewer input DOFs means more modes are required to control the robot.}
  \label{fig:control_modes}
\end{figure}

Similarly, methods have been developed to augment user inputs to follow some constraint. \emph{Virtual fixtures}, commonly used in surgical robotics settings, are employed to project user commands onto path constraints (e.g. straight lines only)~\citep{park_2001,li_2003, marayong_2003, kragic_2005, aarno_2005_virtualfixtures, li_2007}. \citet{mehr_2016} learn constraints online during execution, and apply constraints softly by combining constraint satisfaction with user commands. While these methods benefit from not needing to predict the user's goal, they generally rely on a high degree-of-freedom input, making their use limited for assistive robotics, where disabled users can operate few DOF at a time and thus rely on modal control~\citep{herlant_2016}.

\emph{Blending} methods~\citep{dragan_2013_assistive} attempt to bridge the gap between highly assistive methods with little user control, and minimal assistance with higher user burden. User actions and full autonomy are treated as two independent sources, which are combined by some \emph{arbitration} function that determines the relative contribution of each (\cref{fig:blend_diagram}). \citet{dragan_2013_assistive} show that many methods of shared control teleoperation (e.g. autonomous takeover, potential field methods, virtual fixtures) can be generalized as blending with a particular arbitration function. 

Blending is one of the most used shared control teleopration paradigms due to computational efficiency, simplicity, and empirical effectiveness~\citep{li_2011, carlson_2012, dragan_2013_assistive, muelling_2015, gopinath_2016}. However, blending has two key drawbacks.
First, as two independent decisions are being combined without evaluating the action that will be executed, catastrophic failure can result even when each independent decision would succeed~\citep{trautman_2015}. Second, these systems rely on a \emph{predict-then-act} framework, predicting the single goal the user is trying to achieve before providing any assistance. Often, assistance will not be provided for large portions of execution while the system has low confidence in its prediction, as we found in our feeding experiment (\cref{sec:experiment_hri_2016}). 

Recently, \citet{hauser_2013} presented a system which provides assistance for a distribution over goals. Like our method, this policy-based method minimizes an expected cost-to-go while receiving user inputs (\cref{fig:policy_diagram}). The system iteratively plans trajectories given the current user goal distribution, executes the plan for some time, and updates the distribution given user inputs. In order to efficiently compute the trajectory, it is assumed that the cost function corresponds to squared distance, resulting in the calculation decomposing over goals. Our model generalizes these notions, enabling the use of any cost function for which a value function can be computed.

In this work, we assume the user does not change their goal or actions based on autonomous assistance, putting the burden of goal inference entirely on the system. \citet{nikolaidis_2017_shared} present a game-theoretic approach to shared control teleoperation, where the user adapts to the autonomous system. Each user has an \emph{adaptability}, modelling how likely the user is to change goals based on autonomous assistance. They use a POMDP to learn this adaptability during execution. While more general, this model is computationally intractable for continuous state and actions.


\begin{figure}
 \begin{tikzpicture} 

 \node[inner sep=-2pt] (user) at (0,4)
     {\includegraphics[width=.11\textwidth]{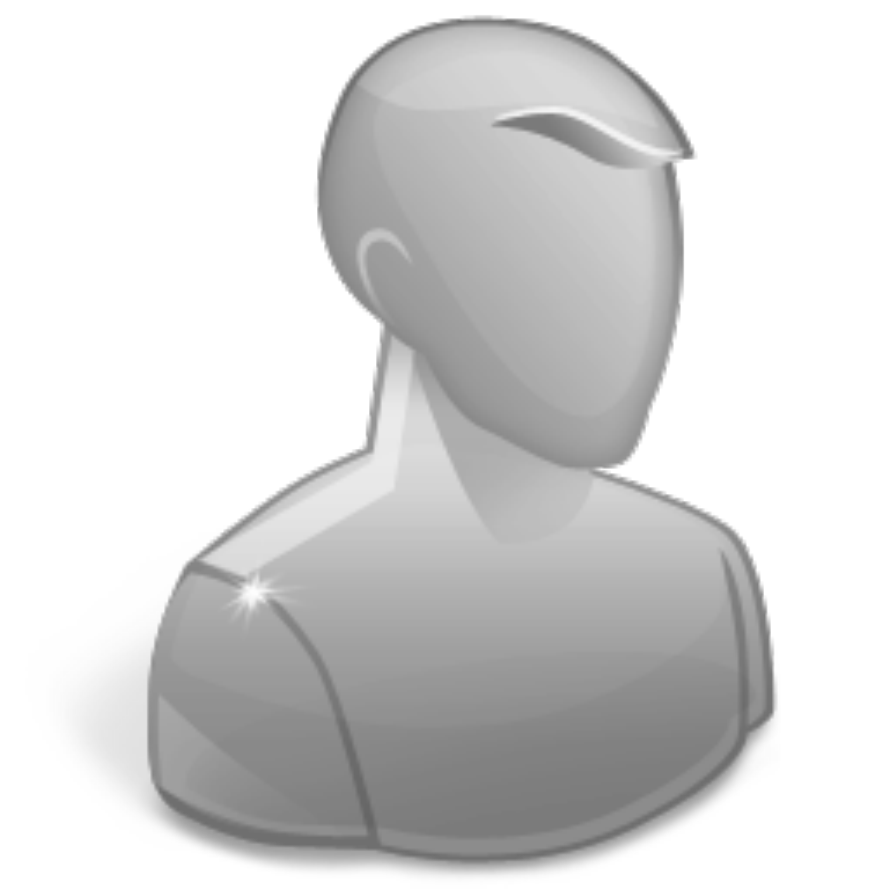}};
 \node[below] at (user.south) (user_label) {$\policyuser$};
 \node[inner sep=-3pt, below=0.75cm of user] (robot)
     {\includegraphics[width=.11\textwidth]{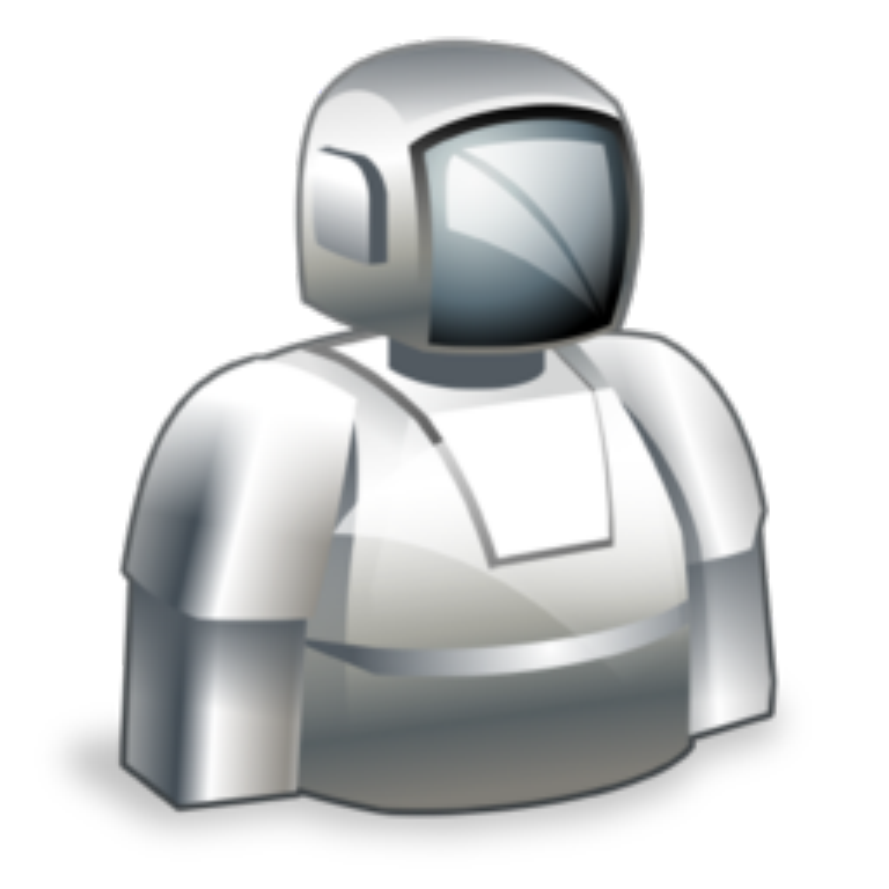}};
 \node[below] at (robot.south) (robot_label) {$\policyrobot$};

 \def \plotoriginx {2.0}
 \def \plotoriginy {2.1}
 \def \plotlength {1.5}
 \coordinate (plot_origin) at (\plotoriginx, \plotoriginy);
 \coordinate (plot_arb_1) at (\plotoriginx + \plotlength/2, \plotoriginy+\plotlength*2/3);
 \coordinate (plot_arb_2) at (\plotoriginx + \plotlength, \plotoriginy+\plotlength*2/3);
 \draw (plot_origin) -- node[below] {\small Confidence} ++ (\plotlength,0);
 \draw (plot_origin) -- node[above, rotate=90] (plot_vert) {\small $\arbitration$} ++ (0,\plotlength);
 \draw[color=blue, thick] (plot_origin) -- (plot_arb_1);
 \draw[color=blue, thick] (plot_arb_1) -- (plot_arb_2);

 \node [draw,rectangle,minimum width=\plotlength cm,minimum height=\plotlength cm,anchor=south west, opacity=0., label={[label distance=0.1cm]above: Arbitration}] (plot_frame) at (plot_origin) {};

 \node[inner sep=0pt,right=2.3cm of plot_frame] (robot_arm)
     {\includegraphics[trim=0 85 0 0, width=.10\textwidth, clip=true]{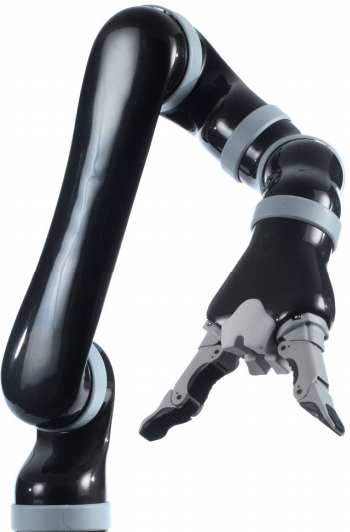}};

 \draw[->, thick] (user) -- node[above] {$\actionuser$} ([xshift=-0.2cm, yshift=-0.2cm]plot_frame.north west);
 \draw[->, thick] (robot) -- node[below] {$\actionrobot$}  ([xshift=-0.2cm, yshift=0.2cm]plot_frame.south west);
 \draw[->, thick] ([xshift=0.1cm]plot_frame.east) -- node[above] {\small $\arbitration a + (1-\arbitration)u$}  (robot_arm);

 \end{tikzpicture}
 \caption{Blend method for shared control teleoperation. The user and robot are both modelled as separate policies $\policyuser$ and $\policyrobot$, each independently providing actions $\actionuser$ and $\actionrobot$ for a single goal. These actions are combined through a specified arbitration function, which generally uses some confidence measure to augment the magnitude of assistance. This combined action is executed on the robot.}
 \label{fig:blend_diagram}
\end{figure}
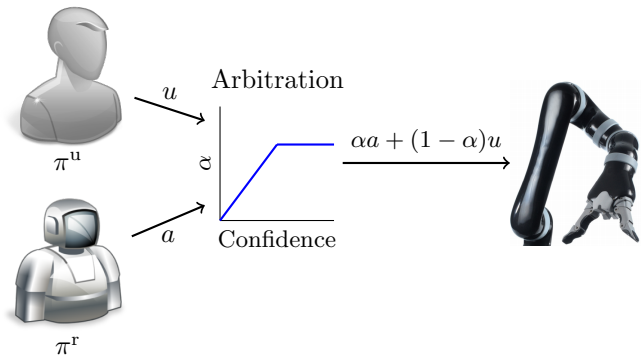

\begin{figure}
  \begin{tikzpicture} 

 \def \distbetween {1.4}

  \node[inner sep=-2pt] (user) at (0,0)
      {\includegraphics[width=.11\textwidth]{user_icon.pdf}};
  \node[below] at (user.south) (user_label) {$\policyuser$};
  \node[inner sep=-3pt, right=\distbetween cm of user] (robot)
      {\includegraphics[width=.11\textwidth]{robot_icon.pdf}};
  \node[below] at (robot.south) (robot_label) {$\policyrobot$};

  \node[below] at (user_label.south) (user_cost_label) {$\costuser(\staterobgoal, \actionuser)$};
  \node[below] at (robot_label.south) (robot_cost_label) {$\costrobot(\staterobgoal, \actionuser, \actionrobot)$};

  \def \shiftarmim {0.2}
  \node[inner sep=0pt,right=\distbetween cm of robot, yshift=-\shiftarmim cm] (robot_arm)
      {\includegraphics[trim=0 85 0 00, width=.10\textwidth, clip=true]{mico_arm_editted.png}};

  \draw[->, thick] (user)  -- node[above] {$\actionuser$} (robot);
  \draw[->, thick] (robot)  -- node[above] {$\actionuser, \actionrobot$} ([yshift=\shiftarmim cm]robot_arm.west);

  \end{tikzpicture}
 \caption{Policy method for shared control teleoperation. The user is modelled as a policy $\policyuser$, which selects user input $\actionuser$ to minimizes the expected sum of user costs $\costuser(\stateenv, \actionuser)$. The user input $\actionuser$ is provided to the system policy $\policyrobot$, which then selects action $\actionrobot$ to minimize its expected sum of costs cost $\costrobot(\state, \actionuser, \actionrobot)$. Both actions are passed to the robot for execution. Unlike the blend method, the user and robot actions are not treated separately, which can lead to catastrophic failure~\citep{trautman_2015}. Instead, the robot action $\actionrobot$ is optimized given the user action $\actionuser$.}
  \label{fig:policy_diagram}
\end{figure}

\subsection{Human-Robot Teaming}
\label{sec:related_collaboration}


In human-robot teaming, robot action selection that models and optimizes for the human teammate leads to better collaboration. \citet{hoffman_2007} show that using predictions of a human collaborator during action selection led to more efficient task completion and more favorable perception of robot contribution to team success. \citet{lasota_2015} show that planning to avoid portions of the workspace the user will occupy led to faster task completion, less user and robot idling time, greater user satisfaction, and greater perceived safety and comfort. \citet{arai_2010} show that users feel high mental strain when a robot collaborator moves too close or too quickly.

Motion planners have been augmented to include user models and collaboration constraints. For static users, researchers have incorporated collaboration constraints such as safety and social acceptability~\citep{sisbot_2007}, and task constraints such as user visibility and reachability~\citep{sisbot_2010, pandey_2010, mainprice_2011}. For moving users, \citet{mainprice_2013} use a Gaussian mixture model to predict user motion, and select a robot goal that avoids the predicted user locations.



Similar ideas have been used to avoid moving pedestrians. \citet{ziebart_2009} learn a predictor of pedestrian motion, and use this to predict the probability a location will be occupied at each time step. They build a time-varying cost map, penalizing locations likely to be occupied, and optimize trajectories for this cost. \citet{chung_2011} use A* search to predict pedestrian motions, including a model of uncertainty, and plan paths using these predictions. \citet{bandy_2012} use fixed models for pedestrian motions, and focus on utilizing a POMDP framework with SARSOP~\citep{kurniawati_2008} for selecting good actions. Like our approach, this enables them to reason over the entire distribution of potential goals. They show this outperforms utilizing only the maximum likelihood estimate of goal prediction for avoidance. 

Others develop methods for how the human-robot team should be structured. \citet{gombolay_2014} study the effects of having the user and robot assign goals to each other. They find that users were willing to cede decision making to the robot if it resulted in greater team fluency~\citep{gombolay_2014}. However, \citet{gombolay_2017} later show that having the autonomous entity assign goals led to less situational awareness. Inspired by training schemes for human-human teams, \citet{stefanos_2013} present a human-robot cross training method, where the user and robot iteratively switch roles to learn a shared plan. Their model leads to greater team fluency, more concurrent motions, greater perceived robot performance, and greater user trust. \citet{koppula_2013} use conditional random fields to predict the user goal (e.g. grasp cup), and have a robot achieve a complementary goal (e.g. pour water into cup).


Others have studied how robot motions can influence the belief of users. \citet{sisbot_2010} fix the gaze of the robot on its goal to communicate intent. \citet{dragan_2013_legible} incorporate legibility into the motion planner for a robotic arm, causing the robot to exaggerate its motion to communicate intent. They show this leads to more quickly and accurately predicting the robot intent~\citep{dragan_2013_legible_hri}. \citet{rezvani_2016} study the effects of conveying a robot's state (e.g. confidence in action selection, anomaly in detection) directly on a user interface for autonomous driving. 

Recent works have gone one step further, selecting robot actions that not only change the perceptions of users, but also their actions. \citet{nikolaidis_2017_mutual} model how likely users are to adopt the robot's policy based on robot actions. They utilize a POMDP to simultaneously learn this user adaptability while steering users to more optimal goals to achieve greater reward. \citet{nikolaidis_2017_game} present a more general game theoretic approach where users change their actions based on robot actions, while not completely adopting the robot's policy. Similarly, \citet{sadigh_2016} generate motions for an autonomous car using predictions of how other drivers will respond, enabling them to change the behavior of other users, and infer the internal user state~\citep{sadigh_iros2016}.

Teaming with an autonomous agent has also been studied outside of robotics. \citet{fern_2010} have studied MDPs and POMDPs for interactive assistants that suggest actions to users, who then accept or reject each action. They show that optimal action selection even in this simplified model is PSPACE-complete. However, a simple greedy policy has bounded regret. \citet{nguyen_2011} and \citet{macindoe_2012} apply POMDPs to cooperative games, where autonomous agents simultaneously infer human intentions and take assistance actions. Like our approach, they model users as stochastically optimizing an MDP, and solve for assistance actions with a POMDP. In contrast to these works, our state and action spaces are continuous.

\subsection{User Prediction}
\label{sec:related_prediction}


A variety of models and methods have been used for intent prediction. Hidden markov model (HMM) based methods~\citep{li_2003,kragic_2005, aarno_2005_virtualfixtures, aarno_2008} predict subtasks or intent during execution, treating the intent as latent state. \citet{schrempf_2007} use a Bayesian network constructed with expert knowledge. \citet{koppula_2013} extend conditional random fields (CRFs) with object affordance to predict potential human motions. \citet{wang_2013_intentioninference} learn a generative predictor by extending Gaussian Process Dynamical Models (GPDMs) with a latent variable for intention. \citet{hauser_2013} utilizes a Gaussian mixture model over task types (e.g. reach, grasp), and predicts both the task type and continuous parameters for that type (e.g. movements) using Gaussian mixture autoregression.


Many successful works in shared autonomy utilize of maximum entropy inverse optimal control (MaxEnt IOC)~\citep{ziebart_2008} for user goal prediction. Briefly, the user is modelled as a stochastic policy approximately optimizing some cost function. By minimizing the worst-case predictive loss, \citet{ziebart_2008} derive a model where trajectory probability decreases exponentially with cost. They then derive a method for inferring a distribution over goals from user inputs, where probabilities correspond to how efficiently the inputs achieve each goal~\citep{ziebart_2009}. A key advantage of this framework for shared autonomy is that the we can directly optimize for the cost function used to model the user.

Exact, global inference over these distributions is computationally infeasible in continuous state and action spaces. Instead, \citet{levine_2012} provide a method that considers the expert demonstrations as only locally optimal, and utilize Laplace's method about the expert demonstration to estimate the log likelihood during learning. Similarly, \citet{dragan_2013_assistive} use Laplace's method about the optimal trajectory between any two points to approximate the distribution over goals during shared control teleoperation. \citet{finn_2016} simultaneously learn a cost function and policy consistent with user demonstrations using deep neural networks, utilizing importance sampling to approximate inference with few samples. Inspired by Generative Adversarial Nets~\citep{goodfellow_2014}, \citet{ho_2016} directly learn a policy to mimic the user through Generative Adversarial Imitation Learning.

We use the approximation of \citet{dragan_2013_assistive} in our framework due to empirical evidence of effectiveness in shared autonomy systems~\citep{dragan_2013_assistive, muelling_2015}.

\section{Framework}
\label{sec:framework}

\graphicspath{{./}{./images_rss_2015/}{./images_hri_2016/}}


We present our framework for minimizing a cost function for shared autonomy with an unknown user goal. We assume the user's goal is fixed, and they take actions to achieve that goal without considering autonomous assistance. These actions are used to predict the user's goal based on how optimal the action is for each goal (\cref{sec:framework_prediction}). Our system uses this distribution to minimize the expected cost-to-go  (\cref{sec:framework_unknown_goal}). As solving for the optimal action is infeasible, we use hindsight optimization to approximate a solution (\cref{sec:framework_hindsight}). For reference, see \cref{table:variable_definitions} in \cref{sec:variable_definitions} for variable definitions.

\subsection{Cost minimization with a known goal}
\label{sec:framework_known_goal}

We first formulate the problem for a known user goal, which we will use in our solution with an unknown goal. We model this problem as a Markov Decision Process (MDP). 

Formally, let $\stateenv \in \Stateenv$ be the environment state (e.g. human and robot pose). Let $\actionuser \in \Actionuser$ be the user actions, and $\actionrobot \in \Actionrobot$ the robot actions. Both agents can affect the environment state - if the user takes action $\actionuser$ and the robot takes action $\actionrobot$ while in state $\stateenv$, the environment stochastically transitions to a new state $\stateenv'$ through $\transitionallargs$. 

We assume the user has an intended goal $\goal \in \Goal$ which does not change during execution. We augment the environment state with this goal, defined by $\state = \left(\stateenv, \goal\right) \in \Stateenv \times \Goal$. We overload our transition function to model the transition in environment state without changing the goal, $\transition( (\stateenv', g) \given (\stateenv, g), \actionuser, \actionrobot) = \transitionallargs$.

We assume access to a user policy for each goal $\policyuser(\actionuser \given \state) = \policyusergoal(\actionuser \given \stateenv) = p(\actionuser \given \stateenv, \goal)$. We model this policy using the maximum entropy inverse optimal control (MaxEnt IOC) framework of~\citet{ziebart_2008}, where the policy corresponds to stochastically optimizing a cost function $\costuser(\state, \actionuser) = \costusergoal(\stateenv, \actionuser)$. We assume the user selects actions based only on $\state$, the current environment state and their intended goal, and does not model any actions that the robot might take. Details are in \cref{sec:framework_prediction}.

The robot selects actions to minimize a cost function dependent on the user goal and action $\costrobot(\state, \actionuser, \actionrobot) = \costrobotgoal(\stateenv, \actionuser, \actionrobot)$. At each time step, we assume the user first selects an action, which the robot observes before selecting $\actionrobot$. The robot selects actions based on the state and user inputs through a policy $\policyrobot(\actionrobot \given \state, \actionuser) = p(\actionrobot \given \state, \actionuser)$. We define the value function for a robot policy $\vrobot^{\policyrobot}$ as the expected cost-to-go from a particular state, assuming some user policy $\policyuser$:
\begin{align*}
  \vrobot^{\policyrobot}(\state) &= \expctarg{\sumtime \costrobot(\state_t, \actionuser_t, \actionrobot_t) \given \state_0 = \state}\\
  \actionuser_t &\sim \policyuser(\cdot \given \state_t)\\
  \actionrobot_t &\sim \policyrobot(\cdot \given \state_t, \actionuser_t)\\
  \state_{t+1} &\sim \transition(\cdot \given \state_t, \actionuser_t, \actionrobot_t)
\end{align*}

The optimal value function $\vopt$ is the cost-to-go for the best robot policy:
\begin{align*}
  \vopt(\state) &= \min_{\policyrobot} \vrobot^{\policyrobot}(\state)
\end{align*}

The action-value function $\qopt$ computes the immediate cost of taking action $\actionrobot$ after observing $\actionuser$, and following the optimal policy thereafter:
\begin{align*}
  \qopt(\stateactions) &= \costrobot(\stateactions) + \expctarg{\vopt(\state')}
\end{align*}
Where $\state' \sim \transition(\cdot \given \stateactions)$. The optimal robot action is given by $\argmin_\actionrobot \qopt(\stateactions)$.

In order to make explicit the dependence on the user goal, we often write these quantities as:
\begin{align*}
  \vgoal(\stateenv) &= \vopt(\state)\\
  \qgoal(\stateenvactions) &= \qopt(\stateactions)
\end{align*}

Computing the optimal policy and corresponding action-value function is a common objective in reinforcement learning. We assume access to this function in our framework, and describe our particular implementation in the experiments.

%

\subsection{Cost Minimization with an unknown goal}
\label{sec:framework_unknown_goal}

We formulate the problem of minimizing a cost function with an unknown user goal as a Partially Observable Markov Decision Process (POMDP). A POMDP maps a distribution over states, known as the \emph{belief} $\belief$, to actions. We assume that all uncertainty is over the user's goal, and the environment state is known. This subclass of POMDPs, where uncertainty is constant, has been studied as a Hidden Goal MDP~\citep{fern_2010}, and as a POMDP-lite~\citep{chen_2016}.

In this framework, we infer a distribution of the user's goal by observing the user actions $\actionuser$. Similar to the known-goal setting (\cref{sec:framework_known_goal}), we define the value function of a belief as:
\begin{align*}
  \vrobot^{\policyrobot}(\belief) &= \expctarg{\sumtime \costrobot(\state_t, \actionuser_t, \actionrobot_t)  \given \belief_0 = \belief} \\
  \state_t &\sim \belief_t\\
  \actionuser_t &\sim \policyuser(\cdot \given \state_t)\\
  \actionrobot_t &\sim \policyrobot(\cdot \given \state_t, \actionuser_t)\\
  \belief_{t+1} &\sim \transitionbelief(\cdot \given \belief_t, \actionuser_t, \actionrobot_t)
\end{align*}
Where the belief transition $\transitionbelief$ corresponds to transitioning the known environment state $\stateenv$ according to $\transition$, and updating our belief over the user's goal as described in $\cref{sec:framework_prediction}$. We can define quantities similar to above over beliefs:
\begin{align}
  \vopt(\belief) &= \min_{\policyrobot} \vrobot^{\policyrobot}(\belief) \label{eq:v_belief}\\
  \qopt(\beliefactions) &= \expctarg{\costrobot(\belief, \actionuser, \actionrobot) + \expctover{\belief'}{\vopt(\belief')}} \nonumber
\end{align}

\subsection{Hindsight Optimization}
\label{sec:framework_hindsight}

Computing the optimal solution for a POMDP with continuous states and actions is generally intractable. Instead, we approximate this quantity through \emph{Hindsight Optimization}~\citep{chong_2000,yoon_2008}, or QMDP~\citep{littman_1995}. This approximation estimates the value function by switching the order of the min and expectation in \cref{eq:v_belief}:
\begin{align*}
  \vhs(\belief) &= \expctover{\belief}{\min_{\policyrobot} \vrobot^{\policyrobot}(\state)}\\
  &= \expctover{\goal}{\vgoal(\stateenv)}\\
  \qhs(\beliefactions) &= \expctover{\belief}{\costrobot(\stateactions) + \expctover{\state'}{\vhs(\state')}}\\
  &= \expctover{\goal}{\qgoal(\stateenvactions)}
\end{align*}

Where we explicitly take the expectation over $\goal \in \Goal$, as we assume that is the only uncertain part of the state.

Conceptually, this approximation corresponds to assuming that all uncertainty will be resolved at the next timestep, and computing the optimal cost-to-go. As this is the best case scenario for our uncertainty, this is a lower bound of the cost-to-go, $\vhs(\belief) \leq \vopt(\belief)$. Hindsight optimization has demonstrated effectiveness in other domains~\citep{yoon_2007, yoon_2008}. However, as it assumes uncertainty will be resolved, it never explicitly gathers information~\citep{littman_1995}, and thus performs poorly when this is necessary.

We believe this method is suitable for shared autonomy for many reasons. Conceptually, we assume the user provides inputs at all times, and therefore we gain information without explicit information gathering. Works in other domains with similar properties have shown that this approximation performs comparably to methods that consider explicit information gathering~\citep{koval_2014}. Computationally, computing $\qhs$ can be done with continuous state and action spaces, enabling fast reaction to user inputs. 



Computing $\qgoal$ for shared autonomy requires utilizing the user policy $\policyusergoal$, which can make computation difficult. This can be alleviated with the following approximations:
\subsubsection*{Stochastic user with robot}
Estimate $\actionuser$ using $\policyusergoal$ at each time step, e.g. by sampling, and utilize the full cost function $\costrobotgoal(\stateenvactions)$ and transition function $\transitionallargs$ to compute $\qgoal$. This would be the standard QMDP approach for our POMDP.

\subsubsection*{Deterministic user with robot}
Estimate $\actionuser$ as the most likely $\actionuser$ from $\policyusergoal$ at each time step, and utilize the full cost function $\costrobotgoal(\stateenvactions)$ and transition function $\transitionallargs$ to compute $\qgoal$. This uses our policy predictor, as above, but does so deterministically, and is thus more computationally efficient.

\subsubsection*{Robot takes over}
Assume the user will stop supplying inputs, and the robot will complete the task. This enables us to use the cost function $\costrobotgoal(\stateenv, 0, \actionrobot)$ and transition function $\transition(\stateenv' \given \stateenv, 0, \actionrobot)$ to compute $\qgoal$. For many cost functions, we can analytically compute this value, e.g. cost of always moving towards the goal at some velocity. An additional benefit of this method is that it makes no assumptions about the user policy $\policyusergoal$, making it more robust to modelling errors. We use this method in our experiments.

Finally, as we often cannot calculate $\argmax_{\actionrobot} \qhs(\beliefactions)$ directly, we use a first-order approximation, which leads to us to following the gradient of $\qhs(\beliefactions)$.



\subsection{User Prediction}
\label{sec:framework_prediction}

In order to infer the user's goal, we rely on a model $\policyusergoal$ to provide the distribution of user actions at state $\stateenv$ for user goal $\goal$. In principle, we could use any generative predictor for this model, e.g.~\citep{koppula_2013, wang_2013_intentioninference}. We choose to use maximum entropy inverse optimal control (MaxEnt IOC)~\citep{ziebart_2008}, as it explicitly models a user cost function $\costusergoal$. We optimize this directly by defining $\costrobotgoal$ as a function of $\costusergoal$.

In this work, we assume the user does not model robot actions. We use this assumption to define an MDP with states $\stateenv \in \Stateenv$ and user actions $\actionuser \in \Actionuser$ as before, transition $\transitionuser(\stateenv' \given \stateenv, \actionuser) = \transition(\stateenv' \given \stateenv, \actionuser, 0)$, and cost $\costusergoal(\stateenv, \actionuser)$. MaxEnt IOC computes a stochastically optimal policy for this MDP.

The distribution of actions at a single state are computed based on how optimal that action is for minimizing cost over a horizon $T$. Define a sequence of environment states and user inputs as $\traj = \left\{ \stateenv_0, \actionuser_0, \cdots, \stateenv_T, \actionuser_T \right\}$. Note that sequences are not required to be trajectories, in that $\stateenv_{t+1}$ is not necessarily the result of applying $\actionuser_t$ in state $\stateenv_t$. Define the cost of a sequence as the sum of costs of all state-input pairs, $\costgoaluser(\traj) = \sum_{t} \costgoaluser(\stateenv_t, \actionuser_t)$. Let $\trajtot$ be a sequence from time $0$ to $t$, and $\trajat{\stateenv}$ a sequence of from time $t$ to $T$, starting at $\stateenv$.

\citet{ziebart_thesis} shows that minimizing the worst-case predictive loss results in a model where the probability of a sequence decreases exponentially with cost, $p(\traj \given \goal) \propto \exp(-\costgoaluser(\traj))$. Importantly, one can efficiently learn a cost function consistent with this model from demonstrations~\citep{ziebart_2008}.

Computationally, the difficulty in computing $p(\traj \given \goal)$ lies in the normalizing constant $\int_{\traj} \exp(-\costgoaluser(\traj))$, known as the partition function. Evaluating this explicitly would require enumerating all sequences and calculating their cost. However, as the cost of a sequence is the sum of costs of all state-action pairs, dynamic programming can be utilized to compute this through soft-minimum value iteration when the state is discrete~\citep{ziebart_2009,ziebart_2012}:
\begin{align*}
  \qgoalsoftt{t}(\stateenv, \actionuser) &= \costgoaluser(\stateenv, \actionuser) + \expctarg{\vgoalsoftt{t+1}(\stateenv')}\\
  \vgoalsoftt{t}(\stateenv) &= \softmin_{\actionuser} \qgoalsoftt{t}(\stateenv, \actionuser)
\end{align*}
Where $\softmin_{x} f(x) = - \log \int_{x} \exp(-f(x)) dx$ and $\stateenv' \sim \transitionuser(\cdot \given \stateenv, \actionuser)$.

The log partition function is given by the soft value function, $\vgoalsoftt{t}(\stateenv) = - \log \int_{\trajat{\stateenv}} \exp\left(-\costgoaluser(\trajat{\stateenv})\right)$, where the integral is over all sequences starting at $\stateenv$ and time $t$. Furthermore, the probability of a single input at a given environment state is given by $\policyuser_t(\actionuser \given \stateenv, \goal) = \exp(\vgoalsoftt{t}(\stateenv) -\qgoalsoftt{t}(\stateenv, \actionuser))$~\citep{ziebart_2009}.

Many works derive a simplification that enables them to only look at the start and current states, ignoring the inputs in between~\citep{ziebart_2012, dragan_2013_assistive}. Key to this assumption is that $\traj$ corresponds to a trajectory, where applying action $\actionuser_t$ at $\stateenv_t$ results in $\stateenv_{t+1}$. However, if the system is providing assistance, this may not be the case. In particular, if the assistance strategy believes the user's goal is $\goal$, the assistance strategy will select actions to minimize $\costusergoal$. Applying these simplifications will result positive feedback, where the robot makes itself more confident about goals it already believes are likely. In order to avoid this, we ensure that the prediction comes from user inputs only, and not robot actions:
\begin{align*}
  p(\traj \given \goal) &= \prod_t \policyuser_t(\actionuser_{t} \given \stateenv_t, \goal)
\end{align*}
To compute the probability of a goal given the partial sequence up to $t$, we apply Bayes' rule:
\begin{align*}
  p(\goal \given \trajtot) &= \frac{p(\trajtot \given \goal) p(\goal) }{\sum_{\goal'} p(\trajtot \given \goal') p(\goal')}
\end{align*}
This corresponds to our POMDP observation model, used to transition our belief over goals through $\transitionbelief$.

\subsubsection{Continuous state and action approximation}
Soft-minimum value iteration is able to find the exact partition function when states and actions are discrete. However, it is computationally intractable to apply in continuous state and action spaces. Instead, we follow \citet{dragan_2013_assistive} and use a second order approximation about the optimal trajectory. They show that, assuming a constant Hessian, we can replace the difficult to compute soft-min functions $\vgoalsoft$ and $\qgoalsoft$ with the min value and action-value functions $\vgoaluser$ and $\qgoaluser$:
\begin{align*}
  \policyuser_t(\actionuser \given \stateenv, \goal) &= \exp(\vgoaluser(\stateenv) -\qgoaluser(\stateenv, \actionuser))
\end{align*}
Recent works have explored extensions of the MaxEnt IOC model for continuous spaces~\citep{boularias_2011, levine_2012, finn_2016}. We leave experiments using these methods for learning and prediction as future work.

\subsection{Multi-Target MDP}
\label{sec:framework_multitarget}

There are often multiple ways to achieve a goal. We refer to each of these ways as a \emph{target}. For a single goal (e.g. object to grasp), let the set of targets (e.g. grasp poses) be $\target \in \Target$. We assume each target has a cost function $\costtarg$, from which we compute the corresponding value and action-value functions $\vtarg$ and $\qtarg$, and soft-value functions $\vtargsoft$ and $\qtargsoft$. We derive the quantities for goals, $\vgoal, \qgoal, \vgoalsoft, \qgoalsoft$, as functions of these target functions.

We state the theorems below, and provide proofs in the appendix (\cref{sec:mingoal_thms}).

\subsubsection{Multi-Target Assistance}
\label{sec:framework_multigarget_assistance}
We assign the cost of a state-action pair to be the cost for the target with the minimum cost-to-go after this state:
\begin{align}
  \costgoal(\stateenvactions) &= \costtargstar(\stateenvactions) \quad \target* = \argmin_\target \vtarg(\stateenv') \label{eq:goal_target_cost}
\end{align}
Where $\stateenv'$ is the environment state after actions $\actionuser$ and $\actionrobot$ are applied at state $\stateenv$. For the following theorem, we require that our user policy be deterministic, which we already assume in our approximations when computing robot actions in \cref{sec:framework_hindsight}.
\begin{restatable}{theorem}{valfundecompose}
\label{thm:mingoal_assist}
Let $\vtarg$ be the value function for target $\target$. Define the cost for the goal as in \cref{eq:goal_target_cost}. For an MDP with deterministic transitions, and a deterministic user policy $\policyuser$, the value and action-value functions $\vgoal$ and $\qgoal$ can be computed as:
\begin{align*}
  \qgoal(\stateenvactions) &= \qtargstar(\stateenvactions) \qquad \target^* = \argmin_\target \vtarg(\stateenv') \\
  \vgoal(\stateenv) &= \min_\target \vtarg(\stateenv)
\end{align*}
\end{restatable}

\subsubsection{Multi-Target Prediction}
\label{sec:framework_multigarget_prediction}
Here, we don't assign the goal cost to be the cost of a single target $\costtarg$, but instead use a distribution over targets.
\begin{restatable}{theorem}{softvalfundecompose}
  \label{thm:mingoal_pred}
  Define the probability of a trajectory and target as $p(\traj, \target) \propto \exp(-\costtarg(\traj))$. Let $\vtargsoft$ and $\qtargsoft$ be the soft-value functions for target $\target$. For an MDP with deterministic transitions, the soft value functions for goal $\goal$, $\vgoalsoft$ and $\qgoalsoft$, can be computed as:
\begin{align*}
  \vgoalsoft(\stateenv) &= \softmin_\target \vtargsoft(\stateenv)\\
  \qgoalsoft(\stateenv, \actionuser) &= \softmin_\target \qtargsoft(\stateenv, \actionuser)
\end{align*}
\end{restatable}

%
%
%

\begin{figure}[t]
\centering
 \begin{subfigure}{0.24\textwidth}
   \centering 
   \includegraphics[width=0.97\textwidth, trim=440 250 500 210, clip=true]{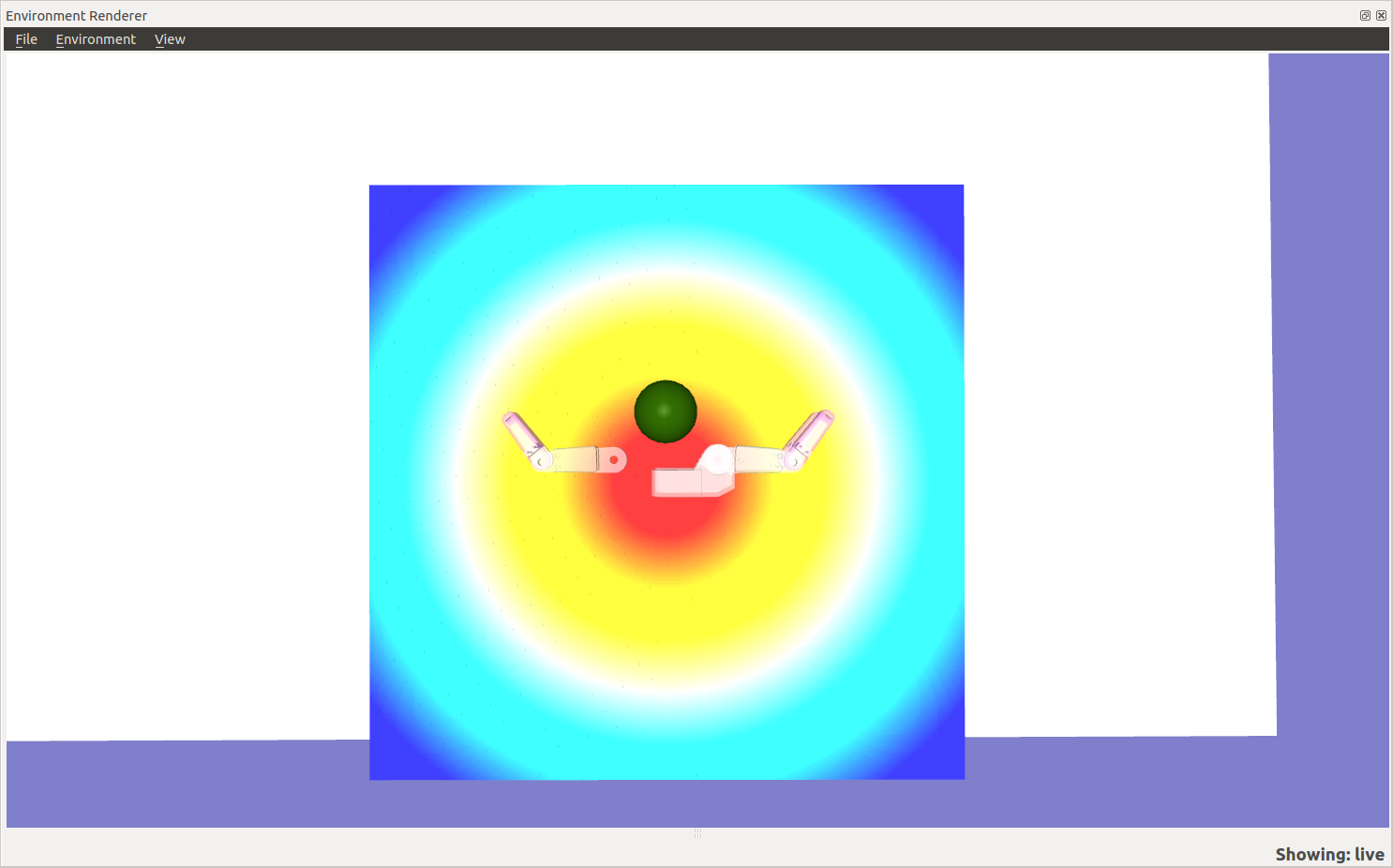}
  \caption{}
 \label{fig:multigoal_1}
 \end{subfigure}
 \begin{subfigure}{0.24\textwidth}
   \centering 
   \includegraphics[width=0.97\textwidth, trim=440 250 500 210, clip=true]{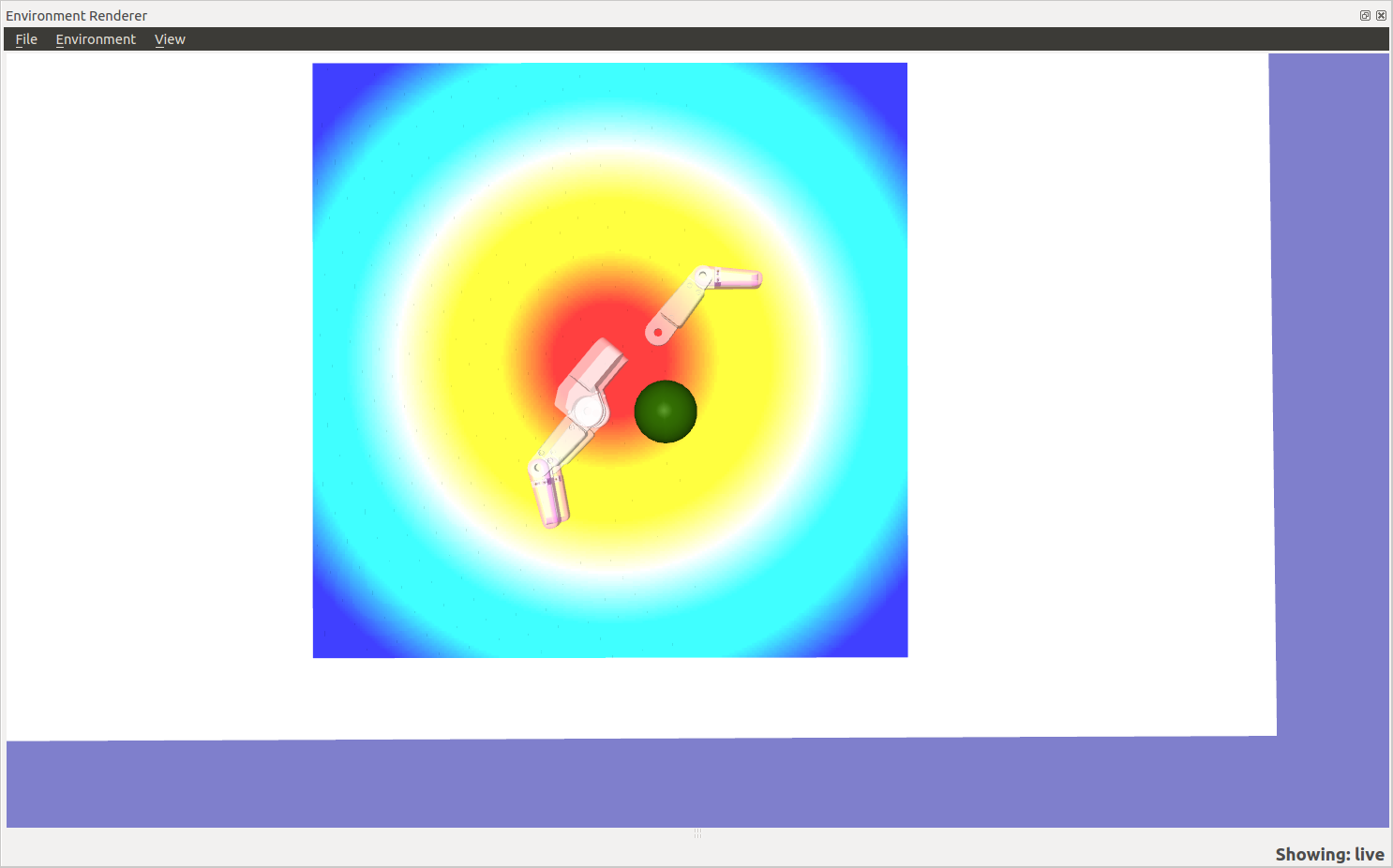}
  \caption{}
 \label{fig:multigoal_2}
 \end{subfigure}
 \begin{subfigure}{0.24\textwidth}
   \centering 
   \includegraphics[width=0.97\textwidth, trim=440 250 500 210, clip=true]{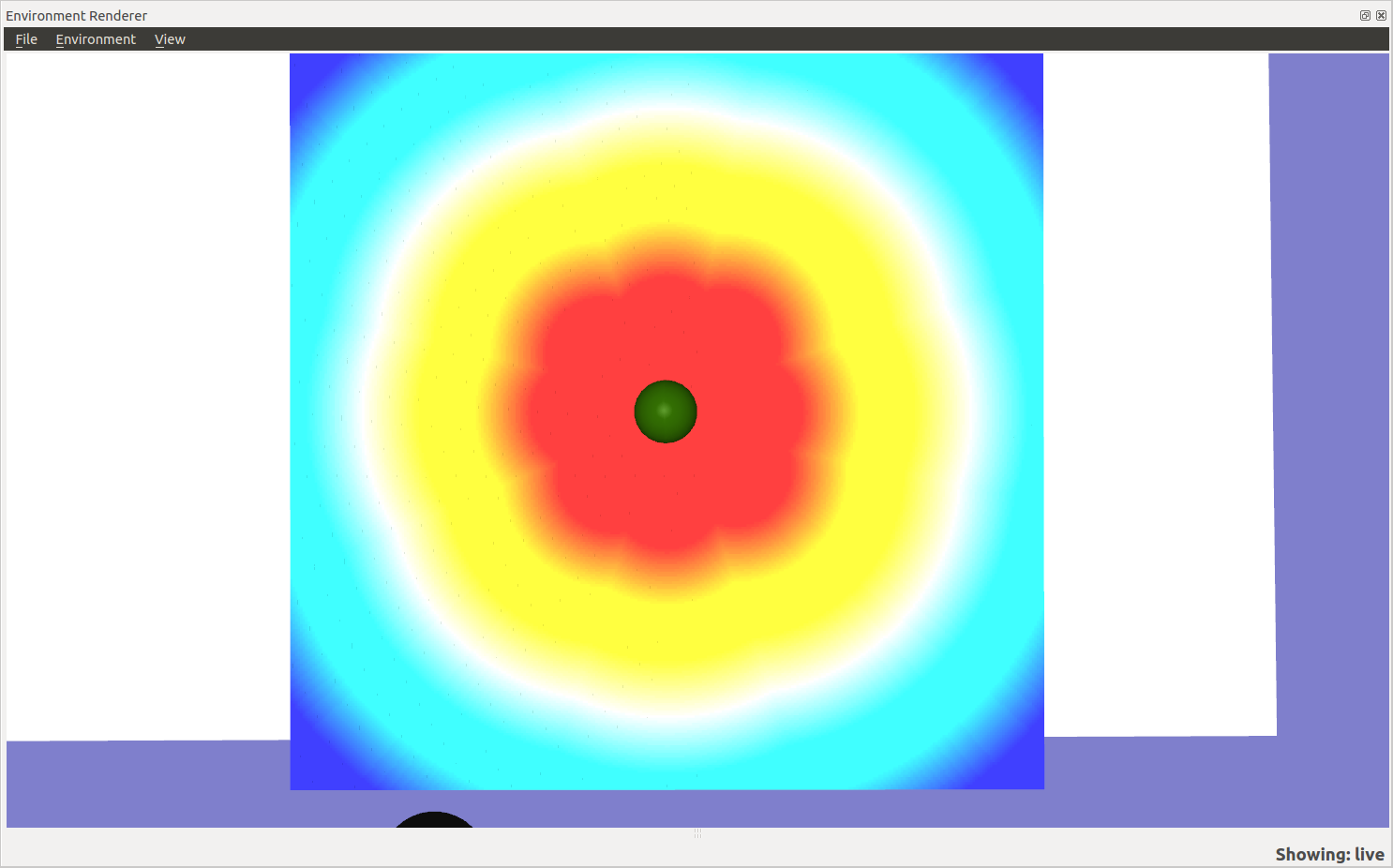}
  \caption{}
 \label{fig:multigoal_3_arb}
 \end{subfigure}
 \begin{subfigure}{0.24\textwidth}
   \centering 
   \includegraphics[width=0.97\textwidth, trim=440 250 500 210, clip=true]{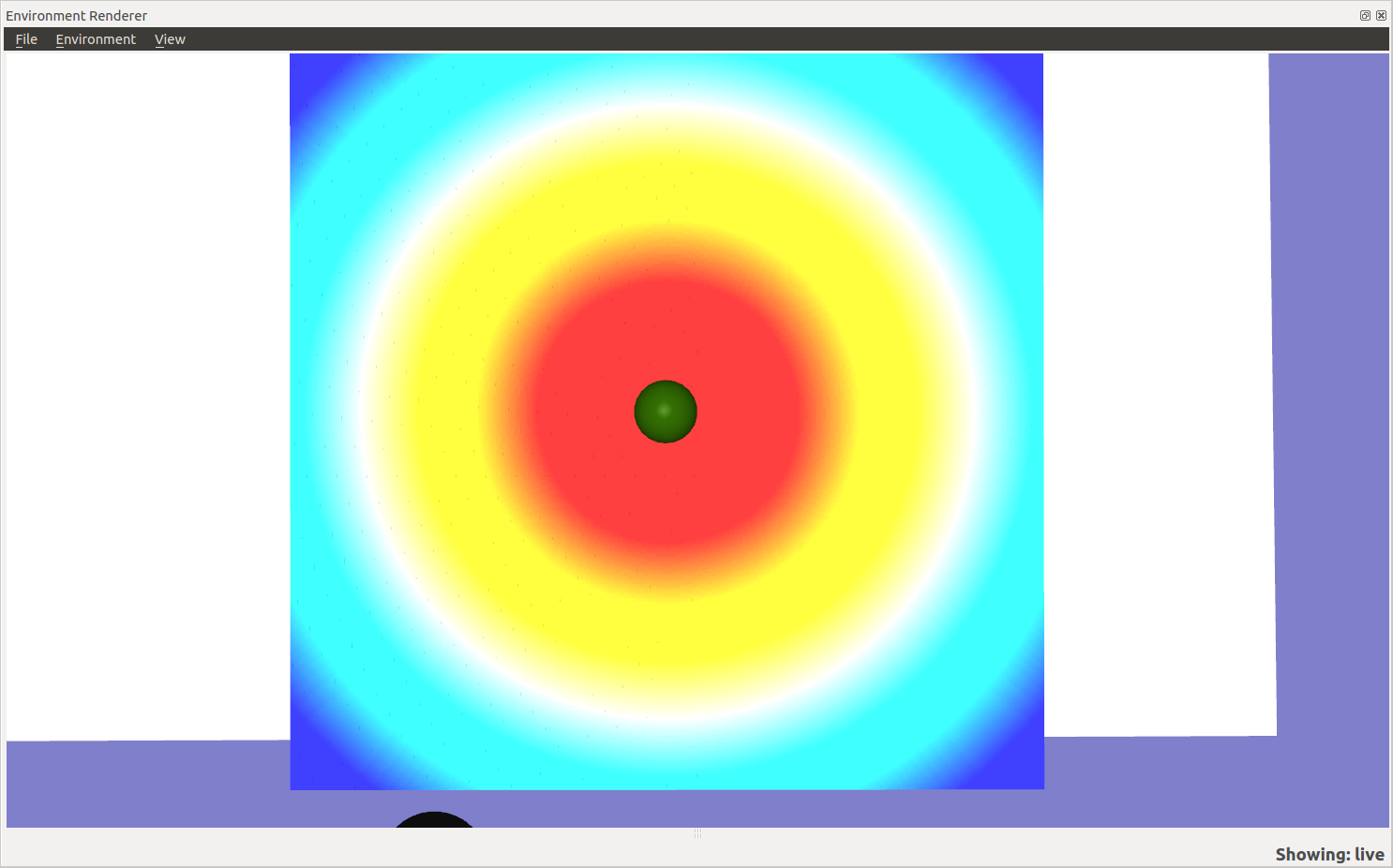}
  \caption{}
 \label{fig:multigoal_3_pred}
 \end{subfigure}
 \caption{Value function for a goal (grasp the ball) decomposed into value functions of targets (grasp poses). (\subref{fig:multigoal_1}, \subref{fig:multigoal_2}) Two targets and their corresponding value function $\vtarg$. In this example, there are 16 targets for the goal. (\subref{fig:multigoal_3_arb}) The value function of a goal $\vgoal$ used for assistance, corresponding to the minimum of all 16 target value functions (\subref{fig:multigoal_3_pred}) The soft-min value function $\vgoalsoft$ used for prediction, corresponding to the soft-min of all 16 target value functions.}
 \label{fig:multigoal}
\end{figure}

\section{Shared Control Teleoperation}
\label{sec:shared_teleop}

\graphicspath{{./}{./images_rss_2015/}}

\begin{figure*}[t]
\centering
  \begin{subfigure}{0.32\textwidth}
    \centering 
    \begin{tikzpicture}[every node/.style={anchor=south west,inner sep=0pt}, x=1mm, y=1mm,]    
      \node {\includegraphics[width=1.0\textwidth, trim=250 150 200 190, clip=true]{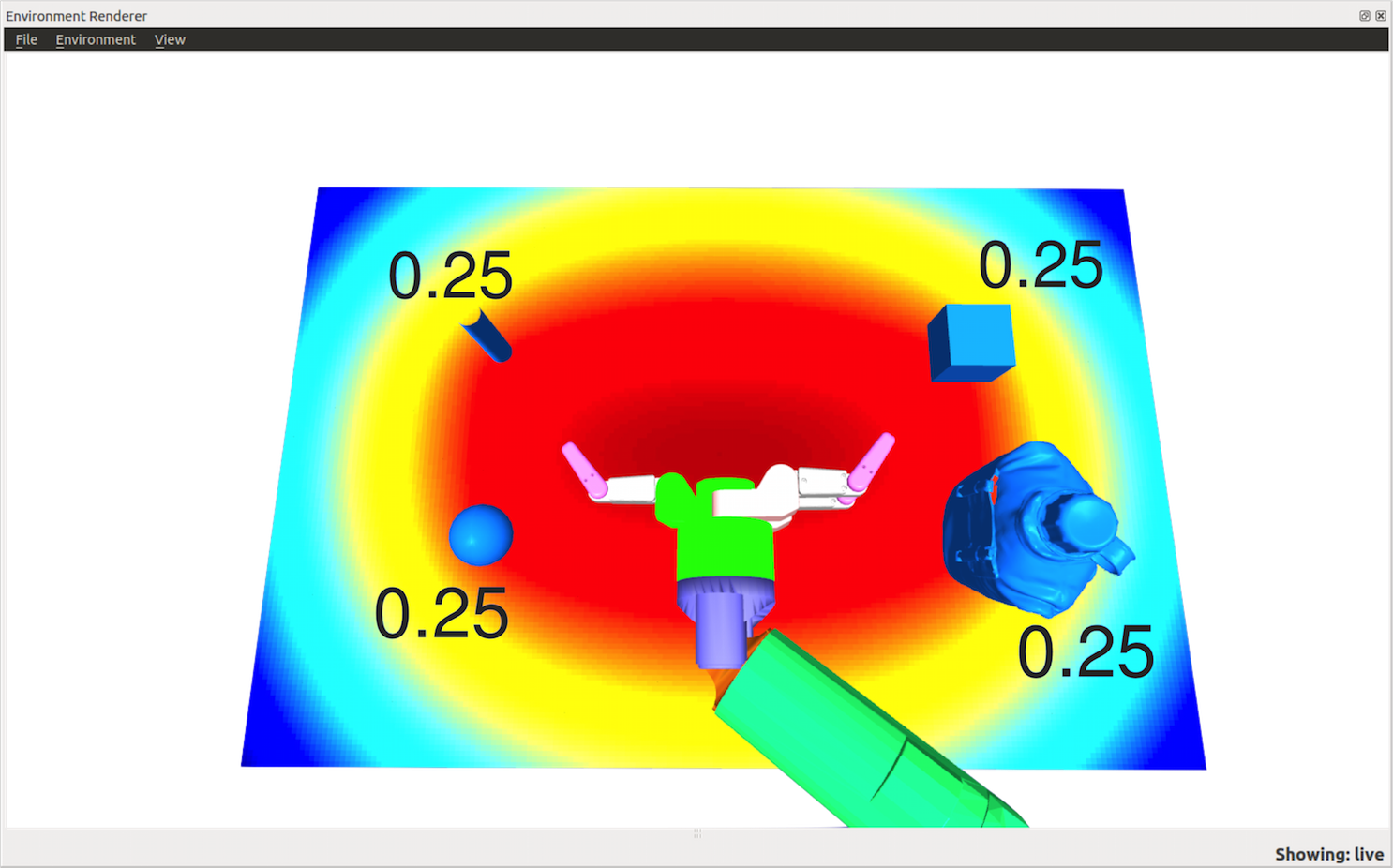}} ;
      \node [opacity=0.6]{\includegraphics[width=1.0\textwidth, trim=250 150 200 190, clip=true]{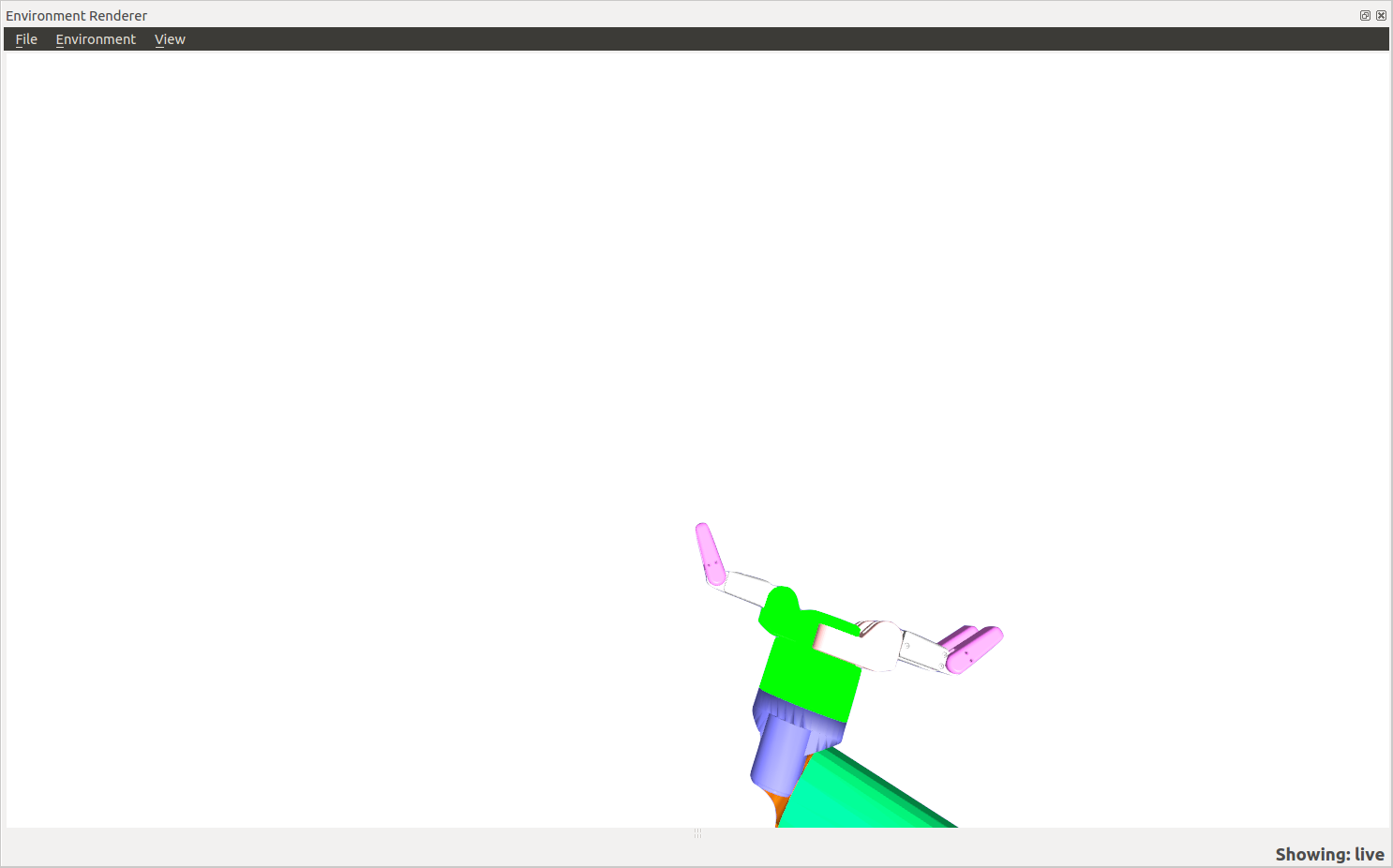} };
    \end{tikzpicture}
    \includegraphics[width=0.7\textwidth, trim=150 515 640 150, clip=true]{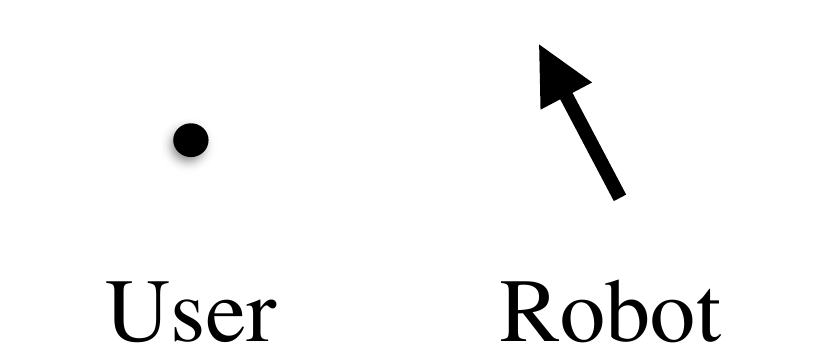}
    \caption{ }
  \label{fig:valfunc_2}
 \end{subfigure}
 \begin{subfigure}{0.32\textwidth}
   \centering 
   \begin{tikzpicture}[every node/.style={anchor=south west,inner sep=0pt}, x=1mm, y=1mm,]    
     \node {\includegraphics[width=1.0\textwidth, trim=250 150 200 190, clip=true]{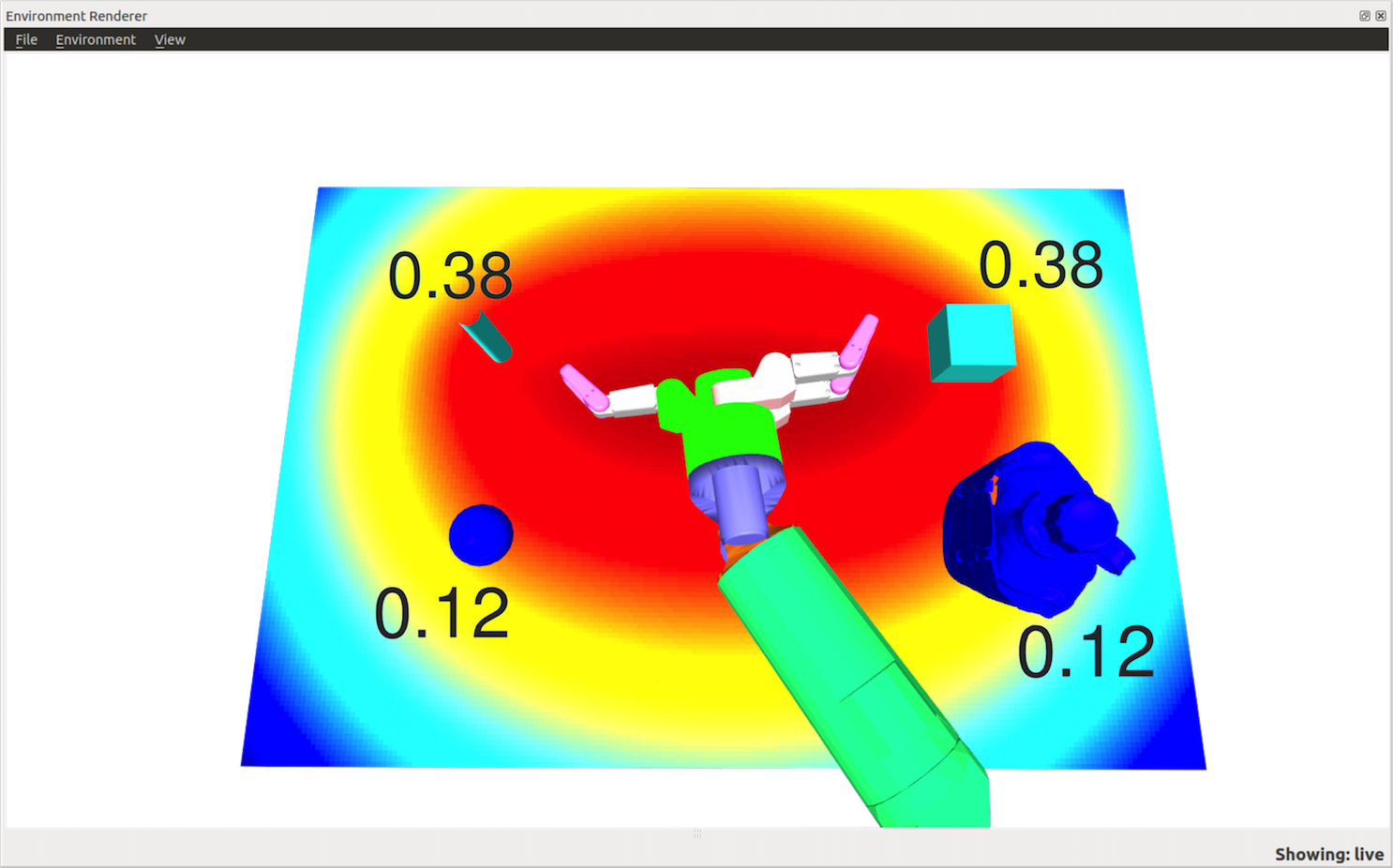}} ;
     \node [opacity=0.6]{\includegraphics[width=1.0\textwidth, trim=250 150 200 190, clip=true]{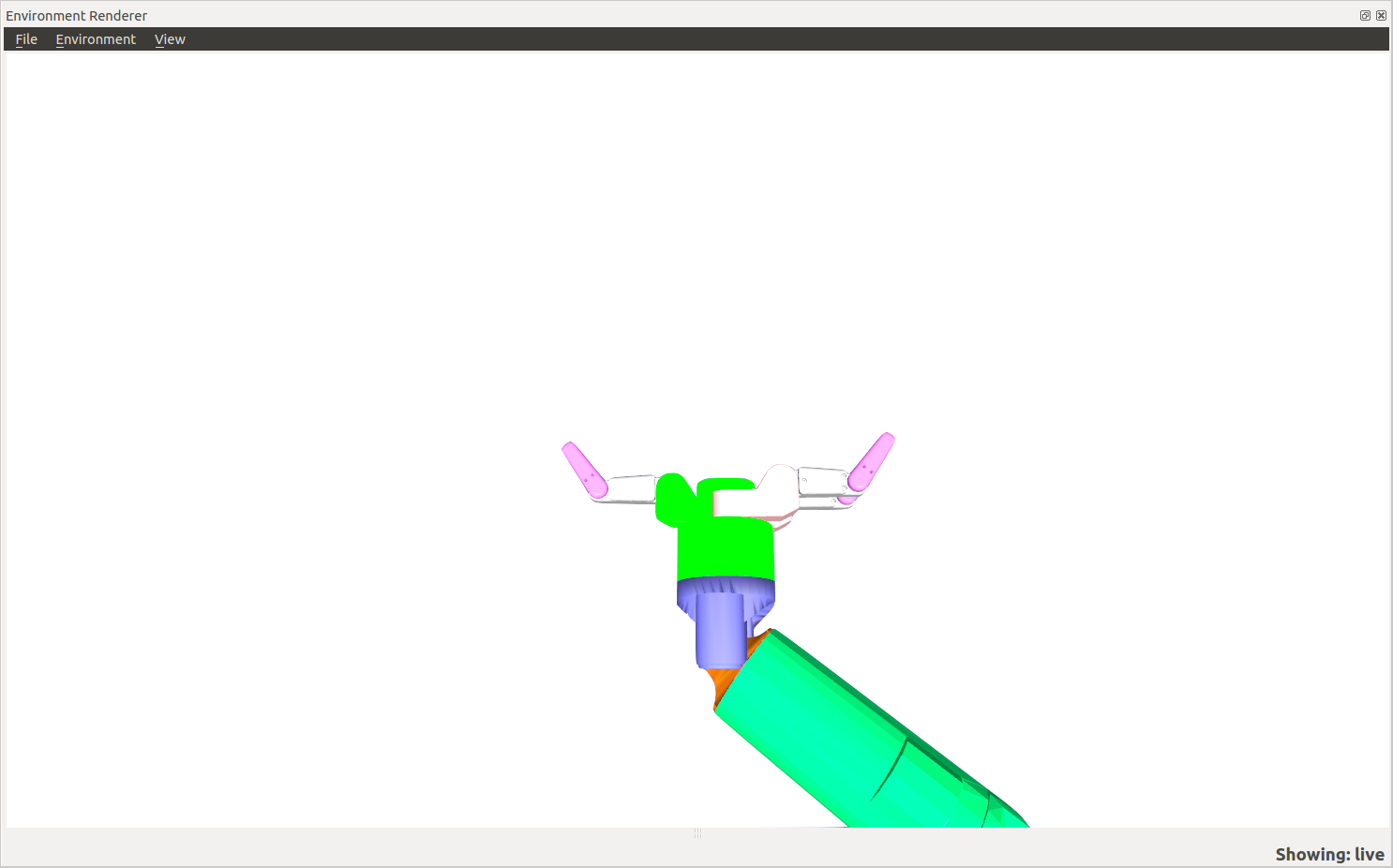} };
    \end{tikzpicture}
    \includegraphics[width=0.7\textwidth, trim=150 515 640 150, clip=true]{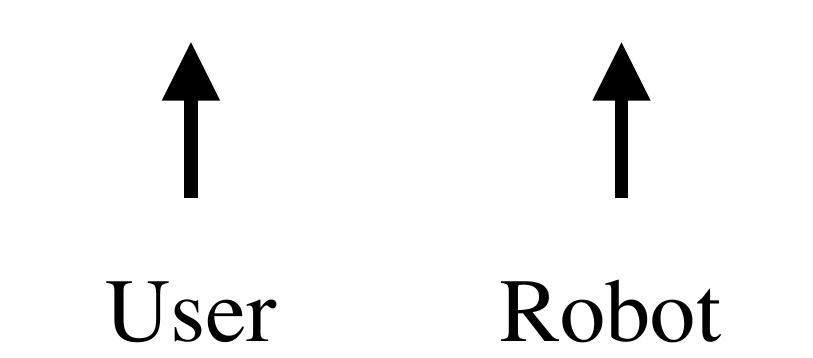}
 \caption{ }
 \label{fig:valfunc_3}
 \end{subfigure}
 \begin{subfigure}{0.32\textwidth}
   \centering 
   \begin{tikzpicture}[every node/.style={anchor=south west,inner sep=0pt}, x=1mm, y=1mm,]    
     \node {\includegraphics[width=1.0\textwidth, trim=250 150 200 190, clip=true]{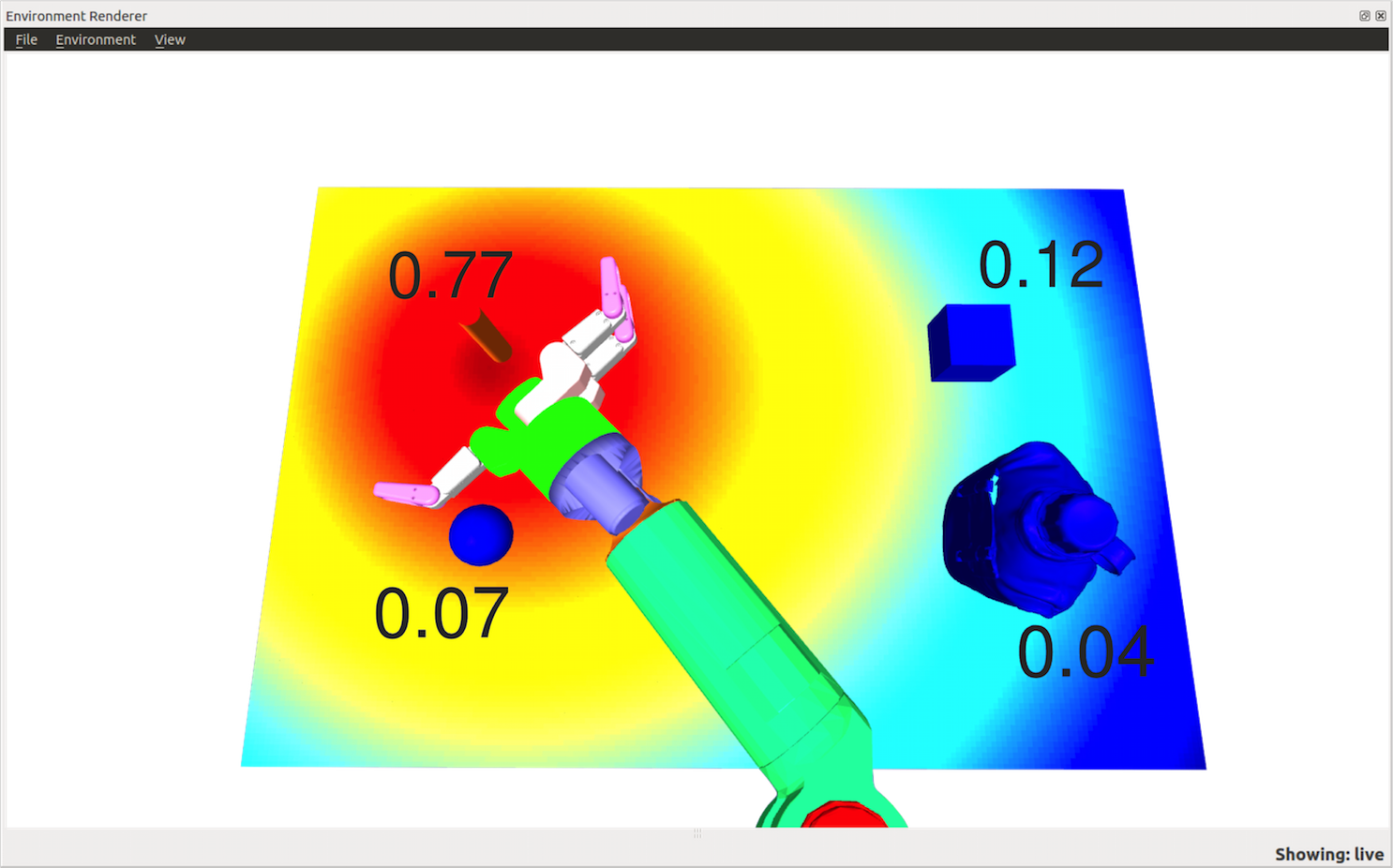}} ;
     \node [opacity=0.6]{\includegraphics[width=1.0\textwidth, trim=250 150 200 190, clip=true]{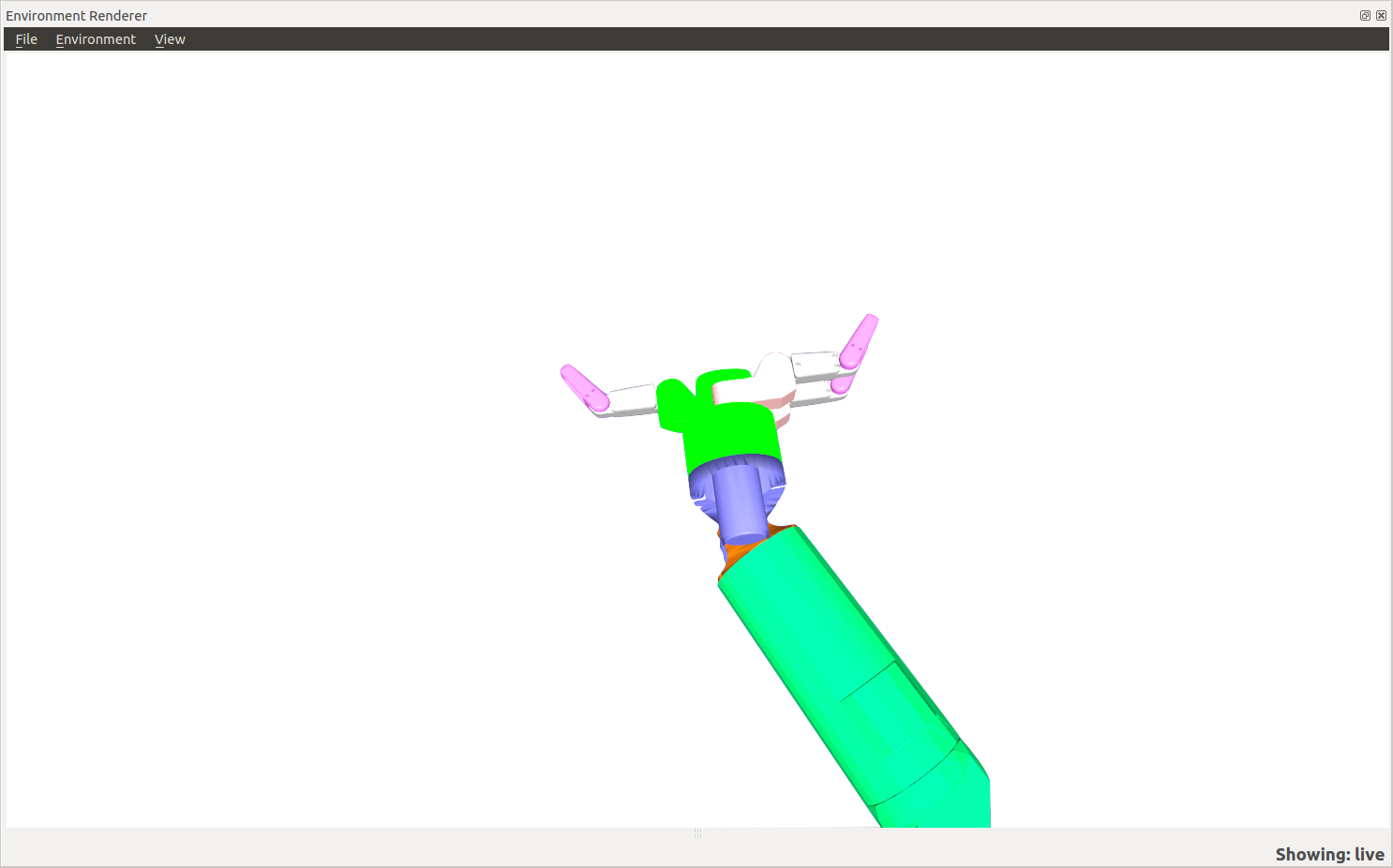} };
 \end{tikzpicture}
  \includegraphics[width=0.7\textwidth, trim=150 515 640 150, clip=true]{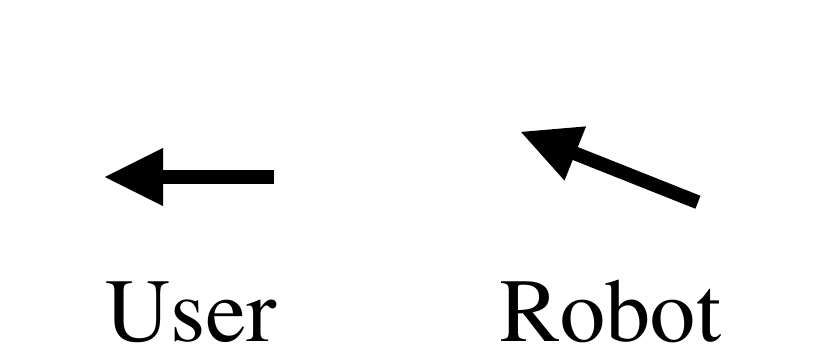}
 \caption{ }
 \label{fig:valfunc_4}
 \end{subfigure}
  \label{fig:valfunc}
  \caption{Estimated goal probabilities and value function for object grasping. Top row: the probability of each goal object and a 2-dimensional slice of the estimated value function. The transparent end-effector corresponds to the initial state, and the opaque end-effector to the next state. Bottom row: the user input and robot control vectors which caused this motion. (\subref{fig:valfunc_2}) Without user input, the robot automatically goes to the position with lowest value, while estimated probabilities and value function are unchanged. (\subref{fig:valfunc_3}) As the user inputs ``forward'', the end-effector moves forward, the probability of goals in that direction increase, and the estimated value function shifts in that direction. (\subref{fig:valfunc_4}) As the user inputs ``left'', the goal probabilities and value function shift in that direction. Note that as the probability of one object dominates the others, the system automatically rotates the end-effector for grasping that object.}
\end{figure*}

We apply our shared autonomy framework to two shared control teleoperation tasks: a simpler task of object grasping (\cref{sec:experiment_rss_2015}) and a more complicated task of feeding (\cref{sec:experiment_hri_2016}). Formally, the state $\stateenv$ corresponds to the end-effector pose of the robot, each goal $\goal$ an object in the world, and each target $\target$ a pose for achieving that goal (e.g. pre-grasp pose). The transition function $\transitionallargs$ deterministically transitions the state by applying both $\actionuser$ and $\actionrobot$ as end-effector velocities. We map user joystick inputs to $\actionuser$ as if the user were controlling the robot through direct teleoperation.

For both tasks, we hand-specify a simple user cost function, $\costtarguser$, from which everything is derived. Let $d$ be the distance between the robot state $\stateenv' = \transitionuser(\stateenv, \actionuser)$ and target $\target$:
 \begin{align*}
   \costtarguser(\stateenv, \actionuser) &= \left\{ \begin{array}{cc} \alpha & d > \delta \\ \frac{\alpha}{\delta} d & d\leq \delta \end{array} \right.
 \end{align*}

That is, a linear cost near a target $(d \leq \delta)$, and a constant cost otherwise. This is based on our observation that users make fast, constant progress towards their goal when far away, and slow down for alignment when near their goal. This is by no means the best cost function, but it does provide a baseline for performance. We might expect, for example, that incorporating collision avoidance into our cost function may enable better performance~\citep{you_2011}. We use this cost function, as it enables closed-form value function computation, enabling inference and execution at 50Hz.

For prediction, when the distance is far away from any target $(d > \delta)$, our algorithm shifts probability towards goals relative to how much progress the user action makes towards the target. If the user stays close to a particular target $(d \leq \delta)$, probability mass automatically shifts to that goal, as the cost for that goal is less than all others.

We set $\costtargrobot(\stateenv, \actionrobot, \actionuser) = \costtarguser(\stateenv, \actionrobot)$, causing the robot to optimize for the user cost function directly\footnote{In our prior work~\citep{javdani_2015_rss}, we used $\costtargrobot(\stateenv, \actionrobot, \actionuser) = \costtarguser(\stateenv, \actionrobot) + (\actionrobot - \actionuser)^2$ in a different framework where only the robot action transitions the state. Both formulations are identical after linearization. Let $\actionrobot^*$ be the optimal optimal robot action in this framework. The additional term $(\actionrobot - \actionuser)^2$ leads to executing the action $\actionuser + \actionrobot^*$, equivalent to first executing the user action $\actionuser$, then $\actionrobot^*$, as in this framework.}, and behave similar to how we observe users behaved. When far away from goals $(d > \delta)$, it makes progress towards all goals in proportion to their probability of being the user's goal. When near a target $(d \leq \delta)$ that has high probability, our system reduces assistance as it approaches the final target pose, letting users adjust the final pose if they wish.


We believe hindsight optimization is a suitable POMDP approximation for shared control teleoperation. A key requirement for shared control teleoperation is efficient computation, in order to make the system feel responsive. With hindsight optimization, we can provide assistance at 50Hz, even with continuous state and action spaces.

The primary drawback of hindsight optimization is the lack of explicit information gathering~\citep{littman_1995}: it assumes all information is revealed at the next timestep, negating any benefit to information gathering. As we assume the user provides inputs at all times, we gain information automatically when it matters. When the optimal action is the same for multiple goals, we take that action. When the optimal action differs, our model gains information proportional to how suboptimal the user action is for each goal, shifting probability mass towards the user goal, and providing more assistance to that goal.

For shared control teleoperation, explicit information gathering would move the user to a location where their actions between goals were maximally different. Prior works suggest that treating users as an oracle is frustrating~\citep{guillory_2011_noise, amershi_2014}, and this method naturally avoids it.

We evaluated this system in two experiments, comparing our POMDP based method, referred to as \emph{policy}, to a conventional predict-then-act approach based on~\citet{dragan_2013_assistive}, referred to as \emph{blend} (\cref{fig:blend_diagram}). In our feeding experiment, we additionally compare to direct teleoperation, referred to as \emph{direct}, and full autonomy, referred to as \emph{autonomy}.

The \emph{blend} baseline of \citet{dragan_2013_assistive} requires estimating the predictor's confidence of the most probable goals, which controls how user action and autonomous assistance are arbitrated (\cref{fig:blend_diagram}). We use the distance-based measure used in the experiments of~\citet{dragan_2013_assistive}, $\text{conf} = \max\left(0, 1-\frac{d}{D}\right)$, where $d$ is the distance to the nearest target, and $D$ is some threshold past which confidence is zero.

\subsection{Grasping Experiment}
\label{sec:experiment_rss_2015}

Our first shared-control teleoperation user study evaluates two methods, our POMDP framework and a predict-then-act blending method~\citep{dragan_2013_assistive}, on the task of object grasping. This task appears broadly in teleoperation systems, appearing in nearly all applications of teleoperated robotic arms. Additionally, we chose this task for its simplicity, evaluating these methods on tasks where direct teleoperation is relatively easy.

\subsubsection{Metrics}
\label{sec:experiment_rss_2015_metrics}

Our experiment aims to evaluate the efficiency and user satisfaction of each method.

\textbf{Objective measures.} We measure the objective efficiency of the system in two ways. \emph{Total execution time} measures how long it took the participant to grasp an object, measuring the effectiveness in achieving the user's goal. \emph{Total joystick input} measures the magnitude of joystick movement during each trial, measuring the user's effort to achieve their goal.

\textbf{Subjective measures.} We also evaluated user satisfaction with the system through through a seven-point Likert scale survey. After using each control method, we asked users to rate if they would \emph{like to use} the method. After using both methods, we asked users which they \emph{preferred}.

\subsubsection{Hypotheses}
\label{sec:hypoths_rss_2015}

Prior work suggests that more autonomy leads to greater efficiency for teleoperated robots~\citep{you_2011, leeper_2012, dragan_2013_assistive, hauser_2013, javdani_2015_rss}. Additionally, prior work indicates that users subjectively prefer more assistance when it leads to more efficient task completion~\citep{you_2011, dragan_2013_assistive}. Based on this, we formulate the following hypotheses:
\newhypothset

\hypothcounter{Participants using the policy method will grasp objects significantly faster than the blend method}{hypoth:grasp_1}

\hypothcounter{Participants using the policy method will grasp objects with significantly less control input than the blend method}{hypoth:grasp_2}

\hypothcounter{Participants will agree more strongly on their preferences for the policy method compared to the blend method}{hypoth:grasp_3}

\subsubsection{Experiment Design}
\label{sec:experiment_rss_2015_design}

\begin{figure}[t]
\centering
\includegraphics[width=0.49\textwidth, trim=150 0 200 0, clip=true]{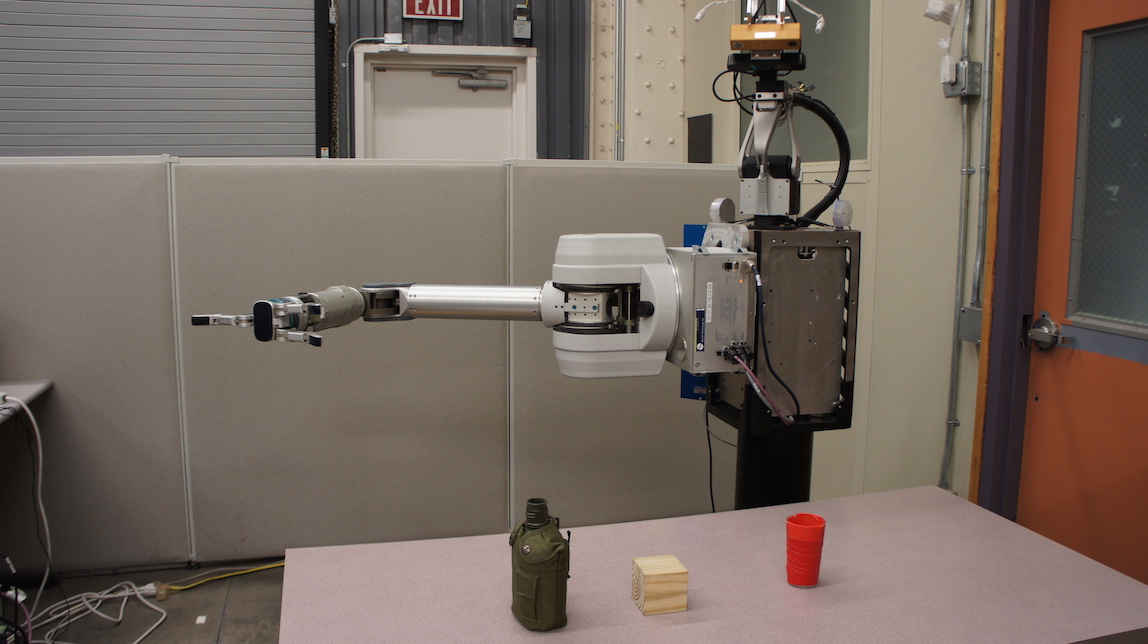}
\caption{Our experimental setup for object grasping. Three objects - a canteen, block, and glass - were placed on the table in front of the robot in a random order. Prior to each trial, the robot moved to the configuration shown. Users picked up each object using each teleoperation system.}
 \label{fig:exper_setup_rss_2015}
\end{figure}

We set up our experiments with three objects on a table: a canteen, a block, and a cup (\cref{fig:exper_setup_rss_2015}). Users teleoperated a robot arm using two joysticks on a Razer Hydra system. The right joystick mapped to the horizontal plane, and the left joystick mapped to the height. A button on the right joystick closed the hand. Each trial consisted of moving from the fixed start pose, shown in \cref{fig:exper_setup_rss_2015}, to the target object, and ended once the hand was closed.

\subsubsection{Procedure}

We conducted a within-subjects study with one independent variable (control method) that had two conditions (policy, blend). We counteract the effects of novelty and practice by counterbalancing the order of conditions. Each participant grasped each object one time for each condition for a total of 6 trials.

We recruited 10 participants (9 male, 1 female), all with experience in robotics, but none with prior exposure to our system. To counterbalance individual differences of users, we chose a within-subjects design, where each user used both systems.

Users were told they would be using two different teleoperation systems, referred to as ``method1'' and ``method2''. Users were not provided any information about the methods. Prior to the recorded trials, users went through a training procedure: First, they teleoperated the robot directly, without any assistance or objects in the scene. Second, they grasped each object one time with each system, repeating if they failed the grasp. Users were then given the option of additional training trials for either system if they wished.

Users then proceeded to the recorded trials. For each system, users picked up each object one time in a random order. Users were told they would complete all trials for one system before the system switched, but were not told the order. However, it was obvious immediately after the first trail started, as the policy method assists from the start pose and blend does not. Upon completing all trials for one system, they were told the system would be switching, and then proceeded to complete all trials for the other system. If users failed at grasping (e.g. they knocked the object over), the data was discarded and they repeated that trial. Execution time and total user input were measured for each trial.

Upon completing all trials, users were given a short survey. For each system, they were asked for their agreement on a 1-7 Likert scale for the following statements:
\begin{enumerate}
  \item ``I felt in \emph{control}''
  \item ``The robot did what I \emph{wanted}''
  \item ``I was able to accomplish the tasks \emph{quickly}''
  \item ``If I was going to teleoperate a robotic arm, I would \emph{like} to use the system''
\end{enumerate}

They were also asked ``which system do you \emph{prefer}'', where $1$ corresponded to blend, $7$ to policy, and $4$ to neutral. Finally, they were asked to explain their choices and provide any general comments. 

\subsubsection{Results}

\begin{figure}[t]
   \includegraphics{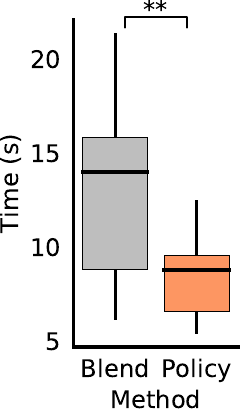}
 \hfill
   \includegraphics[trim=0 0 0 0, clip=true]{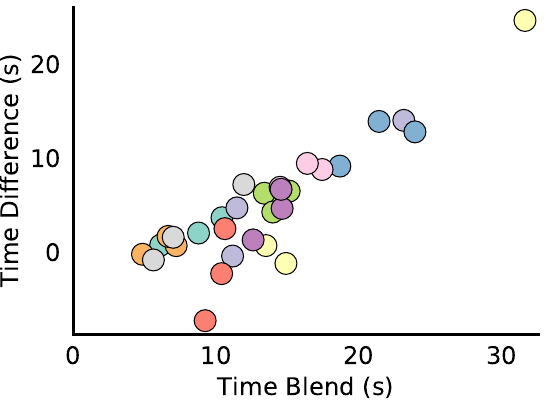}

   \includegraphics{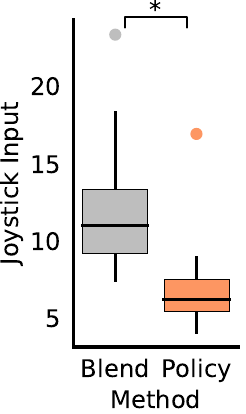}
   \hfill
   \includegraphics[trim=0 0 0 0, clip=true]{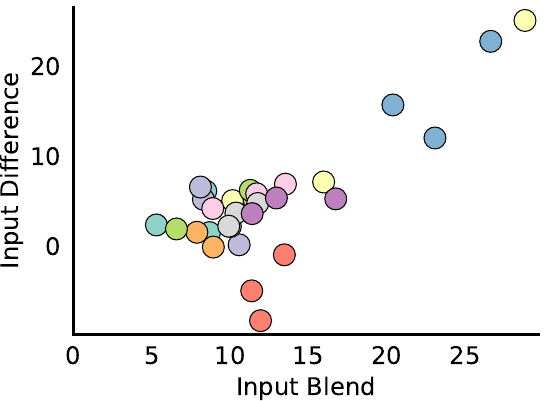}
   \caption{Task completion times and total input for all trials. On the left, box plots for each system. On the right, the time and input of blend minus policy, as a function of the time and total input of blend. Each point corresponds to one trial, and colors correspond to different users. We see that policy was faster $(p<0.01)$ and resulted in less input $(p<0.05)$. Additionally, the difference between systems increases with the time/input of blend.}
 \label{fig:time_control_plots}
 \end{figure}

\begin{figure}[t]
   \centering 
   \begin{subfigure}[b]{0.29\textwidth}
     \centering 
     \includegraphics[trim=0 0 0 0, clip=true]{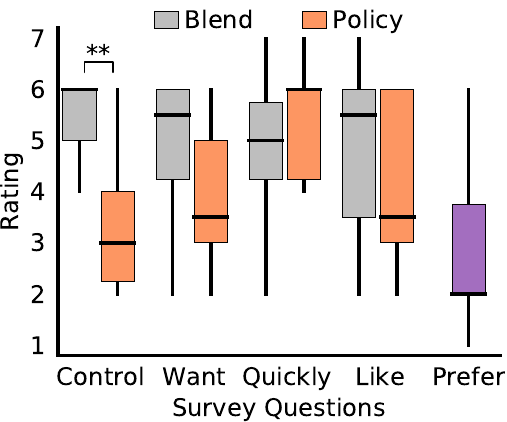}
      \caption{ }
     \label{subfig:survey_means}
   \end{subfigure}
   \hfill
   \begin{subfigure}[b]{0.19\textwidth}
     \centering 
     \includegraphics[trim=0 0 0 0, clip=true]{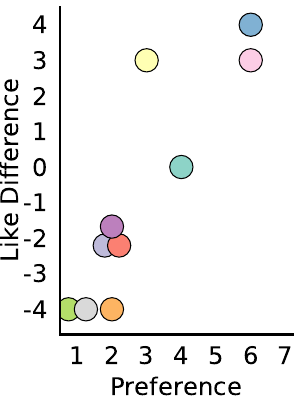}
      \caption{ }
      \label{subfig:prev_vs_like}
   \end{subfigure}
   \caption{(a) Means and standard errors from survey results from our user study. For each system, users were asked if they felt in \emph{control}, if the robot did what they \emph{wanted}, if they were able to accomplish tasks \emph{quickly}, and if they would \emph{like} to use the system. Additionally, they were asked which system they \emph{prefer}, where a rating of 1 corresponds to blend, and 7 corresponds to policy. We found that users agreed with feeling in control more when using the blend method compared to the policy method $(p<0.01)$. (b) The \emph{like} rating of policy minus blend, plotted against the \emph{prefer} rating. When multiple users mapped to the same coordinate, we plot multiple dots around that coordinate. Colors correspond to different users, where the same user has the same color in \cref{fig:time_control_plots}. }
 \label{fig:survey_means}
 \end{figure}

Users were able to successfully use both systems. There were a total of two failures while using each system - once each because the user attempted to grasp too early, and once each because the user knocked the object over. These experiments were reset and repeated.

We assess our hypotheses using a significance level of $\alpha=0.05$. For data that violated the assumption of sphericity, we used a Greenhouse-Geisser correction. If a significant main effect was found, a post-hoc analysis was used to identify which conditions were statistically different from each other, with Holm-Bonferroni corrections for multiple comparisons.

\textbf{Trial times} and \textbf{total control input} were assessed using a two-factor repeated measures ANOVA, using the assistance method and object grasped as factors. Both trial times and total control input had a significant main effect. We found that our policy method resulted in users accomplishing tasks more quickly, supporting \cref{hypoth:grasp_1} $(F(1,9)=12.98, p=0.006)$. Similarly, our policy method resulted in users grasping objects with less input, supporting \cref{hypoth:grasp_2} $(F(1,9)=7.76, p = 0.021)$. See \cref{fig:time_control_plots} for more detailed results.


To assess \textbf{user preference}, we performed a Wilcoxon paired signed-rank test on our survey question asking if they would \emph{like} to use each system, and a Wilcoxon rank-sum test on the survey question of which system they \emph{prefer} against the null hypothesis of no preference (value of 4). There was no evidence to support \cref{hypoth:grasp_3}.

In fact, our data suggests a trend towards the opposite: that users prefer blend over policy. When asked if they would \emph{like} to use the system, there was a small difference between methods (blend: $M=4.90, SD=1.58$, policy: $M=4.10, SD=1.64)$. However, when asked which system they \emph{preferred}, users expressed a stronger preference for blend ($M=2.90, SD=1.76$). While these results are not statistically significant according to our Wilcoxon tests and $\alpha=0.05$, it does suggest a trend towards preferring blend. See \cref{fig:survey_means} for results for all survey questions.


We found this surprising, as prior work indicates a strong correlation between task completion time and user satisfaction, even at the cost of control authority, in both shared autonomy~\citep{dragan_2013_assistive, hauser_2013} and human-robot teaming~\citep{gombolay_2014} settings.\footnote{In prior works where users preferred greater control authority, task completion times were indistinguishable~\citep{kim_2012}.} Not only were users faster, but they recognized they could accomplish tasks more quickly (see \emph{quickly} in \cref{fig:survey_means}). One user specifically commented that  ``[Policy] took more practice to learn\ldots but once I learned I was able to do things a little faster. However, I still don't like feeling it has a mind of its own''.


Users agreed more strongly that they felt in \emph{control} during blend ($Z = -2.687$, $p = 0.007$). Interestingly, when asked if the robot did what they \emph{wanted}, the difference between methods was less drastic. This suggests that for some users, the robot's autonomous actions were in-line with their desired motions, even though the user did not feel that they were in control.

Users also commented that they had to compensate for policy in their inputs. For example, one user stated that ``[policy] did things that I was not expecting and resulted in unplanned motion''. This can perhaps be alleviated with user-specific policies, matching the behavior of particular users.


Some users suggested their preferences may change with better understanding. For example, one user stated they ``disliked (policy) at first, but began to prefer it slightly after learning its behavior. Perhaps I would prefer it more strongly with more experience''. It is possible that with more training, or an explanation of how policy works, users would have preferred the policy method. We leave this for future work.

\subsubsection{Examining trajectories}


\begin{figure}[t]
  \centering
  \begin{subfigure}{0.233\textwidth}
    \centering 
    \includegraphics[trim=0 0 0 0, clip=true]{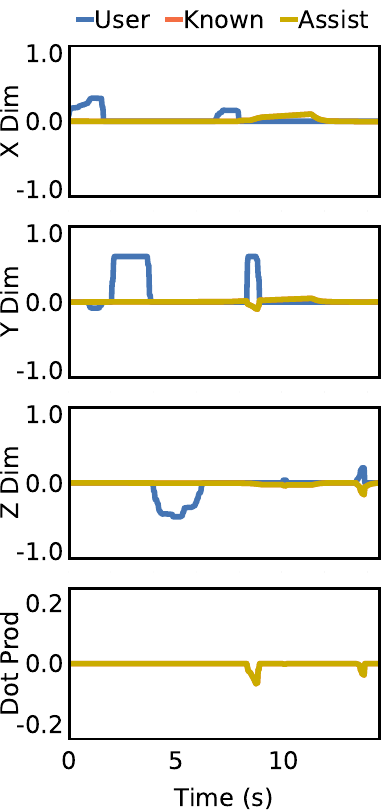}
    \caption{Blend}
    \label{fig:user7_blend}
  \end{subfigure}
  \hfill
  \begin{subfigure}{0.233\textwidth}
    \centering 
    \includegraphics[trim=0 0 0 0, clip=true]{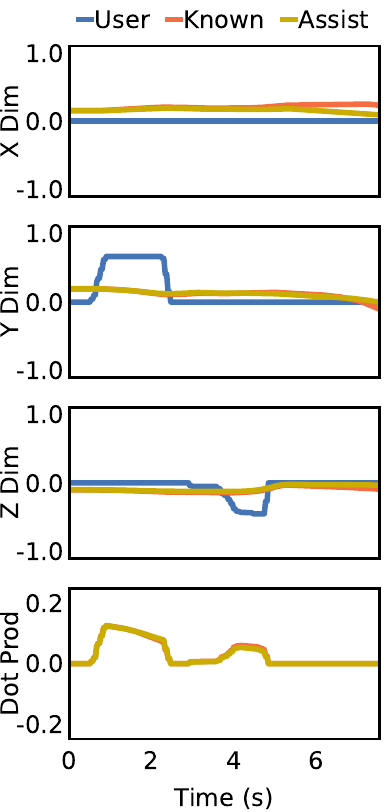}
    \caption{Policy}
    \label{fig:user7_policy}
  \end{subfigure}
  \hfill
  \caption{User input and autonomous actions for a user who preferred policy assistance, using (\subref{fig:user7_blend})~blending and (\subref{fig:user7_policy})~policy for grasping the same object. We plot the user input, autonomous assistance with the estimated distribution, and what the autonomous assistance would have been had the predictor known the true goal. We subtract the user input from the assistance when plotting, to show the autonomous action as compared to direct teleoperation. The top 3 figures show each dimension separately. The bottom shows the dot product between the user input and assistance action. This user changed their strategy during policy assistance, letting the robot do the bulk of the work, and only applying enough input to correct the robot for their goal. Note that this user never applied input in the `X' dimension in this or any of their three policy trials, as the assistance always went towards all objects in that dimension.}
  \label{fig:user7}
\end{figure}

\begin{figure}[t]
  \centering
  \hfill
  \begin{subfigure}{0.233\textwidth}
    \centering 
    \includegraphics[trim=0 0 0 0, clip=true]{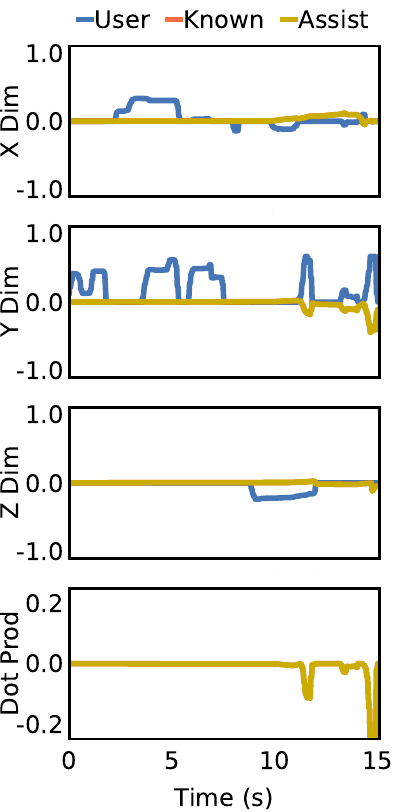}
    \caption{Blend}
    \label{fig:user6_blend}
  \end{subfigure}
  \hfill
  \begin{subfigure}{0.233\textwidth}
    \centering 
    \includegraphics[trim=0 0 0 0, clip=true]{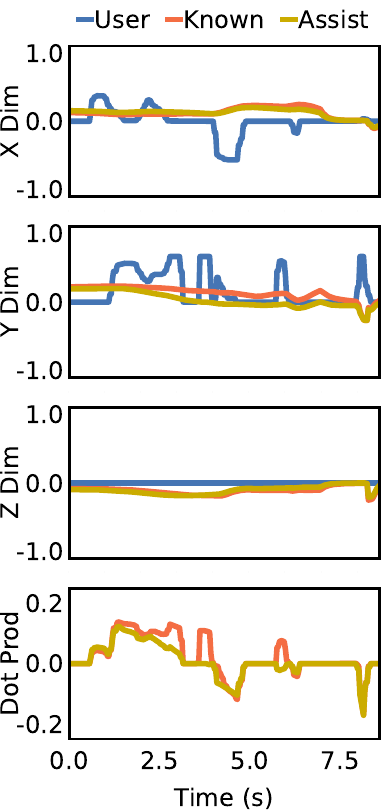}
    \caption{Policy}
    \label{fig:user6_policy}
  \end{subfigure}
  \caption{User input and autonomous assistance for a user who preferred blending, with plots as in \cref{fig:user7}. The user inputs sometimes opposed the autonomous assistance (such as in the `X' dimension) for both the estimated distribution and known goal, suggesting the cost function didn't accomplish the task in the way the user wanted. Even still, the user was able to accomplish the task faster with the autonomous assistance then blending. }
  \label{fig:user6}
\end{figure}

Users with different preferences had very different strategies for using each system. Some users who preferred the assistance policy changed their strategy to take advantage of the constant assistance towards all goals, applying minimal input to guide the robot to the correct goal (\cref{fig:user7}). In contrast, users who preferred blending were often opposing the actions of the autonomous policy (\cref{fig:user6}). This suggests the robot was following a strategy different from their own.

%

\subsection{Feeding Experiment}
\label{sec:experiment_hri_2016}

\graphicspath{{./}{./images_hri_2016/}}


Building from the results of the grasping study (\cref{sec:experiment_rss_2015}), we designed a broader evaluation of our system. In this evaluation, we test our system in an eating task using a Kinova Mico robot manipulator. We chose the Mico robot because it is a commercially available assistive device, and thus provides a realistic testbed for assistive applications. We selected the task of eating for two reasons. First, eating independently is a real need; it has been identified as one of the most important tasks for assistive robotic arms~\citep{chung_2013}. Second, eating independently is hard; interviews with current users of assistive arms have found that people generally do not attempt to use their robot arm for eating, as it requires too much effort~\citep{herlant_2016}. By evaluating our systems on the desirable but difficult task of eating, we show how shared autonomy can improve over traditional methods for controlling an assistive robot in a real-world domain that has implications for people's quality of life.

We also extended our evaluation by considering two additional control methods: direct teleoperation and full robot autonomy. Direct teleoperation is how  assistive robot manipulators like the Mico are currently operated by users. Full autonomy represents a condition in which the robot is behaving ``optimally'' for its own goal, but does not take the user's goal into account. 

Thus, in this evaluation, we conducted a user study to evaluate four methods of robot control---our POMDP framework, a predict-then-act blending method~\citep{dragan_2013_assistive}, direct teleoperation, and full autonomy---in an assistive eating task. 

\subsubsection{Metrics}
\label{sec:experiment_hri_2016_metrics}
Our experiments aim to evaluate the effectiveness and user satisfaction of each method. 

\textbf{Objective measures.} We measure the objective efficiency of the system in four ways. \emph{Success rate} identifies the proportion of successfully completed trials, where success is determined by whether the user was able to pick up their intended piece of food. \emph{Total execution time} measures how long it took the participant to retrieve the food in each trial. \emph{Number of mode switches} identifies how many times participants had to switch control modes during the trial (\cref{fig:control_modes}). \emph{Total joystick input} measures the magnitude of joystick movement during each trial. The first two measures evaluate how effectively the participant could reach their goal, while the last two measures evaluate how much effort it took them to do so.

\textbf{Subjective measures.} We also evaluated user satisfaction with the system through subjective measures. After five trials with each control method, we asked users to respond to questions about each system using a seven point Likert scale. These questions, specified in \cref{sec:experiment_hri_2016_procedure}, assessed user preferences, their perceived ability to achieve their goal, and feeling they were in control. Additionally, after they saw all of the methods, we asked users to \emph{rank order} the methods according to their preference. 

\subsubsection{Hypotheses}
\label{sec:hypoths_feeding}


As in the previous evaluation, we are motivated by prior work that suggests that more autonomy leads to greater efficiency and accuracy for teleoperated robots~\citep{you_2011, leeper_2012, dragan_2013_assistive, hauser_2013, javdani_2015_rss}. We formulate the following hypotheses regarding the efficiency of our control methods, measured through objective metrics.

\newhypothset

\hypothcounter{Using methods with more autonomous assistance will lead to more successful task completions}{hypoth:feeding_1}

\hypothcounter{Using methods with more autonomous assistance will result in faster task completion}{hypoth:feeding_2}

\hypothcounter{Using methods with more autonomous assistance will lead to fewer mode switches}{hypoth:feeding_3}

\hypothcounter{Using methods with more autonomous assistance will lead to less joystick input}{hypoth:feeding_4}

Feeding with an assistive arm is difficult~\citep{herlant_2016}, and prior work indicates that users subjectively prefer more assistance when the task is difficult even though they have less control~\citep{you_2011, dragan_2013_assistive}. Based on this, we formulate the following hypotheses regarding user preferences, measured through our subjective metrics:

\hypothcounter{Participants will more strongly agree on feeling in control for methods with less autonomous assistance}{hypoth:feeding_5}

\hypothcounter{Participants will more strongly agree preference and usability subjective measures for methods with more autonomous assistance}{hypoth:feeding_6}

\hypothcounter{Participants will rank methods with more autonomous assistance above methods with less autonomous assistance}{hypoth:feeding_7}

Our hypotheses depend on an ordering of ``more'' or ``less'' autonomous assistance. The four control methods in this study naturally fall into the following ordering (from least to most assistance): direct teleoperation, blending, policy, and full autonomy. Between the two shared autonomy methods, policy provides more assistance because it creates assistive robot behavior over the entire duration of the trajectory, whereas blend must wait until the intent prediction confidence exceeds some threshold before it produces an assistive robot motion. 

\subsubsection{Experimental Design}
\label{sec:experiment_hri_2016_design}

\begin{figure*}[t]
\centering
  \begin{subfigure}{0.240\textwidth}
    \includegraphics[width=1.0\textwidth, trim=200 0 700 0, clip=true]{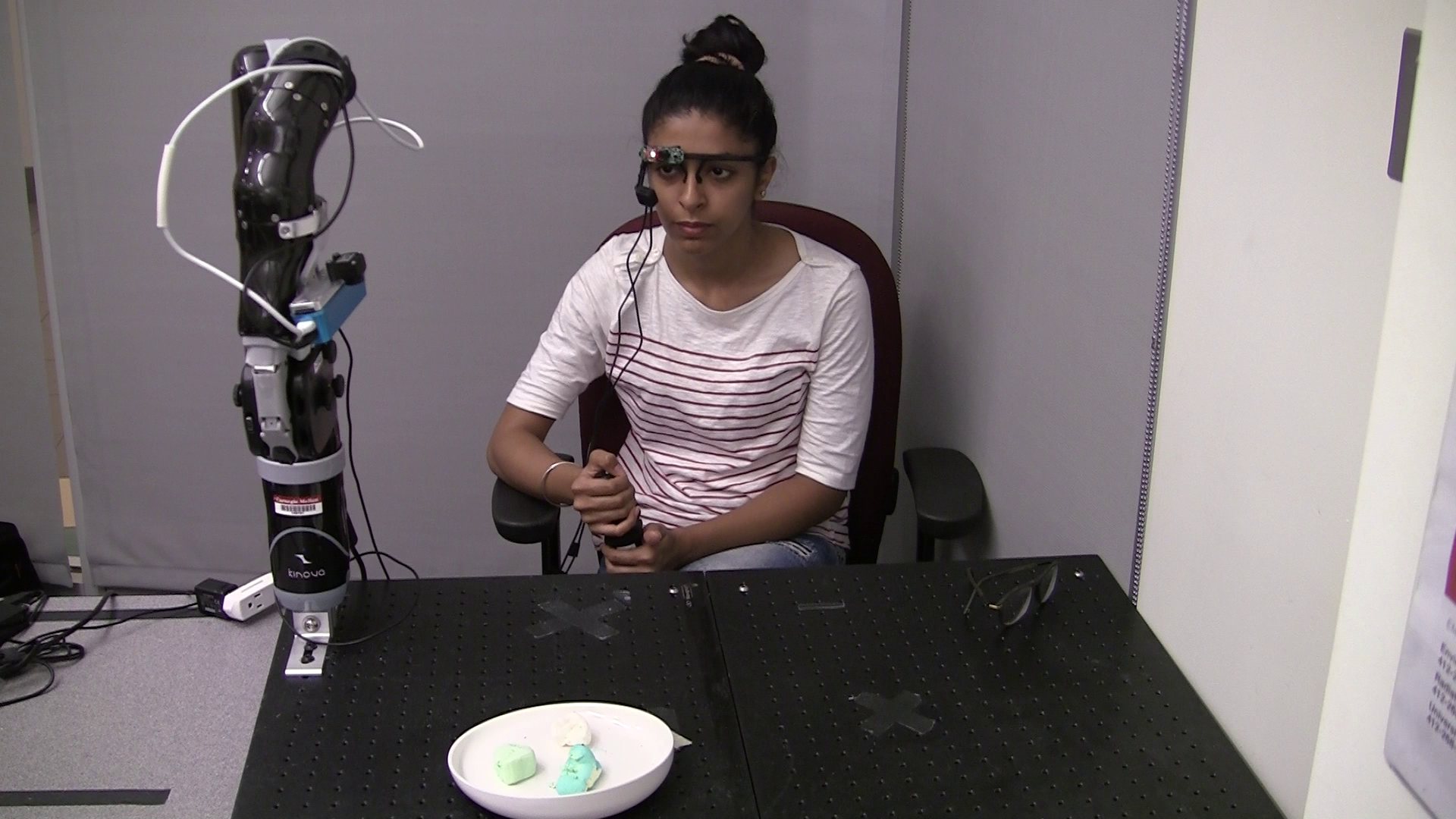}
    \caption{Detect}
    \label{subfig:study_1}
  \end{subfigure}
  \hfill
  \begin{subfigure}{0.240\textwidth}
    \includegraphics[width=1.0\textwidth, trim=200 0 700 0, clip=true]{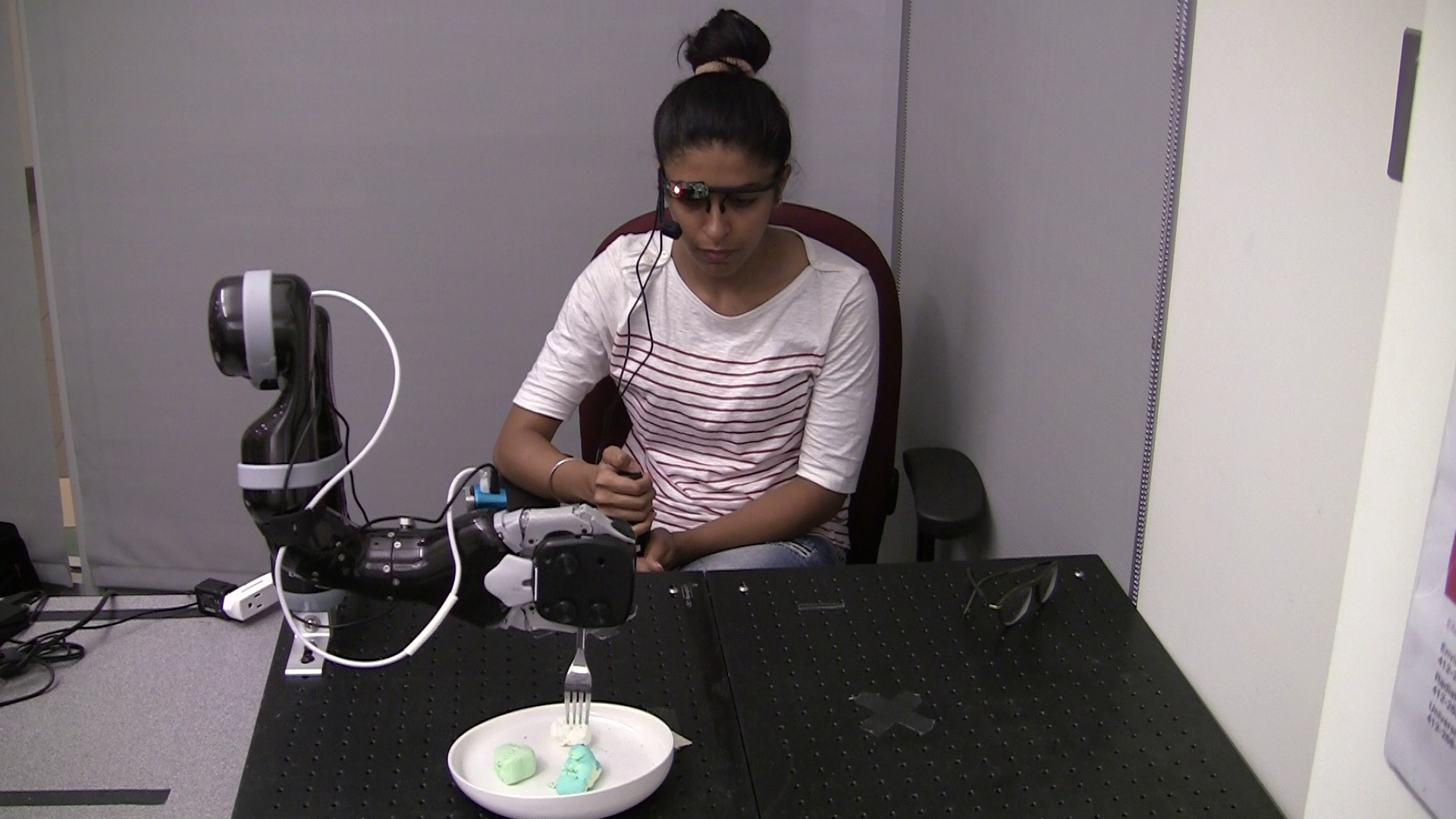}
    \caption{Align}
    \label{subfig:study_2}
  \end{subfigure}
  \hfill
  \begin{subfigure}{0.240\textwidth}
    \includegraphics[width=1.0\textwidth, trim=200 0 700 0, clip=true]{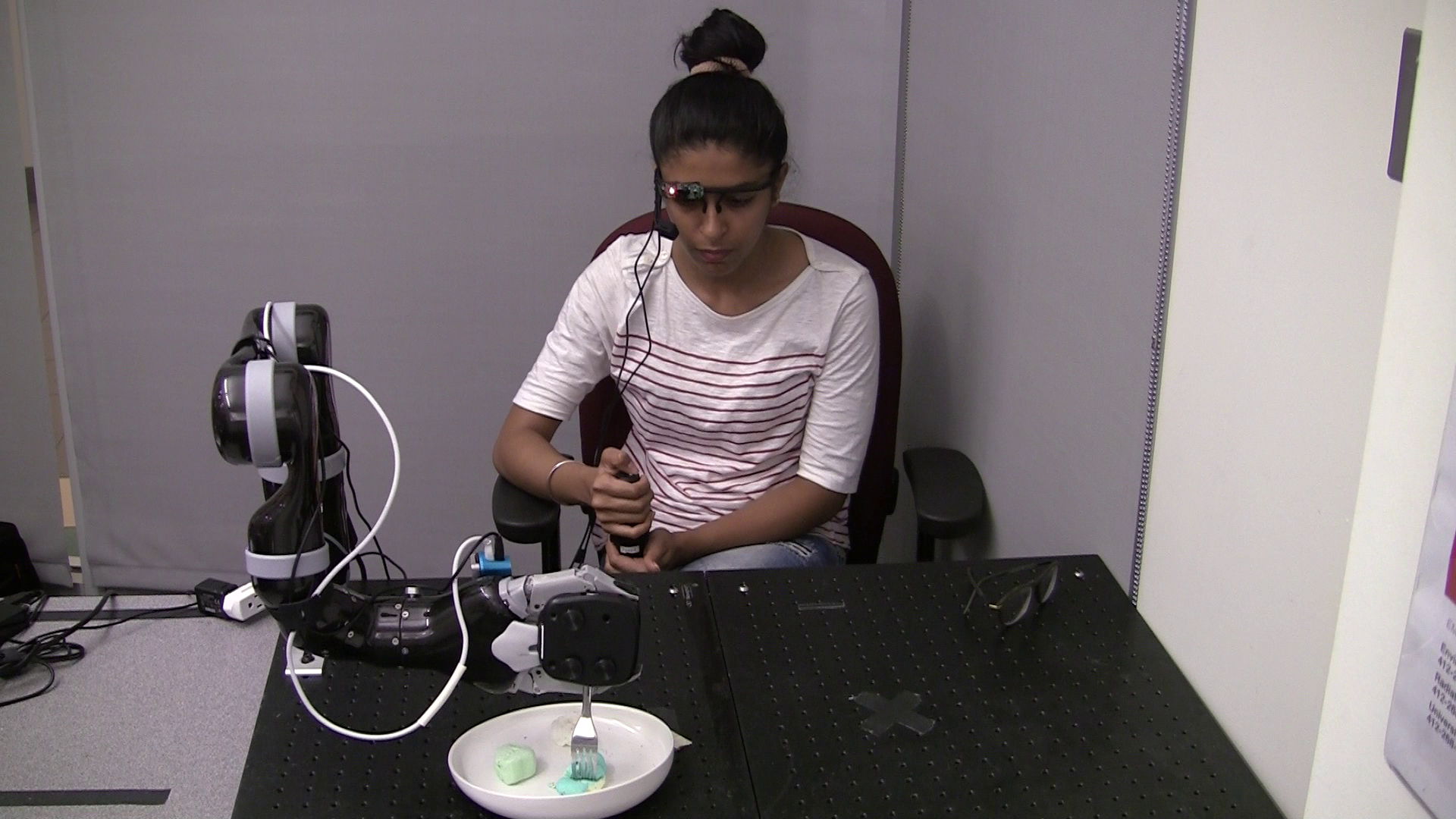}
    \caption{Acquire}
    \label{subfig:study_3}
  \end{subfigure}
  \hfill
  \begin{subfigure}{0.240\textwidth}
    \includegraphics[width=1.0\textwidth, trim=200 0 700 0, clip=true]{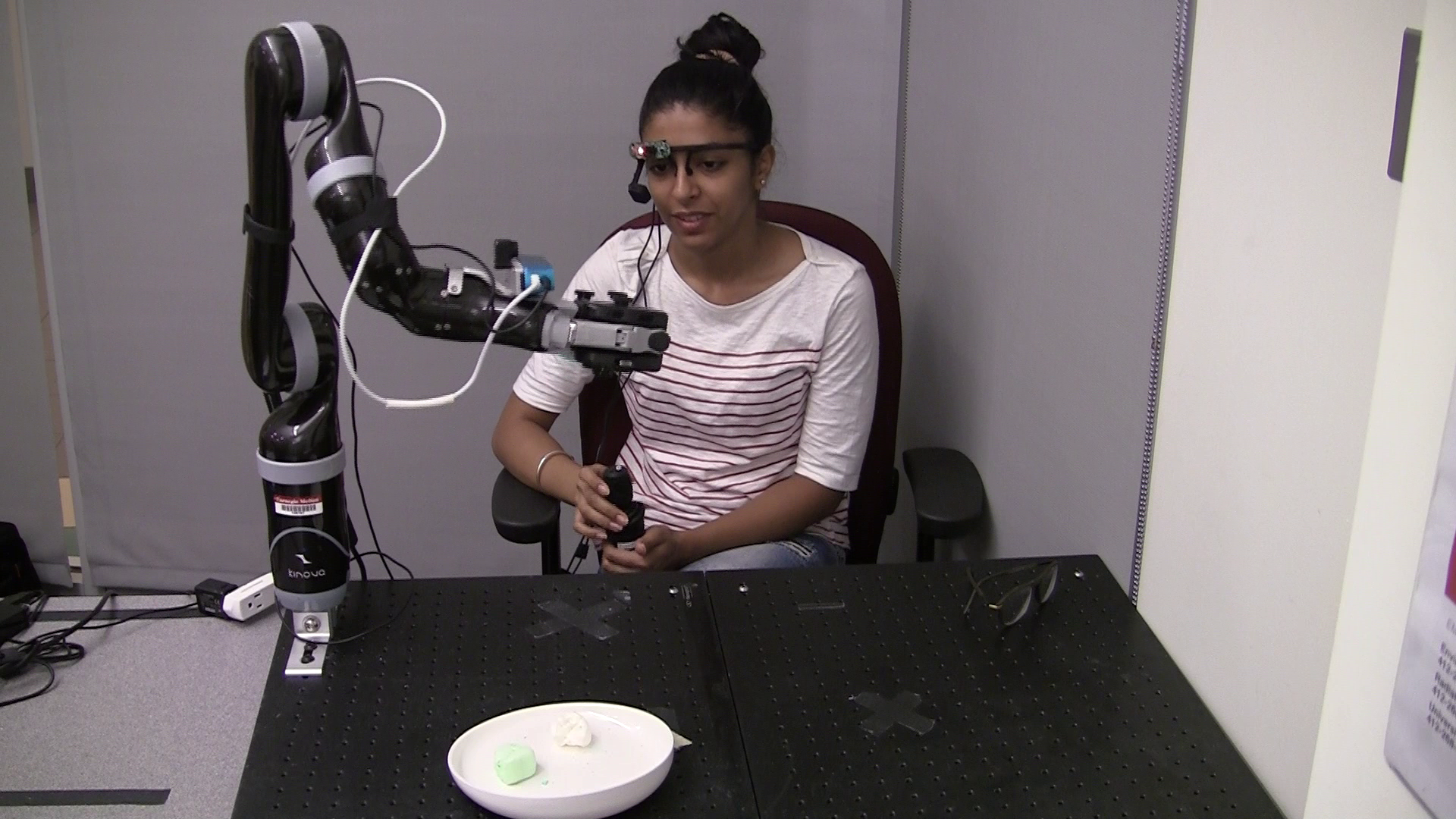}
    \caption{Serve}
    \label{subfig:study_4}
  \end{subfigure}
  \caption{Our bite grasping study. A plate with three bites of food was placed in front of users. (\subref{subfig:study_1}) The robot start by detecting the pose of all bites of food. (\subref{subfig:study_2}) The user then uses one of the four methods to align the fork with their desired bite. When the user indicates they are aligned, the robot automatically (\subref{subfig:study_3}) acquires and (\subref{subfig:study_4}) serves the bite.}
  \label{fig:study_ex}
\end{figure*}

To evaluate each robot control algorithm on a realistic assistive task, participants tried to spear bites of food from a plate onto a fork held in the robot's end effector (\cref{fig:study_ex}). For each trial, participants controlled the robot through a joystick and attempted to retrieve one of three bites of food on a plate.

Each trial followed a fixed bite retrieval sequence. First, the robot would move to a pose where its wrist-mounted camera could detect bites of food on the plate. This step ensured that the system was robust to bite locations and could operate no matter where on the plate the bites were located. While the camera captured and processed the scene to identify bite locations, we asked users to verbally specify which bite they wanted to retrieve\footnote{Users verbally specified which bite they wanted for all methods except autonomous, in which the algorithm selects the bite}, which allowed us to identify whether people were able to successfully retrieve their target bite. 

Next, participants used the joystick to position the robot's end effector so that the fork was directly above their target bite. Six DOF control was available in three modes of 2 DOF each (\cref{fig:control_modes}), and participants could switch between modes by pressing a button on the joystick. 

Once they had the fork positioned above their target bite, the participant prompted the robot to retrieve the bite by pressing and holding the mode switch button. The robot would then automatically move straight down to the height of the table, spearing the bite on the fork. Finally, the robot automatically served the bite.

\subsubsection{Procedure}
\label{sec:experiment_hri_2016_procedure}

We conducted a within-subjects study with one independent variable (control method) that had four conditions (full teleoperation, blend, policy, and full autonomy). Because each participant saw all control methods, we counteract the effects of novelty and practice by fully counterbalancing the order of conditions. Each participant completed five trials for each condition for a total of 20 trials. The bite retrieval sequence described in \cref{sec:experiment_hri_2016_design} was the same in each trial across the four control conditions. The only difference between trials was the control method used for the alignment step, where the fork is positioned above the bite. We measure the metrics discussed in \cref{sec:hypoths_feeding} only during this step.

We recruited 23 able-bodied participants from the local community (11 male, 12 female, ages 19 to 59). After obtaining written consent, participants were given a brief overview of the feeding task, and told the robot may provide help or take over completely. Users then received instruction for teleoperating the system with modal control, and were given five minutes to practice using the robot under direct teleoperation. An eye tracking system was then placed on users for future data analysis, but participant gaze had no effect on the assistance provided by the robot.

As described in \cref{sec:experiment_hri_2016_design}, participants used a joystick to spear a piece of food from a plate on a fork held in the robot's end effector. The different control methods were never explained or identified to users, and were simply referred to by their order of presentation (e.g., ``method 1,'' ``method 2,'' etc.). After using each method, users were given a short questionnaire pertaining to that specific method. The questions were:
\begin{enumerate}
  \item ``I felt in \emph{control}''
  \item ``The robot did what I \emph{wanted}''
  \item ``I was able to accomplish the tasks \emph{quickly}''
  \item ``My \emph{goals} were perceived accurately''\label{question:accurately}
  \item ``If I were going to teleoperate a robotic arm, I would \emph{like} to use the system''
\end{enumerate}
These questions are identical to those asked in the previous evaluation (\cref{sec:experiment_rss_2015}), with the addition of question \ref{question:accurately}, which focuses specifically on the user's goals. Participants were also provided space to write additional comments. After completing all 20 trials, participants were asked to \emph{rank} all four methods in order of preference and provide final comments.

\subsubsection{Results}
\label{sec:results_feeding}
One participant was unable to complete the tasks due to lack of comprehension of instructions, and was excluded from the analysis. One participant did not use the blend method because the robot's finger broke during a previous trial. This user's blend condition and final ranking data were excluded from the analysis, but all other data (which were completed before the finger breakage) were used. Two other participants missed one trial each due to technical issues.

Our metrics are detailed in \cref{sec:experiment_hri_2016_metrics}. For each participant, we computed the task success rate for each method. For metrics measured per trial (execution time, number of mode switches, and total joystick input), we averaged the data across all five trials in each condition, enabling us to treat each user as one independent datapoint in our analyses. Differences in our metrics across conditions were analyzed using a repeated measures ANOVA with a significance threshold of $\alpha=0.05$. For data that violated the assumption of sphericity, we used a Greenhouse-Geisser correction. If a significant main effect was found, a post-hoc analysis was used to identify which conditions were statistically different from each other, with Holm-Bonferroni corrections for multiple comparisons. 

\sloppy
\textbf{Success Rate} differed significantly between control methods $(F(2.33, 49.00)= 4.57, p = 0.011)$. Post-hoc analysis revealed that more autonomy resulted in significant differences of task completion between policy and direct $(p = 0.021)$, and a significant difference between policy and blend $(p=0.0498)$. All other comparisons were not significant. Surprisingly, we found that policy actually had a higher average task completion ratio than autonomy, though not significantly so. Thus, we found support \cref{hypoth:feeding_1} (\cref{subfig:trial_success}).

\textbf{Total execution time} differed significantly between methods $(F(1.89, 39.73) = 43.55, p < 0.001)$. Post-hoc analysis revealed that more autonomy resulted in faster task completion: autonomy condition completion times were faster than policy $(p=0.001)$, blend $(p < 0.001)$, and direct $(p < 0.001)$. There were also significant differences between policy and blend $(p < 0.001)$, and policy and direct $(p < 0.001)$. The only pair of methods which did not have a significant difference was blend and direct. Thus, we found support for \cref{hypoth:feeding_2} (\cref{subfig:time}).

\textbf{Number of mode switches} differed significantly between methods $(F(2.30, 48.39) = 65.16, p < 0.001)$. Post-hoc analysis revealed that more autonomy resulted fewer mode switches between autonomy and blend $(p < 0.001)$, autonomy and direct $(p < 0.001)$, policy and blend $(p < 0.001)$, and policy and direct $(p < 0.001)$. Interestingly, there was not a significant difference in the number of mode switches between full autonomy and policy, even though users cannot mode switch when using full autonomy at all. Thus, we found support for \cref{hypoth:feeding_3} (\cref{subfig:num_mode_switches}).

\textbf{Total joystick input} differed significantly between methods $(F(1.67, 35.14) = 65.35, p < 0.001)$. Post-hoc analysis revealed that more autonomy resulted in less total joystick input between all pairs of methods: autonomy and policy $(p < 0.001)$, autonomy and blend $(p < 0.001)$, autonomy and direct $(p < 0.001)$, policy and blend $(p < 0.001)$, policy and direct $(p < 0.001)$, and blend and direct $(p = 0.026)$. Thus, we found support for \cref{hypoth:feeding_4} (\cref{subfig:user_input}).

%
%

\begin{figure}[t]
\centering
\captionsetup[subfigure]{margin={0.7cm, 0pt}, aboveskip=0.5pt,belowskip=+5.pt}
\begin{subfigure}{0.238\textwidth}
  \includegraphics[clip=true]{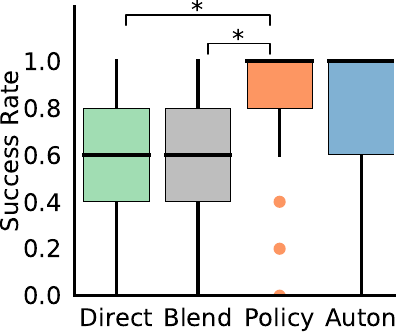}
  \caption{Trial Success}
  \label{subfig:trial_success}
\end{subfigure}
\hfill
\begin{subfigure}{0.242\textwidth}
  \includegraphics[clip=true]{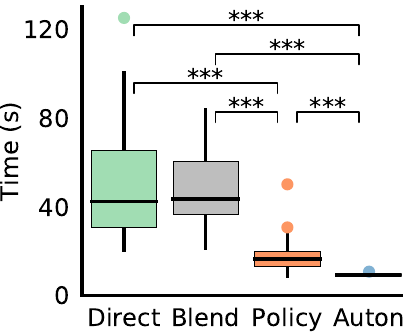}
  \caption{Time}
  \label{subfig:time}
\end{subfigure}

\begin{subfigure}{0.238\textwidth}
  \includegraphics[clip=true]{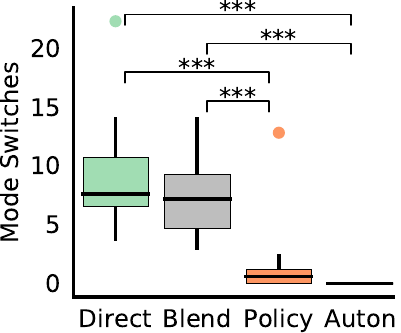}
  \caption{Mode Switches}
  \label{subfig:num_mode_switches}
\end{subfigure}
\hfill
\begin{subfigure}{0.242\textwidth}
  \includegraphics[clip=true]{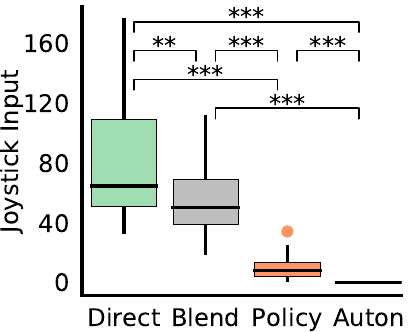}
  \caption{User Input}
  \label{subfig:user_input}
\end{subfigure}

\begin{subfigure}{0.238\textwidth}
  \includegraphics[clip=true]{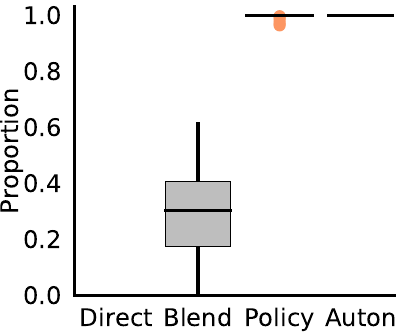}
  \caption{Assistance Ratio}
  \label{subfig:assist_ratio}
\end{subfigure}
\hfill
\begin{subfigure}{0.242\textwidth}
  \includegraphics[clip=true]{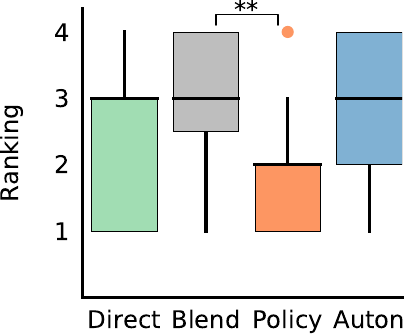}
  \caption{Method Rank}
  \label{subfig:rank}
\end{subfigure}
\caption{Boxplots for each algorithm across all users of the (\subref{subfig:trial_success}) task completion ratio, (\subref{subfig:time}) total execution time, (\subref{subfig:num_mode_switches}) number of mode switches, (\subref{subfig:user_input}) total joystick input, (\subref{subfig:assist_ratio}) the ratio of time that robotic assistance was provided, and (\subref{subfig:rank}) the ranking as provided by each user, where 1 corresponds to the most preferred algorithm. Pairs that were found significant during post-analysis are plotted, where ${*}$ indicates $p<0.05$, ${*}{*}$ that $p<0.01$, and ${*}{*}{*}$ that $p<0.001$.}
 \label{fig:box_results}
\end{figure}

\begin{figure}[t]
  \centering
  \captionsetup[subfigure]{aboveskip=0.5pt,belowskip=+5.pt}
  \begin{subfigure}{0.47\textwidth}
    \includegraphics{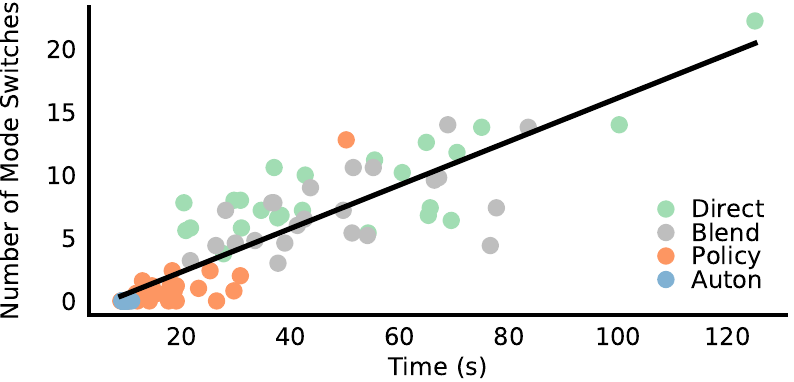}
    \caption{Time vs. Mode Switches}
    \label{subfig:time_vs_mode_switch}
  \end{subfigure}
  \begin{subfigure}{0.47\textwidth}
    \includegraphics{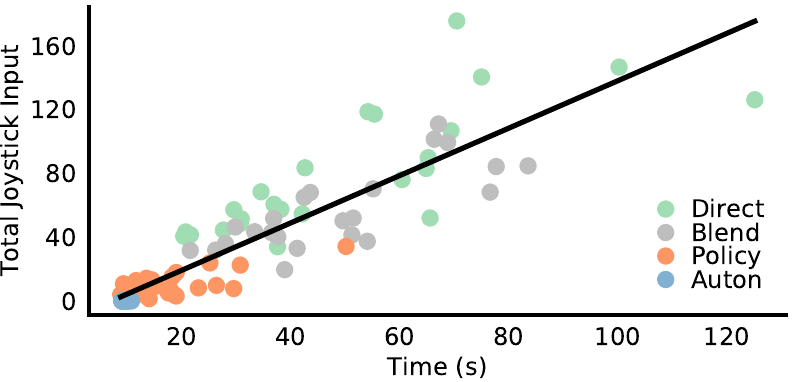}
    \caption{Time vs. Total Joystick Input}
    \label{subfig:time_vs_joystick}
  \end{subfigure}
  \caption{Time vs.\ user input in both the number of mode switches (\subref{subfig:time_vs_mode_switch}) and joystick input (\subref{subfig:time_vs_joystick}). Each point corresponds to the average for one user for each method. We see a general trend that trials with more time corresponded to more user input. We also fit a line so all points for all methods. Note that the direct teleoperation methods are generally above the line, indicating that shared and full autonomy usually results in less user input even for similar task completion time.}
  \label{fig:time_vs_user_inputs}
\end{figure}

\begin{figure}[t]
  \includegraphics{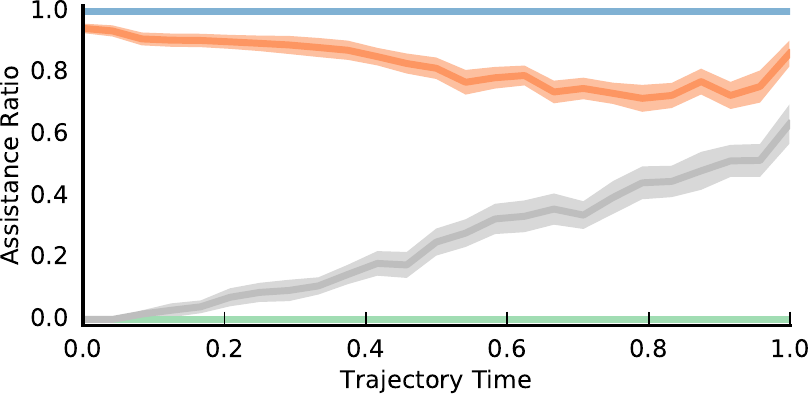}
  \caption{Ratio of the magnitude of the assistance to user input as a function of time. Line shows mean of the assistance ratio as a function of the proportion of the trajectory. Shaded array plots the standard error over users. We see that blend initially provides no assistance, as the predictor is not confident in the user goal. In contrast, policy provides assistance throughout the trajectory. We also see that policy decreases in assistance ratio over time, as many users provided little input until the system moved and oriented the fork near all objects, at which time they provided input to express their preference and align the fork.}
  \label{fig:assist_ratio_over_time}
\end{figure}

User reported subjective measures for the survey questions are assessed using a Friedman's test and a significance threshold of $p=0.05$. If significance was found, a post-hoc analysis was performed, comparing all pairs with Holm-Bonferroni corrections.

User agreement on \textbf{control} differed significantly between methods, $\xi^2(3) = 15.44, p<0.001$, with more autonomy leading to less feeling of control. Post-hoc analysis revealed that all pairs were significant, where autonomy resulting in less feeling of control compared to policy $(p<0.001)$, blend $(p=0.001)$, and direct $(p<0.001)$. Policy resulted in less feeling of control compared to blend $(p<0.001)$ and direct $(p=0.008)$. Blend resulted in less feeling of control compared to direct $(p=0.002)$. Thus, we found suppoert for \cref{hypoth:feeding_5}.

User agreement on preference and usability subjective measures sometimeses differed significantly between methods. User agreement on \textbf{liking} differed significantly between methods, $\xi^2(3) = 8.74, p=0.033$. Post-hoc analysis revealed that between the two shared autonomy methods (policy and blend), users liked the more autonomous method more $(p=0.012)$. User ability for achieving goals \textbf{quickly} also differed significantly between methods, $\xi^2{3} = 11.90, p=0.008$. Post-hoc analysis revelead that users felt they could achieve their goals more quickly with policy than with blend $(p=0.010)$ and direct $(p=0.043)$. We found no significant differences for our other measures. Thus, we find partial support for \cref{hypoth:feeding_6} (\cref{fig:survey_questions}).

\textbf{Ranking} differed significantly between methods, $\xi^2(3) = 10.31, p=0.016$. Again, post-hoc analysis revealed that between the two shared autonomy methods (policy and blend), users ranked the more autonomous one higher $(p=0.006)$. Thus, we find support for \cref{hypoth:feeding_7}. As for the like rating, we also found that on average, users ranked direct teleopration higher than both blend and full autonomy, though not significantly so (\cref{subfig:rank}).


\begin{figure*}[t]
  \includegraphics[clip=true]{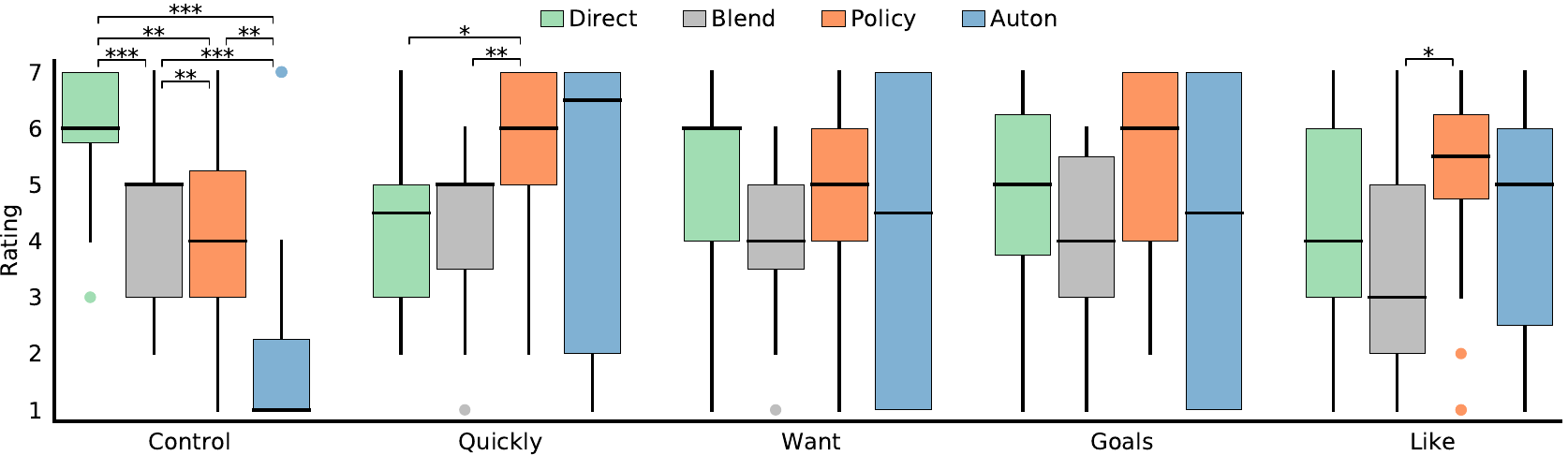}
  \caption{Boxplots for user responses to all survey question. See \cref{sec:experiment_hri_2016_procedure} for specific questions. Pairs that were found significant during post-analysis are plotted, where ${*}$ indicates $p<0.05$, ${*}{*}$ that $p<0.01$, and ${*}{*}{*}$ that $p<0.001$. We note that policy was perceived as quick, even though autonomy actually had lower task completion (\cref{subfig:time}). Additionally, autonomy had a very high variance in user responses for many questions, with users very mixed on if it did what they wanted, and achieved their goal. On average, we see that policy did better then other methods for most user responses.}
  \label{fig:survey_questions}
\end{figure*}

\subsubsection{Discussion}
The robot in this study was controlled through a 2 DOF joystick and a single button, which is comparable to the assistive robot arms in use today.

As expected, we saw a general trend in which more autonomy resulted in better performance across all objective measures (task completion ratio, execution time, number of mode switches, and total joystick input), supporting \cref{hypoth:feeding_1}--\cref{hypoth:feeding_4}. We also saw evidence that autonomy decreased feelings of control, supporting \cref{hypoth:feeding_5}. However, it improved people's subjective evaluations of usability and preference, particularly between the shared autonomy methods (policy and blend), supporting \cref{hypoth:feeding_6} and \cref{hypoth:feeding_7}. Most objective measures (particularly total execution time, number of mode switches, and total joystick input) showed significant differences between all or nearly all pairs of methods, while the subjective results were less certain, with significant differences between fewer pairs of methods.

We can draw several insights from these findings. First, autonomy improves peoples' performance on a realistic assistive task by requiring less physical effort to control the robot. People use fewer mode switches (which require button presses) and move the joystick less in the more autonomous conditions, but still perform the task more quickly and effectively. For example, in the policy method, $8$ of our $22$ users did not use any mode switches for any trial, but this method yielded the highest completion ratio and low execution times. Clearly, some robot autonomy can benefit people's experience by reducing the amount of work they have to do.

Interestingly, full autonomy is not always as effective as allowing the user to retain some control. For example, the policy method had a slightly (though not significantly) higher average completion ratio than the full autonomy method. This appears to be the result of users fine-tuning the robot's end effector position to compensate for small visual or motor inaccuracies in the automatic bite localization process. Because the task of spearing relatively small bites of food requires precise end effector localization, users' ability to fine-tune the final fork alignment seems to benefit the overall success rate. Though some users were able to achieve it, our policy method isn't designed to allow this kind of fine-tuning, and will continually move the robot's end effector back to the erroneous location against the user's control. Detecting when this may be occurring and decreasing assistance would likely enhance people's ability to fine-tune alignment, and improve their task completion rate even further.

Given the success of blending in previous studies~\citep{li_2011, carlson_2012, dragan_2013_assistive, muelling_2015, gopinath_2016}, we were surprised by the poor performance of blend in our study. We found no significant difference for blending over direct teleopration for success rate, task completion time, or number of mode switches. We also saw that it performed the worst among all methods for both user liking and ranking. 
One possible explanation is that blend spent relatively little time assisting users (\cref{subfig:assist_ratio}). For this task, the goal predictor was unable to confidently predict the user's goal for $69\%$ of execution time, limiting the amount of assistance (\cref{fig:assist_ratio_over_time}). Furthermore, the difficult portion of the task---rotating the fork tip to face downward---occurred at the beginning of execution. Thus, as one user put it ``While the robot would eventually line up the arm over the plate, most of the hard work was done by me.'' In contrast, user comments for shared autonomy indicated that ``having help earlier with fork orientation was best.'' 
This suggests that the \emph{magnitude} of assistance was less important then assisting at a time that would have been helpful. And in fact, assisting only during the portion where the user could do well themselves resulted in additional frustration.

Although worse by all objective metrics, participants tended to prefer direct teleoperation over autonomy. This is not entirely surprising, given prior work where users expressed preference for more control~\cite{kim_2012}. However, for difficult tasks like this one, users in prior works tend to favor more assistance~\citep{you_2011, dragan_2013_assistive}. Many users commented that they disliked autonomy due to the lack of item selection, for example, ``While [autonomy] was fastest and easiest, it did not account for the marshmallow I wanted.'' Another user mentioned that autonomy ``made me feel inadequate.''

We also found that users responded to failures by blaming the system, even when using direct teleoperation. Of the eight users who failed to successfully spear a bite during an autonomous trial, five users commented on the failure of the algorithm. In contrast, of the 19 users who had one or more failure during teleoperation, only two commented on their own performance. Instead, users made comments about the system itself, such as how the system ``seemed off for some reason'' or ``did not do what I intended.'' One user blamed their viewpoint for causing difficulty for the alignment, and another the joystick. This suggests that people are more likely to penalize autonomy for its shortcomings than their own control. Interestingly, this was not the case for the shared autonomy methods. We find that when users had some control over the robot's movement, they did not blame the algorithm's failures (for example, mistaken alignments) on the system.

\begin{table*}
  \newcommand{\algoauton}{Auton}
  \newcommand{\algopolicy}{Policy}
  \newcommand{\algoblend}{Blend}
  \newcommand{\algodirect}{Direct}
  \newcommand{\sigresult}[1]{\bm{#1}}
  \newcommand{\spssnoval}{\sigresult{<\!0.001}}
  \newcommand{\nonsigresult}{\text{NS}}
  \centering
  \begin{tabular}{|c|c|c|c|c|c|c|}
    \hline
    Metric & \algoauton-\algopolicy & \algoauton-\algoblend & \algoauton-\algodirect & \algopolicy-\algoblend & \algopolicy-\algodirect & \algoblend-\algodirect \\
    \hline
    Success Rate & $\nonsigresult$ & $\nonsigresult$ & $\nonsigresult$ & $\sigresult{0.050}$ & $\sigresult{0.021}$ & $\nonsigresult$ \\
    Completion Time & $\spssnoval$ & $\spssnoval$ & $\spssnoval$ & $\spssnoval$ & $\spssnoval$ & $\nonsigresult$ \\
    Mode Switches & $\nonsigresult$ & $\spssnoval$ & $\spssnoval$ & $\spssnoval$ & $\spssnoval$ & $\nonsigresult$ \\
    Control Input & $\spssnoval$ & $\spssnoval$ & $\spssnoval$ & $\spssnoval$ & $\spssnoval$ & $\sigresult{0.004}$ \\
    Ranking & $\nonsigresult$ & $\nonsigresult$ & $\nonsigresult$ & $\sigresult{0.006}$ & $\nonsigresult$ & $\nonsigresult$ \\
    Like Rating& $\nonsigresult$ & $\nonsigresult$ & $\nonsigresult$ & $\sigresult{0.012}$ & $\nonsigresult$ & $\nonsigresult$ \\
    Control Rating& $\spssnoval$ & $\sigresult{.001}$ & $\spssnoval$ & $\spssnoval$ & $\sigresult{0.008}$ & $\sigresult{.002}$ \\
    Quickly Rating& $\nonsigresult$ & $\nonsigresult$ & $\nonsigresult$ & $\sigresult{0.010}$ & $\sigresult{0.043}$ & $\nonsigresult$ \\
    \hline
    \end{tabular} 
    \caption{Post-Hoc p-value for every pair of algorithms for each hypothesis. For Success rate, completion time, mode switches, and total joystick input, results are from a repeated measures ANOVA. For like rating and ranking, results are from a Wilcoxon signed-rank test. All values reported with Holm-Bonferroni corrections.}
    \label{tab:eating_p_val_results}
\end{table*}

\section{Human-Robot Teaming}
\label{sec:human_robot_teaming}


In human-robot teaming, the user and robot want to achieve a set of related goals. Formally, we assume a set of user goals $\usergoal \in \userGoal$ and robot goals $\robotgoal \in \robotGoal$, where both want to achieve all goals. However, there may be constraints on how these goals can be achieved (e.g. user and robot cannot simultaneously use the same object~\citep{hoffman_2007}). We apply a conservative model for these constraints through a \emph{goal restriction set} $\goalrestrictionset = \left\{ (\usergoal, \robotgoal) : \text{Cannot achieve $\usergoal$ and $\robotgoal$ simultaneously}\right\}$. In order to efficiently collaborate with the user, our objective is to simultaneously predict the human's intended goal, and achieve a robot goal not in the restricted set. We remove the achieved goals from their corresponding goal sets, and repeat this process until all robot goals are achieved.

The state $\stateenv$ corresponds to the state of both the user and robot, where $\actionuser$ affects the user portion of state, and $\actionrobot$ affects the robot portion. The transition function $\transitionallargs$ deterministically transitions the state by applying $\actionuser$ and $\actionrobot$ sequentially.

For prediction, we used the same cost function for $\costtarguser$ as in our shared teleoperation experiments (\cref{sec:shared_teleop}). Let $d$ be the distance between the robot state $\stateenv' = \transitionuser(\stateenv, \actionuser)$\footnote{We sometimes instead observe $\stateenv'$ directly (e.g. sensing the pose of the user hand)} and target $\target$:
\begin{align*}
  \costtarguser(\stateenv, \actionuser) &= \left\{ \begin{array}{cc} \alpha & d > \delta \\ \frac{\alpha}{\delta} d & d\leq \delta \end{array} \right.
\end{align*}
Which behaves identically to our shared control teleoperation setting: when the distance is far away from any target $(d > \delta)$, probability shifts towards goals relative to how much progress the user makes towards them. When the user stays close to a particular target $(d \leq \delta)$, probability mass shifts to that goal, as the cost for that goal is less than all others.

Unlike our shared control teleoperation setting, our robot cost function does not aim to achieve the same goal as the user, but rather any goal not in the restricted set. As in our shared autonomy framework, let $\goal$ be the user's goal. The cost function for a particular user goal is:
\begin{align*}
  \costgoalrobot(\stateenv, \actionuser, \actionrobot) &= \min_{\robotgoal \text{\;s.t.\;} (\goal, \robotgoal) \not\in \goalrestrictionset} \costuser_{\robotgoal}(\stateenv, \actionrobot)
\end{align*}
Where $\costgoaluser$ uses the cost for each target $\costtarguser$ to compute the cost function as described in \cref{sec:framework_multitarget}. Additionally, note that the $\min$ over cost functions looks identical to the $\min$ over targets to compute the cost for a goal. Thus, for deterministic transition functions, we can use the same proof for computing the value function of a goal given the value function for all targets (\cref{sec:framework_multigarget_assistance}) to compute the value function for a robot goal given the value function for all user goals:
\begin{align*}
  \vgoalrobot(\stateenv) &= \min_{\robotgoal \text{\;s.t.\;} (\goal, \robotgoal) \not\in \goalrestrictionset} \vuser_{\robotgoal}(\stateenv)
\end{align*}

This simple cost function provides us a baseline for performance. We might expect better collaboration performance by incorporating costs for collision avoidance with the user~\citep{mainprice_2013, lasota_2015}, social acceptability of actions~\citep{sisbot_2007}, and user visibility and reachability~\citep{sisbot_2010, pandey_2010, mainprice_2011}. We use this cost function to test the viability of our framework as it enables closed-form computation of the value function.

This cost and value function causes the robot to go to any goal currently in it's goal set $\robotgoal \in \robotGoal$ which is not in the restriction set of the user goal $\goal$. Under this model, the robot makes progress towards goals that are unlikely to be in the restricted set and have low cost-to-go. As the form of the cost function is identical to that which we used in shared control teleoperation, the robot behaves similarly: making constant progress when far away $(d > \delta)$, and slowing down for alignment when near $(d \leq \delta)$. The robot terminates and completes the task once some condition is met (e.g. $d \leq \epsilon$).

\subsubsection*{Hindsight Optimization for Human-Robot Teaming}

Similar to shared control teleoperation, we believe hindsight optimization is a suitable POMDP approximation for human-robot teaming. The efficient computation enables us to respond quickly to changing user goals, even with continuous state and action spaces. For our formulation of human-robot teaming, explicit information gathering is not possible: As we assume the user and robot affect different parts of state space, robot actions are unable to explicitly gather information about the user's goal. Instead, we gain information freely from user actions.


\subsection{Human-Robot Teaming Experiment}
\label{sec:experiment_iros_2016}

\graphicspath{{./}{./images_iros_2016/}}

We apply our shared autonomy framework to a human-robot teaming task of gift-wrapping, where the user and robot must both perform a task on each box to be gift wrapped. Our goal restriction set enforces that they cannot perform a task on the same box at the same time. 

In a user study, we compare three methods: our shared autonomy framework, referred to as \emph{policy}, a standard predict-then-act system, referred to as \emph{plan}, and a non-adaptive system where the robot executes a fixed sequence of motions, referred to as \emph{fixed}.

%

\subsubsection{Metrics}

\emph{Task fluency} involves seamless coordination of action. One measure for task fluency is the minimum distance between the human and robot end effectors during a trial. This was measured automatically by a Kinect mounted on the robot's head, operating at 30Hz. Our second fluency measure is the proportion of trial time spent in collision. Collisions occur when the distance between the robot's end effector and the human's hand goes below a certain threshold. We determined that 8cm was a reasonable collision threshold based on observations before beginning the study. 

\emph{Task efficiency} relates to the speed with which the task is completed. Objective measures for task efficiency were total task duration for robot and for human, the amount of human idle time during the trial, and the proportion of trial time spent idling. Idling is defined as time a participant spends with their hands still (i.e., not completing the task). For example, idling occurs when the human has to wait for the robot to stamp a box before they can tie the ribbon on it. We only considered idling time while the robot was executing its tasks, so idle behaviors that occurred after the robot was finished stamping the boxes---which could not have been caused by the robot's behavior---were not taken into account.

We also measured subjective \emph{human satisfaction} with each method through a seven-point Likert scale survey evaluating perceived safety (four questions) and sense of collaboration (four questions). The questions were:
\begin{enumerate}
  \item ``HERB was a good partner''
  \item ``I think HERB and I worked well as a team''
  \item ``I'm dissatisfied with how HERB and I worked together''
  \item ``I trust HERB''
  \item ``I felt that HERB kept a safe distance from me''
  \item ``HERB got in my way''
  \item ``HERB moved too fast''
  \item ``I felt uncomfortable working so close to HERB''
\end{enumerate}

\subsubsection{Hypotheses}

We hypothesize that:

\newhypothset

\hypothcounter{Task fluency will be improved with our policy method compared with the plan and fixed methods}{hypoth:teaming_1}

\hypothcounter{Task efficiency will be improved with our policy method compared with the plan and fixed methods}{hypoth:teaming_2}

\hypothcounter{People will subjectively prefer the policy method to the plan or fixed methods}{hypoth:teaming_3}

\subsubsection{Experimental Design}
We developed a gift-wrapping task (\cref{fig:collab_exper_setup}). A row of four boxes was arranged on a table between the human and the robot; each box had a ribbon underneath it. The robot's task was to stamp the top of each box with a marker it held in its hand. The human's task was to tie a bow from the ribbon around each box. By nature of the task, the goals had to be selected serially, though ordering was unspecified. Though participants were not explicitly instructed to avoid the robot, tying the bow while the robot was stamping the box was challenging because the robot's hand interfered, which provided a natural disincentive toward selecting the same goal simultaneously.
\begin{figure}
	\centering
	\includegraphics[width=0.9\columnwidth]{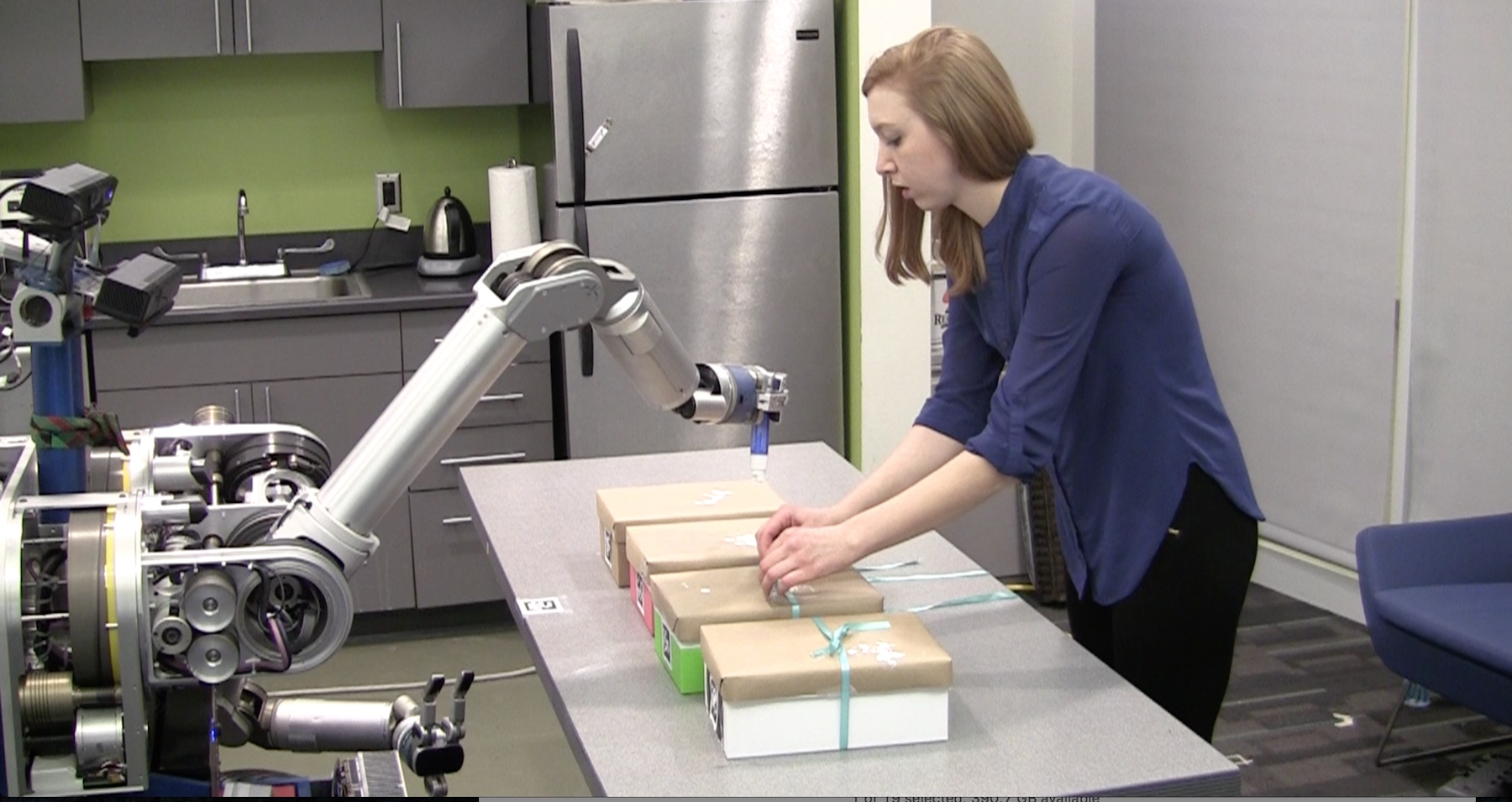}
  \caption{Participants performed a collaborative gift-wrapping task with HERB to evaluate our POMDP based reactive system against a state of the art predict-then-act method, and a non-adaptive fixed sequence of robot goals.}
	\label{fig:collab_exper_setup}
\end{figure}

\subsubsection{Implementation}
We implemented the three control methods on HERB~\citet{srinivasa_2012}, a bi-manual mobile manipulator with two Barrett WAM arms. A Kinect was used for skeleton tracking and object detection. Motion planning was performed using CHOMP, except for our policy method in which motion planning works according to \cref{sec:framework}.

The stamping marker was pre-loaded in HERB's hand. A stamping action began at a home position, the robot extended its arm toward a box, stamped the box with the marker, and retracted its arm back to the home position.

To implement the fixed method, the system simply calculated a random ordering of the four boxes, then performed a stamping action for each box. To implement the predict-then-act method, the system ran the human goal prediction algorithm from \cref{sec:framework_prediction} until a certain confidence was reached (50\%), then selected a goal that was not within the restricted set $\goalrestrictionset$ and performed a stamping action on that goal. There was no additional human goal monitoring once the goal action was selected. In contrast, our policy implementation performed as described in \cref{sec:human_robot_teaming}, accounting continually for adapting human goals and seamlessly re-planning when the human's goal changed.

\subsubsection{Procedure}
We conducted a within-subjects study with one independent variable (control method) that had 3 conditions (policy, plan, and fixed). Each performed the gift-wrapping task three times, once with each robot control method. To counteract the effects of novelty and practice, we counterbalanced on the order of conditions.

We recruited 28 participants (14 female, 14 male; mean age 24, SD 6) from the local community. Each participant was compensated \$5 for their time. After providing consent, participants were introduced to the task by a researcher. They then performed the three gift-wrapping trials sequentially. Immediately after each trial, before continuing to the next one, participants completed an eight question Likert-scale survey to evaluate their collaboration with HERB on that trial. At the end of the study, participants provided verbal feedback about the three methods. All trials and feedback were video recorded.

\subsubsection{Results}
\label{sec:results_iros_2016}
Two participants were excluded from all analyses for noncompliance during the study (not following directions). Additionally, for the fluency objective measures, five other participants were excluded due to Kinect tracking errors that affected the automatic calculation of minimum distance and time under collision threshold. Other analyses were based on video data and were not affected by Kinect tracking errors.

\begin{figure}[t]
	\centering
	\begin{subfigure}[t]{0.49\columnwidth}
		\includegraphics{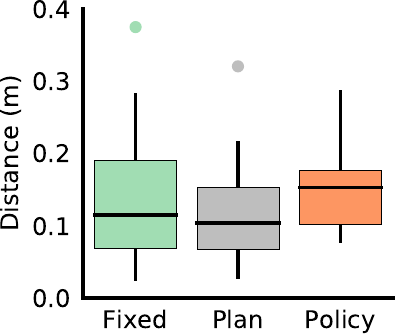}
		\caption{Minimum Distance}
		\label{fig:result_mindist}
	\end{subfigure}
	\hfill
	\begin{subfigure}[t]{0.49\columnwidth}
		\includegraphics{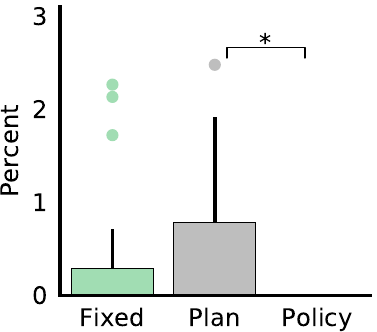}
		\caption{Percent in Collision}
		\label{fig:result_percentcollision}
	\end{subfigure}
	\caption{Distance metrics: no difference between methods for minimum distance during interaction, but the policy method yields significantly $(p<0.05)$ less time in collision between human and robot.}
	\label{fig:distance}
\end{figure}

\begin{figure*}[t]
	\begin{subfigure}[b]{0.32\textwidth}
		\includegraphics{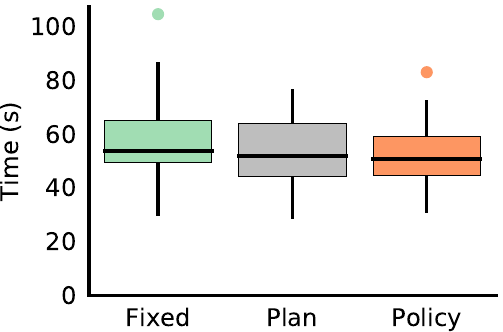}
		\caption{Human Trial Duration}
		\label{fig:result_humantime}
	\end{subfigure}
	\hfill
	\begin{subfigure}[b]{0.32\textwidth}
		\includegraphics{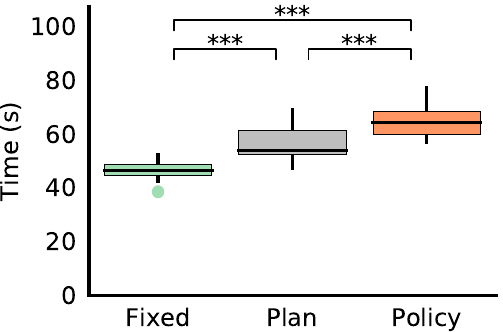}
		\caption{Robot Trial Duration}
		\label{fig:result_robottime}
	\end{subfigure}
	\hfill
	\begin{subfigure}[b]{0.32\textwidth}
		\includegraphics{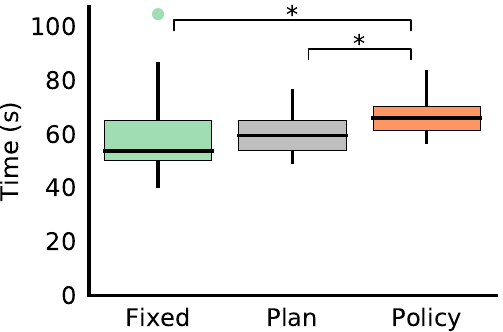}
		\caption{Total Trial Duration}
		\label{fig:result_totaltime}
	\end{subfigure}
	\caption{Duration metrics, with pairs that differed significantly during post-analysis are plotted, where ${*}$ indicates $p<0.05$ and ${*}{*}{*}$ that $p<0.001$. Human trial time was approximately the same across all methods, but robot time increased with the computational requirements of the method. Total time thus also increased with algorithmic complexity.}
	\label{fig:duration}
\end{figure*}

\begin{figure}[t]
	\begin{subfigure}[b]{0.49\columnwidth}
		\includegraphics{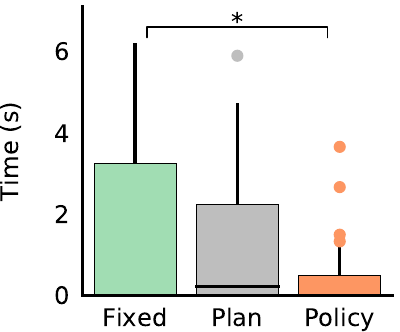}
		\caption{Human Idle Time}
		\label{fig:result_idletime}
	\end{subfigure}
	\hfill
	\begin{subfigure}[b]{0.49\columnwidth}
		\includegraphics{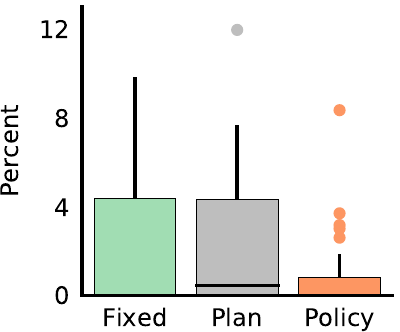}
		\caption{Human Idle Percentage}
		\label{fig:result_idlepercentage}
	\end{subfigure}
	\caption{Idle time metrics: policy yielded significantly $(p<0.05)$ less absolute idle time than the fixed method.}
	\label{fig:idle}
\end{figure}

We assess our hypotheses using a significance level of $\alpha=0.05$. For data that violated the assumption of sphericity, we used a Greenhouse-Geisser correction. If a significant main effect was found, a post-hoc analysis was used to identify which conditions were statistically different from each other, with Holm-Bonferroni corrections for multiple comparisons.

To evaluate \cref{hypoth:teaming_1} (fluency), we conducted a repeated measures ANOVA testing the effects of method type (policy, plan, and fixed) on our two measures of human-robot distance: the minimum distance between participant and robot end effectors during each trial, and the proportion of trial time spent with end effector distance below the 8cm collision threshold (\cref{fig:distance}). The minimum distance metric was not significant ($F(2,40) = 1.405, p = 0.257$). However, proportion of trial time spent in collision was significantly affected by method type ($F(2,40) = 3.639, p = 0.035$). Interestingly, the policy method never entered under the collision threshold. Post-hoc pairwise comparisons with a Holm-Bonferroni correction reveal that the policy method yielded significantly ($p = 0.027$) less time in collision than the plan method (policy $M=0.0\%, SD=0$; plan $M=0.44\%, SD = 0.7$). 

Therefore, \cref{hypoth:teaming_1} is partially supported: the policy method actually yielded no collisions during the trials, whereas the plan method yielded collisions during 0.4\% of the trial time on average. This confirms the intuition behind the differences in the two methods: the policy continually monitors human goals, and thus never collides with the human, whereas the plan method commits to an action once a confidence level has been reached, and is not adaptable to changing human goals.

To evaluate \cref{hypoth:teaming_2} (efficiency), we conducted a similar repeated measures ANOVA for the effect of method type on task durations for robot and human (\cref{fig:duration}), as well as human time spent idling (\cref{fig:idle}).
Human task duration was highly variable and no significant effect for method was found ($F(2,50)=2.259, p = 0.115$). On the other hand, robot task duration was significantly affected by method condition ($F(2,50)=79.653, p<0.001$). Post-hoc pairwise comparisons with a Bonferroni correction reveal that differences between all conditions are significant at the $p < 0.001$ level. Unsurprisingly, robot task completion time was shortest in the fixed condition, in which the robot simply executed its actions without monitoring human goals ($M=46.4s, SD=3.5s$). It was significantly longer with the plan method, which had to wait until prediction reached a confidence threshold to begin its action ($M=56.7s, SD=6.0$). Robot task time was still longer for the policy method, which continually monitored human goals and smoothly replanned motions when required, slowing down the overall trajectory execution ($M=64.6s, SD=5.3$). 

Total task duration (the maximum of human and robot time) also showed a statistically significant difference ($F(2,50)=4.887, p = 0.012$). Post-hoc tests with a Bonferroni-Holm correction show that both fixed ($M=58.6s, SD=14.1$) and plan ($M=60.6s, SD=7.1$) performed significantly ($p = 0.026$ and $p=0.032$, respectively) faster than policy ($M=65.9s, SD=6.3$). This is due to the slower execution time of the policy method, which dominates the total execution time. 

Total idle time was also significantly affected by method type ($F(2,50)=3.809, p=0.029$). Post-hoc pairwise comparisons with Bonferroni correction reveal that the policy method yielded significantly ($p = 0.048$) less idle time than the fixed condition (policy $M=0.46s, SD=0.93$, fixed $M=1.62s, SD=2.1$). Idle time percentage (total idle time divided by human trial completion time) was also significant ($F(2,50)=3.258, p=0.047$). Post-hoc pairwise tests with Bonferroni-Holm correction finds no significance between pairs. In other words, the policy method performed significantly better than the fixed method for reducing human idling time, while the plan method did not.


Therefore, \cref{hypoth:teaming_2} is partially supported: although total human task time was not significantly influenced by method condition, the total robot task time and human idle time were all significantly affected by which method was running on the robot. The robot task time was slower using the policy method, but human idling was significantly reduced by the policy method.

To evaluate \cref{hypoth:teaming_3} (subjective responses), we first conducted a Chronbach's alpha test to assure that the eight survey questions were internally consistent. The four questions asked in the negative (e.g., ``I'm dissatisfied with how HERB and I worked together'') were reverse coded so their scales matched the positive questions. The result of the test showed high consistency ($\alpha=0.849$), so we proceeded with our analysis by averaging together the participant ratings across all eight questions. 

During the experiment, participants sometimes saw collisions with the robot. We predict that collisions will be an important covariate on the subjective ratings of the three methods. In order to account for whether a collision occurred on each trial in our within-subjects design, we cannot conduct a simple repeated measures ANOVA. Instead, we conduct a linear mixed model analysis, with average rating as our dependent variable; method (policy, plan, and fixed), collision (present or absent), and their interaction as fixed factors; and method condition as a repeated measure and participant ID as a covariate to account for the fact that participant ratings were not independent across the three conditions. Table \ref{tab:subjective} shows details of the scores for each method broken down by whether a collision occurred.

\begin{table}
	\centering
	\begin{tabular}{lcccc}
		\toprule
		& \multicolumn{2}{c}{No Collision} & \multicolumn{2}{c}{Collision}\\
		\cmidrule(r){2-3} \cmidrule(l){4-5}
		& \emph{mean (SD)} & \emph{N} & \emph{mean (SD)} & \emph{N} \\ 
		\cmidrule(r){2-2}\cmidrule(lr){3-3}\cmidrule(lr){4-4} \cmidrule(l){5-5}
		Fixed 		&	5.625 (1.28)	& 14	& 4.448	(1.23)	& 12	\\
		Plan      & 	5.389 (1.05)	& 18	& 4.875	(1.28)	& 8		\\
		Policy 		&	5.308 (0.94)	& 26	& ---			& 0		\\
		\bottomrule
	\end{tabular}
	\caption{Subjective ratings for each method condition, separated by whether a collision occurred during that trial.}
	\label{tab:subjective}
\end{table}

We found that collision had a significant effect on ratings ($F(1,47.933)=6.055, p=0.018$), but method did not ($F(1,47.933)=0.312, p = 0.733$). No interaction was found. In other words, ratings were significantly affected by whether or not a participant saw a collision, but not by which method they saw independent of that collision. Therefore, \cref{hypoth:teaming_3} is not directly supported. However, our analysis shows that collisions lead to poor ratings, and our results above show that the policy method yields fewer collisions. We believe a more efficient implementation of our policy method to enable faster robot task completion, while maintaining fewer collisions, may result in users preferring the policy method.

\section{Discussion and Conclusion}
\label{sec:conclusion}

In this work, we present a method for shared autonomy that does not rely on predicting a single user goal, but assists for a distribution over goals. Our motivation was a lack of assistance when using predict-then-act methods - in our own experiment (\cref{sec:experiment_hri_2016}), resulting in no assistance for $69\%$ of execution time. To assist for any distribution over goals, we formulate shared autonomy as a POMDP with uncertainty over user goals. To provide assistance in real-time over continuous state and action spaces, we used hindsight optimization~\citep{littman_1995,chong_2000,yoon_2008} to approximate solutions. We tested our method on two shared-control teleoperation scenarios, and one human-robot teaming scenario. Compared to predict-then-act methods, our method achieves goals faster, requires less user input, decreases user idling time, and results in fewer user-robot collisions.

In our shared control teleoperation experiments, we found user preference differed for each task, even though our method outperformed a predict-then-act method across all objective measures for both tasks. This is not entirely surprising, as prior works have also been mixed on whether users prefer more control authority or better task completion~\cite{you_2011, kim_2012, dragan_2013_assistive}. In our studies, user's tended to prefer a predict-then-act approach for the simpler grasping scenario, though not significantly so. For the more complex eating task, users significantly preferred our shared autonomy method to a predict-then-act method. In fact, our method and blending were the only pair of algorithms that had a significant difference across all objective measures and the subjective measuring of like and rank (\cref{tab:eating_p_val_results}).

However, we believe this difference of rating cannot simply be explained by task difficulty and timing, as the experiments had other important differences. 
The grasping task required minimal rotation, and relied entirely on assistance to achieve it. Using blending, the user could focus on teleoperating the arm near the object, at which point the predictor would confidently predict the user goal, and assistance would orient the hand. For the feeding task, however, orienting the fork was necessary before moving the arm, at which point the predictor could confidently predict the user goal. For this task, predict-then-act methods usually did not reach their confidence threshold until users completed the most difficult portion of the task - cycling control modes to rotate and orient the fork. These mode switches have been identified as a significant contributor to operator difficulty and time consumption~\citep{herlant_2016}. This inability to confidently predict a goal until the fork was oriented caused predict-then-act methods to provide no assistance for the first $29.4$ seconds on average - which is greater then the total average time of our method ($18.5s$). We believe users were more willing to give up control authority if they did not need to do multiple mode switches and orient the fork, which subjectively felt much more tedious then moving the position.



In all experiments, we used a simple distance-based cost function, for which we could compute value functions in closed form. This enabled us to compute prediction and assistance 50 times a second, making the system feel responsive and reactive. However, this simple cost function could only provide simple assistance, with the objective of minimizing the time to reach a goal. Our new insights into possible differences of user costs for rotation and mode switches as compared to translation can be incorporated into the cost function, with the goal of minimizing user effort.

For human-robot teaming, the total task time was dominated by the robot, with the user generally finishing before the robot. In situations like this, augmenting the cost function to be more aggressive with robot motion, even at the cost of responsiveness to the user, may be beneficial. Additionally, incorporating more optimal robot policies may enable faster robot motions within the current framework.

Finally, though we believe these results show great promise for shared control teleoperation and teaming, we note users varied greatly in their preferences and desires. Prior works in shared control teleoperation have been mixed on whether users prefer control authority or more assistance~\cite{you_2011, kim_2012, dragan_2013_assistive}. Our own experiments were also mixed, depending on the task. Even within a task, users had high variance, with users fairly split for grasping (\cref{fig:survey_means}), and a high variance for user responses for full autonomy for eating (\cref{fig:survey_questions}). For teaming, users were similarly mixed in their rating for an algorithm depending on whether or not they collided with the robot (\cref{tab:subjective}). This variance suggests a need for the algorithm to adapt to each individual user, learning their particular preferences. New work by \citet{nikolaidis_2017_shared} captures these ideas through the user's \emph{adaptability}, but we believe even richer user models and their incorporation into the system action selection would make shared autonomy systems better collaborators.


\bibliographystyle{SageH}
\bibliography{references}


\begin{appendices}

\begin{table*}[!t]
\centering
\begin{tabular}{rl}
\toprule
Symbol & Description\\ \midrule
$\stateenv \in \Stateenv$ & Environment state, e.g. robot and human pose\\
$\goal \in \Goal$ & User goal\\
$\stateenvgoal \in \Stateenvgoal$ & $\stateenvgoal = (\stateenv, \goal)$. State and user goal\\
$\actionuser \in \Actionuser$ & User action\\
$\actionrobot \in \Actionrobot$ & Robot action\\
$\costuser(\state, \actionuser) = \costusergoal(\stateenv, \actionuser)$ & Cost function for each user goal\\
$\costrobot(\state, \actionuser, \actionrobot) = \costrobotgoal(\stateenv, \actionuser, \actionrobot)$ & Robot cost function for each goal\\
$\transitionallargs$ & Transition function of environment state\\
$\transition( (\stateenv', g) \given (\stateenv, g), \actionuser, \actionrobot) = \transition( \stateenv' \given \stateenv, \actionuser, \actionrobot)$ & User goal does not change with transition\\
$\transitionuser( \stateenv' \given \stateenv, \actionuser) = \transition( \stateenv' \given \stateenv, \actionuser, 0)$ & User transition function assumes the user is in full control\\
$\vgoal(\stateenv) = \vopt(\state)$ & The value function for a user goal and environment state\\
$\qgoal(\stateenvactions) = \qopt(\stateactions)$ & The action-value function for a user goal and environment state\\
$\belief$ & Belief, or distribution over states in our POMDP.\\
$\transitionbelief(\belief' \given \belief, \actionuser, \actionrobot)$ & Transition function of belief state\\
$\vrobot^{\policyrobot}(\belief)$ & Value function for following policy $\policyrobot$ given belief $\belief$\\
$\qrobot^{\policyrobot}(\belief, \actionuser, \actionrobot)$ & Action-Value for taking actions $\actionuser$ and $\actionrobot$ and following $\policyrobot$ thereafter\\
$\vhs(\belief)$ & Value given by Hindsight Optimization approximation\\
$\qhs(\belief, \actionuser, \actionrobot)$ & Action-Value given by Hindsight Optimization approximation\\
\bottomrule
\end{tabular}
\captionof{table}{Variable definitions}
\label{table:variable_definitions}
\end{table*}

\section{Variable Definitions} 
\label{sec:variable_definitions}

For reference, we provide a table of variable definitions in \cref{table:variable_definitions}.

\section{Multi-Target MDPs Proofs}
\label{sec:mingoal_thms}

Below we provide the proofs for decomposing the value functions for MDPs with multiple targets, as introduced in~\cref{sec:framework_multitarget}.

\subsection{\Cref{thm:mingoal_assist}: Decomposing value functions}
Here, we show the proof for our theorem that we can decompose the value functions over that the targets for deterministic MDPs. The proofs here are written for our shared autonomy scenario. However, the same results hold for any deterministic MDP:
\valfundecompose*
\begin{proof}
We show how the standard value iteration algorithm, computing $\qgoal$ and $\vgoal$ backwards, breaks down at each time step. At the final timestep T, we get:
\begin{align*}
  \qgoalt{T}(\stateenvactions) &= \costgoal(\stateenvactions)\\
  &= \costtarg(\stateenvactions) \qquad \text{for any $\target$}\\
  \vgoalt{T}(\stateenv) &= \min_{\actionrobot} \costgoal(\stateenvactions) \qquad \actionuser = \policyuser(\stateenv)\\
  &= \min_{\actionrobot} \min_\target \costtarg(\stateenvactions) \\
  &= \min_\target \vtargt{T}(\stateenv)
\end{align*}
Let $\target^* = \argmin \vtarg(\stateenv')$ as before. Now, we show the recursive step:
\begin{align*}
  \qgoalt{t-1}(\stateenvactions) &= \costgoal(\stateenvactions) + \vgoalt{t}(\stateenv')\\
  &= \costtargstar(\stateenvactions) + \min_\target \vtargt{t}(\stateenv') \hspace{1.3em}\\
  &= \costtargstar(\stateenvactions) + \vtargstart{t}(\stateenv')\\
  &= \qtargstar(\stateenvactions)\\
  \vgoalt{t-1}(\stateenv) &= \min_{\actionrobot} \qgoalt{t-1}(\stateenvactions)  \qquad \actionuser = \policyuser(\stateenv)\\
  &=  \min_{\actionrobot} \costtargstar(\stateenvactions) + \vtargstart{t}(\stateenv')\\
  & \geq  \min_{\actionrobot} \min_\target \left( \costtarg(\stateenvactions) + \vtargt{t}(\stateenv') \right)\\
  &= \min_\target \vtargt{t-1}(\stateenv)
\end{align*}

Additionally, we know that $\vgoal(\staterobot) \leq \min_{\target} \vtarg(\staterobot)$, since $\vtarg(\staterobot)$ measures the cost-to-go for a specific target, and the total cost-to-go is bounded by this value for a deterministic system. Therefore, $\vgoal(\staterobot) = \min_{\target} \vtarg(\staterobot)$.
\end{proof}

\subsection{\Cref{thm:mingoal_pred}: Decomposing soft value functions}
Here, we show the proof for our theorem that we can decompose the soft value functions over that the targets for deterministic MDPs:
\softvalfundecompose*

\begin{proof}
As the cost is additive along the trajectory, we can expand out $\exp(-\costtarg(\traj))$ and marginalize over future inputs to get the probability of an input now:
\begin{align*}
  \policyuser(\actionuser_t,\target| \staterobot_t) &= \frac{ \exp(-\costtarg(\staterobot_t, \actionuser_t)) \int \exp(-\costtarg(\trajatp{\staterobot_{t+1}})) } {\sum_{\target'}\int \exp(-\costtargprime(\trajat{\staterobot_{t}}))} 
\end{align*}
Where the integrals are over all trajectories. By definition, $\exp(-\vtargsoftt{t}(\staterobot_t)) = \int \exp(-\costtarg(\trajat{\staterobot_t}))$:
\begin{align*}
  &= \frac{ \exp(-\costtarg(\staterobot_t, \actionuser_t)) \exp(-\vtargsoftt{t+1}(\staterobot_{t+1}))} {\sum_{\target'} \exp(-\vsoft_{\target',t}(\staterobot_{t}) )} \\
  &= \frac{ \exp(-\qtargsoftt{t}(\staterobot_t, \actionuser_t))} {\sum_{\target'} \exp(-\vsoft_{\target',t}(\staterobot_{t}) )} 
\end{align*}
Marginalizing out $\target$ and simplifying:
\begin{align*}
  & \policyuser(\actionuser_t| \staterobot_t) = \frac{\sum_\target \exp( -\qtargsoftt{t}(\staterobot_t, \actionuser_t))} {\sum_{\target} \exp(-\vtargsoftt{t}(\staterobot_{t}) )} \\
  &= \exp \left( \log \left( \frac{\sum_\target \exp( -\qtargsoftt{t}(\staterobot_t, \actionuser_t))} {\sum_{\target} \exp(-\vtargsoftt{t}(\staterobot_{t}) )} \right) \right)\\
  &= \exp \left( \softmin_\target \vtargsoftt{t}(\staterobot_t) - \softmin_\target \qtargsoft{t}(\staterobot_t, \actionuser_t) \right)
\end{align*}
As $\vgoalsoftt{t}$ and $\qgoalsoftt{t}$ are defined such that $\policyuser_t(\actionuser | \staterobot, \goal) = \exp(\vgoalsoftt{t}(\staterobot) -\qgoalsoftt{t}(\staterobot, \actionuser))$, our proof is complete.
\end{proof}

\newcommand{\inreg}[1]{n_{#1}}
\newcommand{\inreghat}[1]{\widehat{n}_{#1}}
\newcommand{\numwithobs}[1]{n^{#1}}
\newcommand{\numwithobshat}[1]{\widehat{n}^{#1}}
\newcommand{\inregandobs}[2]{n_{#1}^{#2}}
\newcommand{\inregandobshat}[2]{\widehat{n}_{#1}^{#2}}
\newcommand{\totalhypoths}{N}
\newcommand{\totalhypothshat}{\widehat{N}}

\newcommand{\edgermind}{l}
\newcommand{\inregrm}{\inreg{\edgermind}}
\newcommand{\subregrm}{\subreg{\edgermind}}

\newcommand{\hyperedgesetk}{\hyperedgeset_\edgermind}
\newcommand{\hyperedgesetnok}{\overline{\hyperedgeset_{\edgermind}}}

\newcommand{\minhyperedgeset}{\hyperedgeset^{\min}}
\newcommand{\minhyperedgesetk}{\minhyperedgeset_\edgermind}
\newcommand{\minhyperedgesetnok}{\overline{\minhyperedgeset_{\edgermind}}}
\newcommand{\nominhyperedgeset}{\hyperedgeset^{\overline{\min}}}
\newcommand{\nominhyperedgesetk}{\nominhyperedgeset_\edgermind}
\newcommand{\nominhyperedgesetnok}{\overline{\nominhyperedgeset_{\edgermind}}}
\newcommand{\numk}{ {\vert\hyperedge_\edgermind\vert} }

\newcommand*\circled[1]{\tikz[baseline=(char.base)]{\node[shape=circle,draw,inner sep=2pt] (char) {#1};}}

\newcommand{\sumhe}{\sum_{\hyperedge \in \hyperedgeset}}
\newcommand{\sumhek}{\sum_{\hyperedge \in \hyperedgesetk}}
\newcommand{\sumhenok}{\sum_{\hyperedge \in \hyperedgesetnok}}
\newcommand{\summinhenok}{\sum_{\hyperedge \in \minhyperedgesetnok}}
\newcommand{\sumnominhenok}{\sum_{\hyperedge \in \nominhyperedgesetnok}}
\newcommand{\sumallobs}{\sum_{\observationitem \in \observationset}} 
\newcommand{\sumobsnotc}{\sum_{\observationitem \in \observationset \backslash c}} 

\newcommand{\prodedge}{\prod_{i \in \hyperedge}}
\newcommand{\prodedgenok}{\prod_{i \in \hyperedge, i \neq l}}

\end{appendices}

%
%

\end{document}